\documentclass[11pt]{article}

\usepackage[margin=1.0in]{geometry}

\usepackage{caption,footnote,authblk}
\usepackage{empheq}

\usepackage{makecell,enumitem}
\usepackage{booktabs}
\usepackage{array}
\usepackage{url}
\usepackage{bm}
\usepackage{mystyle}

\usepackage{bbm}
\usepackage{relsize}

\usepackage{natbib}
\usepackage{algorithm,algorithmic}

\usepackage{color}
\usepackage{mathtools}

\usepackage[usenames,dvipsnames,svgnames,table]{xcolor}
\usepackage[colorlinks,
linkcolor=red,
citecolor=blue
]{hyperref}

\newlength\myindent
\makeatletter
\def\@fnsymbol#1{\ensuremath{\ifcase#1\or  \natural \or \dagger\or * \or \ddagger\or
   \mathsection\or \mathparagraph\or \|\or **\or \dagger\dagger
   \or \ddagger\ddagger \else\@ctrerr\fi}}
\makeatother

\def\##1\#{\begin{align}#1\end{align}}
\def\$#1\${\begin{align*}#1\end{align*}}

\usepackage{comment}

\begin{document}
\title{Finite-Sample Analysis For  Decentralized Batch Multi-Agent Reinforcement Learning With Networked Agents}
 
\author{Kaiqing Zhang\thanks{Department of Electrical and Computer Engineering \&  Coordinated Science Laboratory, University of Illinois at Urbana-Champaign}\quad~ Zhuoran Yang\thanks{Department of Operations Research and Financial Engineering, Princeton University}\quad~ Han Liu\thanks{Department of Electrical Engineering and Computer Science and Statistics, Northwestern  University}~\quad~ Tong Zhang\thanks{Department of Computer Science \& Engineering, the Hong Kong University of Science and Technology}\quad~ Tamer Ba\c{s}ar$^\natural$}

\date{}

\maketitle


\begin{abstract}
Despite the increasing interest  in   multi-agent reinforcement learning (MARL) in multiple  communities, understanding its theoretical foundation has long been recognized as a challenging  problem.  In this work, we address  this problem  by providing   a finite-sample analysis for  decentralized batch MARL with networked agents. Specifically, we consider  two   decentralized MARL settings, where teams of agents are connected by time-varying communication networks, and either collaborate or compete in a zero-sum game  setting, without any central controller.  These settings cover many conventional MARL settings in the literature. 
For both settings, 
we develop  batch  MARL algorithms that can be implemented in a decentralized fashion, and quantify the finite-sample errors of the estimated action-value functions. Our  error analysis captures   how the function class, the number of samples within each iteration, and the number of iterations  determine the statistical accuracy of the proposed algorithms.
Our results, compared to the finite-sample bounds for single-agent RL, involve  additional error terms caused by decentralized computation, which is inherent in  our decentralized MARL setting. This work appears to be the first finite-sample analysis for batch MARL, a step towards rigorous theoretical understanding of general MARL algorithms in the finite-sample regime. 
\end{abstract}  

\section{Introduction}
Multi-agent reinforcement learning (MARL) has received increasing attention in the reinforcement  learning community, with recent advances in both empirical \citep{foerster2016learning,lowe2017multi,lanctot2017unified} and theoretical studies \citep{perolat2016learning,zhang2018fully,zhang18cdc,doan2019finite,srinivasan2018actor,zhang2019policyb}. 
With various models of  multi-agent systems, MARL has been applied to  a wide range of domains, including distributed control, telecommunications, and  economics. 
See \cite{busoniu2008comprehensive} and \cite{zhang2019multi} for the surveys on MARL. 

Various settings  exist in the literature on  multi-agent RL, which are  mainly  categorized  into three types:   the cooperative setting,  the competitive setting, and a  mix  of the  two. In particular, cooperative MARL is usually modeled as either a multi-agent Markov decision process (MDP) \citep{boutilier1996planning}, or a team Markov game \citep{wang2003reinforcement}, where the agents are assumed to share a common reward function.
A more general while  challenging setting for cooperative MARL  considers   heterogeneous reward functions for different agents, while the collective goal is to maximize the average of the long-term return among all agents \citep{kar2013cal,zhang2018fully,suttle2019multi}. This setting makes it nontrivial  to design \emph{decentralized} MARL algorithms, in which agents make globally optimal decisions using  only local information, \emph{without} any  coordination provided by any central controller. 
Decentralized protocol is favored over a centralized one, due to its  better scalability, privacy-preserving property, and computational efficiency \citep{kar2013cal,zhang2018fully,zhang2018distributed}.
Such a  protocol has been broadly advocated in  practical multi-agent systems,  including unmanned (aerial) vehicles \citep{fax2004information,zhang18dynamicb}, smart power grid \citep{zhang18distributed,zhang18dynamica}, and robotics  \citep{corke2005networked}. 
Several preliminary attempts have been made towards the development of MARL algorithms in this setting \citep{kar2013cal,zhang2018fully}, with theoretical guarantees for convergence. 
However, these  theoretical guarantees are essentially \emph{asymptotic}  results, i.e., the algorithms are shown to converge 
as the number of iterations increases to infinity. No analysis has been conducted    to quantify the performance of these  algorithms with a  finite number of iterations/samples.

Competitive MARL, on the other hand, is usually investigated under the framework of Markov games, especially  zero-sum Markov games. Most existing competitive MARL algorithms typically concern  two-player games \citep{littman1994markov,bowling2002multiagent,perolat2015approximate}, which can be viewed as a generalization of the standard MDP-based RL model \citep{littman1996generalized}.  
In fact, the  decentralized protocol  mentioned earlier can  be incorporated  to generalize such two-player  competitive setting. Specifically, one may consider two teams of cooperative agents that form a zero-sum Markov game. Within each team, no central controller exists to coordinate the agents. We note that this generalized setting   applies to many practical examples of multi-agent systems, including two-team-battle video games \citep{do2017moba}, robot soccer games \citep{kitano1997robocup}, and security for cyber-physical systems \citep{cardenas2009challenges}. 
This setting can also be viewed as a special case of the mixed setting, which is usually   modeled as a general-sum Markov game,  and is nontrivial to solve using RL algorithms \citep{littman2001friend,zinkevich2006cyclic}. This specialization facilitates the non-asymptotic analysis on  the  MARL algorithms in the mixed setting.

In general, finite-sample analysis  is  relatively scarce even for  single-agent RL,  in contrast  to its  empirical studies and asymptotic convergence analysis. One  line of  work studies the probably approximately correct (PAC) learnability and sample complexity of RL algorithms \citep{kakade2003sample,strehl2009reinforcement,jiang2016contextual}. Though built upon  solid theoretical foundations, most  algorithms are developed for the tabular-case RL only \citep{kakade2003sample,strehl2009reinforcement}, and are computationally intractable for large-scale RL. Another significant line of work concentrates on the finite-sample analysis for \emph{batch RL} algorithms  \citep{munos2008finite,antos2008fitted,antos2008learning,yang18cdc, yang2018dqn}, using  tools  from statistical learning theory. Specifically, these results characterize the errors of output value functions after a finite number of  the iterations  using a finite number of samples. These  algorithms are built upon value function approximation and use possibly single trajectory data, which can thus handle massive state-action spaces, and enjoy the advantages of off-policy exploration  \citep{antos2008fitted,antos2008learning}. 

Following the batch RL framework, we aim to establish a finite sample analysis for both cooperative and competitive MARL problems.   For both  settings, we propose  decentralized fitted Q-iteration algorithms with value function approximation and provide finite-sample analysis. Specifically, for the cooperative setting where a team of finite agents on a communication network aims to maximize the  global average of the cumulative discounted reward obtained by all the agents, we establish the statistical error of the action-value function returned by a decentralized variation of the fitted Q-iteration algorithm, which is measured by the $\ell_2$-distance between the estimated action-value function and the optimal one. Interestingly, we show that the statistical error can be decomposed into a sum of three error terms that  reflect the effects of 
the function class, number of samples in each fitted Q-iteration, and the number of iterations, respectively.
Similarly, for the competitive MARL where the goal is to achieve the Nash equilibrium of a zero-sum game played by two teams of networked agents, we propose a  decentralized algorithm that is shown to approximately achieve the desired Nash equilibrium. Moreover, in this case, we also  establish the  $\ell_2$-distance between the  action-value function returned by the algorithm and the one at  the Nash equilibrium.

\vspace{5pt}
\noindent {\bf{Main Contribution.}} 
Our contribution   is two-fold. First, we propose batch MARL algorithms for both the cooperative and competitive settings where the agents are only allowed to communicate over a network and the  reward of each individual agent is observed only by  itself. 
Our algorithms incorporate function approximation of the value functions and can be implemented in a  decentralized fashion.   Second, more importantly, we provide a finite-sample analysis for these MARL algorithms, which characterizes the statistical errors of the action-value function returned by our algorithms. Our error bounds provides a clear  characterization of how the complexity of the employed function class, the number of samples, and the number of iterations affect the statistical accuracy of our algorithms.
This  work appears to provide the first finite-sample analysis for batch MARL  for either cooperative or competitive settings.   

\vspace{5pt}
\noindent {\bf{Related Work.}} 
The original  formulation of MARL traces back to \cite{littman1994markov}, based on the framework of Markov games. Since then, various settings of MARL have been advocated. For  
 the cooperative setting, \cite{boutilier1996planning,lauer2000algorithm} proposed a  \emph{multi-agent MDP} model, where all agents are assumed to share identical reward functions. The formulations in \cite{wang2003reinforcement,mathkar2017distributed} are based on the setting  of team Markov games, which also assume that the  agents necessarily have a common reward function. More recently, \cite{kar2013cal,zhang2018fully,lee2018primal,wai2018multi,chen2018communication,suttle2019multi} advocated the  decentralized MARL  setting,  which allows heterogeneous rewards of  agents. For the competitive setting, the model of \emph{two-player} zero-sum Markov games has been studied extensively  in the literature \citep{littman1994markov,bowling2002multiagent,perolat2015approximate,yang2018dqn}, which can be readily recovered from the two-team competitive MARL in the present work.  
 In particular, 
 the recent work \cite{yang2018dqn} that  studies the finite-sample performance of minimax deep Q-learning for {two-player} zero-sum games can be viewed as a specialization of our two-team setting. 
 We note that the most relevant early work   that also considered such a two-team competitive MARL setting  was \cite{lagoudakis2003learning}, where the value function was assumed to have a factored structure among agents. As a result, a computationally efficient algorithm integrated with least-square policy iteration was proposed to learn  the good strategies for a team of agents against the other. However, no theoretical, let alone finite-sample, analysis was established in the work.   Several recent work also investigated  the mixed setting with both cooperative and competitive agents \citep{foerster2016learning,tampuu2017multiagent,lowe2017multi}, but with  focus on empirical instead of theoretical  studies.
Besides, \cite{chakraborty2014multiagent} also considered multi-agent learning with sample complexity analysis under the framework of repeated matrix games.   
 
To lay theoretical foundations for RL, an increasing attention has been paid to finite-sample, namely, non-asymptotic   analysis,  of the algorithms. One line of work studies the sample complexity of RL algorithms under the framework of PAC-MDP  \citep{kakade2003sample,strehl2009reinforcement,jiang2016contextual}, focusing on the efficient exploration of the algorithms.  Another more relevant line of work investigated the finite-sample performance  of batch RL algorithms,  based on the tool of approximate dynamic programming for analysis. \cite{munos2008finite} studied a   finite-sample bounds for the fitted-value iteration algorithm, followed by \cite{antos2008learning,antos2008fitted} on fitted-policy iteration and continuous action fitted-Q iteration. 
Similar ideas and techniques were also explored in  \citep{scherrer2015approximate,farahmand2016regularized, yang2018dqn}, for modified  policy iteration,  nonparametric function spaces, and deep neural networks, respectively. 
Besides, for online RL algorithms, \cite{liu2015finite,gal2018finite,jalaj2018finite,srikant2019finite} recently carried out  the finite-sample analysis   for  temporal difference  learning algorithms, and have been extended to the distributed/decentralized multi-agent setting \citep{doan2019finite,doan2019finiteb} concurrently to the preparation of this paper. Complementary to \cite{doan2019finite,doan2019finiteb}, our analysis focuses on the batch RL paradigm that is aligned with \citep{munos2008finite,antos2008learning,antos2008fitted,scherrer2015approximate}. Moreover,  our finite-sample results are developed for  \emph{control} algorithms, while  
\cite{doan2019finite,doan2019finiteb} only analyzed \emph{policy evaluation}  algorithms.  

 The most relevant analysis for batch MARL was provided by  \cite{perolat2015approximate,perolat2016learning}, focusing on the  error propagation of the algorithm  using  also the tool of approximate dynamic programming. However, no statistical rate was provided in the error analysis,  nor the computational complexity that solves the  fitting problem at each iteration. 
 We note that the latter becomes inevitable in our  decentralized setting,   since the computation procedure explicitly shows up in our design of   decentralized MARL algorithms. This was not touched upon  in the single-agent  analysis  \citep{antos2008learning,antos2008fitted,scherrer2015approximate}, since they assume the exact solution to the  fitting problem  at each iteration can be obtained. 
  
 	In addition, in the regime of  decentralized decision-making, decentralized partially-observable MDP (Dec-POMDP) \citep{oliehoek2016concise} is recognized as the most general and powerful model. Accordingly, some finite-sample analysis based on PAC analysis  has been established in the literature \citep{amato2009achieving,banerjee2012sample,ceren2016reinforcement}. However, since Dec-POMDPs are known to be NEXP-complete and thus  difficult to solve in general, these  Dec-POMDP solvers are  built upon some or all the following requirements: i) a \emph{centralized planning} procedure to optimize the policies for all the agents \citep{amato2009achieving}; ii) the availability of \emph{the model or the simulator} for the sampling  \citep{amato2009achieving,banerjee2012sample}, or not completely model-free \citep{ceren2016reinforcement}; iii) a \emph{special structure} of the reward \citep{amato2009achieving}, or the policy-learning process   \citep{banerjee2012sample}. Also,  these PAC results only apply to the  small-scale setting with tabular state-actions and mostly finite time horizons. In contrast, our analysis is amenable to the batch RL algorithms that utilize function approximation for the setting with large state-action spaces.

\vspace{3pt}
\noindent {\bf{Notation.}} 
For  a measurable space with domain $\cS$, we denote the set of   measurable functions on $\cS$ that are bounded by $V  $ in absolute value by    $\cF(\cS,V)$.
  Let $\cP(\cS)$ be the set of all probability measures on $\cS$.  For any $\nu \in \cP(\cS)$ and any measurable function $f \colon \cS \rightarrow \RR$, we denote by $\| f \|_{\nu}$ the $\ell_2$-norm of $f$ with respect to measure $\nu$. 
We use $a\vee b$ to denote $\max\{a,b\}$ for any $a,b\in\RR$, and define the set $[K]=\{1,2,\cdots,K\}$.



\section{Problem Formulation}
In this section, we formulate the   decentralized MARL problem with networked agents, for both  cooperative and  competitive settings.

\subsection{Decentralized Cooperative MARL} \label{sec:colla_bkg}
Consider   a team of $N$ agents, denoted by $\cN = [ N ] $, that  operate in a common environment in a cooperative fashion. 
In the  decentralized setting,  there exists no central controller  that is able to either  collect rewards  or make the  decisions for the agents. 
Alternatively, to foster the collaboration, agents are assumed to be able to exchange information via a possibly time-varying  communication network. We denote the communication network by a graph $G_{\tau}=(\cN, E_\tau)$, where the edge set $E_\tau$ represents the set of communication links at time $\tau\in \NN$. 
 Formally,  we  define the following model of multi-agent MDP (M-MDP) with networked  agents.

\begin{definition} [Multi-Agent MDP with Networked Agents]\label{def:NMMDP}
A multi-agent MDP with networked  agents is characterized by a  tuple 
$( \cS, \{\cA^i\}_{i\in\cN}, P, \{R^i\}_{i\in\cN},  \{ G_{\tau} \}_{\tau \geq 0}, \gamma) $ where $\cS$ is the global state space shared by all the agents in $\cN$, and $\cA^{i}$ is the set of  actions that agent $i$ can choose from. Let 
$\cA=\prod_{i=1}^{N}\cA^i$ denote the joint action space of all agents. Moreover, $P:\cS\times\cA\to \cP(\cS)$ is the  probability distribution of the next state, $R^i :\cS\times\cA \to\cP(\RR)$ is
the distribution of local reward function of agent $i$, which  both depend on the joint actions $a$ and the global state $s$, and $\gamma\in(0,1)$ is the discount factor.
$\cS $ is a compact subset of $\RR^d$ which can be infinite,  $\cA$ has finite cardinality $A=|\cA|$, and the rewards have absolute values uniformly bounded by $R_{\max}$. At time $\tau$,\footnote{Note that this time index $\tau$ can be different from the time index $t$ for the M-MDP, e.g., it can be the index of the algorithms updates. See examples in \S\ref{sec:algorithms}. } the agents are connected by the communication network $G_{\tau}$. 
The states and the joint actions are globally observable while the    rewards are observed only  locally.
\end{definition}

By this definition, agents observe the global state $s_t$ and perform joint actions $a_t = (a_t^1, \ldots, a_t^N) \in \cA$ at time $t$. In consequence, each agent $i$ receives an instantaneous  reward $r_{t}^i$ that samples from the distribution  $R^i(\cdot\given s_t,a_t)$. Moreover, the environment evolves to a new state $s_{t+1}$ according to the transition probability $P(\cdot \given s_t, a_t)$. We refer to this model as a \emph{decentralized} one because  each agent makes \emph{individual}   decisions based on the \emph{local} information acquired from the  network.
In particular, we assume that given the current state, each agent $i$ chooses  actions independently to each other, following its own policy  $\pi^i :\cS\to \cP(\cA^i)$. Thus, the joint policy of all agents,  denoted by $\pi \colon \cS\to\cP(\cA)$, satisfies $\pi(a\given s)=\prod_{i\in\cN}\pi^i(a^i\given s)$ for any $s\in\cS$ and $a\in\cA$. 

The cooperative  goal of the agents is to maximize the  \emph{global average} of the  cumulative discounted reward obtained by  all agents over the network, which can be formally written as 
\$
\max_{\pi} \quad \frac{1}{N}\sum_{i\in\cN}\EE\Bigg(\sum_{t=0}^\infty \gamma^t \cdot r^i_t\Bigg).
\$
Accordingly, under any joint policy $\pi$,
the action-value function $Q_\pi: \cS \times \cA \to \real$ can be defined as 
\$
Q_\pi(s,a) = \frac{1}{N}\sum_{i\in\cN}\EE_{a_t \sim \pi(\cdot\given s_t) } \bigg [\sum_{t=0}^\infty \gamma^t \cdot r^i_t\bigggiven s_0 = s ,  a_0=a   \bigg ].
\$
Notice that since $r^i_t\in[-R_{\max},R_{\max}]$ for any $i\in\cN$ and $t\geq 0$,  $Q_\pi$ are bounded by $R_{\max}/(1-\gamma)$ in absolute value for any policy $\pi$.
We let $Q_{\max}=R_{\max}/(1-\gamma)$ for notational convenience. Thus  we have  $Q_\pi\in\cF(\cS\times\cA,Q_{\max})$ for any $\pi$.
We refer to  $Q_\pi$ as  \emph{global $Q$-function} hereafter.
For notational convenience, under joint policy $\pi$, we  define the operator  $P_{\pi}: \cF(\cS\times\cA,Q_{\max})\to \cF(\cS\times\cA,Q_{\max})$ and the Bellman operator ${\cT}_\pi : \cF(\cS\times\cA,Q_{\max})\to \cF(\cS\times\cA,Q_{\max})$ that correspond to the globally averaged  reward   as follows 
\#
(P_\pi Q) (s, a) =~  \EE_{s' \sim P(\cdot \given s,a), a' \sim \pi(\cdot \given s') } \bigl [Q(s', a') \bigr ],\qquad \qquad 
 ({\cT}_\pi Q) (s, a) =~ \overline{r}(s, a) + \gamma \cdot (P_\pi Q) (s, a), \label{eq:operator_P}
\#
where 
$\overline{r}(s, a)=\sum_{i\in\cN} r^i(s,a)\cdot N^{-1}$ denotes the globally averaged reward with $r^i(s,a)=\int r R^i(dr\given s,a)$. Note that the action-value function $Q_{\pi}$ is the unique fixed point of $\cT_{\pi}$.
Similarly, we also define the optimal Bellman operator corresponding to the averaged reward  $\overline{r}$ as
\$ 
	({\cT}Q) (s,a) = \overline{r}(s, a) + \gamma \cdot\EE_{s' \sim P(\cdot  \given s,a) } \bigl[ \max_{a' \in \cA} Q(s', a')\bigr].
\$
Given a vector of $Q$-functions $\Qb\in [\cF(\cS\times\cA,Q_{\max})]^N$  with $\Qb=[Q^i]_{i\in\cN}$,   
we also define the \emph{average Bellman operator}  $\tcT:[\cF(\cS\times\cA,Q_{\max})]^N \to \cF(\cS\times\cA,Q_{\max})$ as 
\#\label{equ:operator_average_T_opt}
(\tcT\Qb)(s,a) = \frac{1}{N}\sum_{i\in\cN} (\cT^iQ^i)(s,a)~~~~\text{with}~~~~\cT^iQ^i={r}^i(s, a) + \gamma\cdot  \EE_{s' \sim P(\cdot  \given s,a)}\biggl[\frac{1}{N}\cdot\sum_{i\in\cN}\max_{a' \in \cA} Q^i(s', a') \biggr]. 
\#
Note that $\tcT\Qb={\cT}Q$ if $Q^i=Q$ for all $i\in\cN$.

In addition, for  any action-value function $Q\colon \cS \times \cA \rightarrow \RR$, one can define the  \emph{one-step} greedy policy $\pi_Q$ to be the  deterministic policy that chooses the action with the largest $Q$-value, i.e., for any $s\in\cS$, it holds that 
\$
\pi_{Q}(a\given s) =1 ~~\text{if }~~a=\argmax_{a'\in \cA}~ Q(s,  a' ).
\$
If there are more than one actions $a'$ that maximize the $Q(s,a')$, we break the tie randomly. 
Furthermore, we can define an operator $\cG$, which 
generates the average greedy policy of a vector of $Q$-function, i.e., $\cG(\Qb)=N^{-1}\sum_{i\in\cN}\pi_{Q^i}$, 
where $\pi_{Q^i}$ denotes the greedy policy with respect to $Q^i$.

\subsection{Decentralized Two-Team Competitive MARL}\label{sec:compete_bkg}
Now, we extend the  decentralized cooperative MARL model to a competitive setting. In particular, we consider two teams,  referred to as \emph{Team $1$} and \emph{Team $2$} that operate in a common environment. Let  $\cN$ and $\cM$ be the sets of agents in Team $1$ and Team 
$2$, respectively,  with $|\cN|=N$  and $|\cM|=M$. 
 We assume the two teams form a zero-sum Markov game, i.e., the instantaneous rewards of all agents sum up to zero. Moreover, within each team, the agents can exchange information via a communication network, and collaborate in a  decentralized fashion as defined in \S\ref{sec:colla_bkg}. 
We give a formal definition for   such a model of zero-sum Markov game with networked agents. 

\begin{definition}
[Zero-Sum Markov Game with Networked Agents]
\label{def:networked_ZS}
A \emph{zero-sum Markov game with networked agents} is   characterized by a  tuple $
\Bigl ( \cS, \big\{\{\cA^i\}_{i\in\cN}, \{\cB^j\}_{j\in\cM}\big\}, P, \big\{\{R^{1,i}\}_{i\in\cN},\{R^{2,j}\}_{j\in\cM}\big\}  \big\{\{ G^1_{\tau} \}_{\tau \geq 0},\\\{ G^2_{\tau} \}_{\tau \geq 0}\big\}, \gamma \Bigr)$, where $\cS$ is the global state space shared by all the agents in $\cN$ and $\cM$,  $\cA^{i}$  (resp. $\cB^{i}$) are the sets of  actions for any agent $i\in\cN$ (resp. $j\in\cM$).  Let 
$\cA=\prod_{i\in\cN}\cA^i$ (resp. $\cB=\prod_{j\in\cM}\cB^j$) denote the joint action space of all agents in Team $1$ (resp.  in Team $2$). Moreover, $P:\cS\times\cA\times\cB\to \cP(\cS)$ is the  probability distribution of the next state, $R^{1,i},R^{2,j}:\cS\times\cA\times\cB \to\cP(\RR)$  are 
the distribution of local reward function of agent $i\in\cN$ and $j\in\cM$, respectively, and $\gamma\in(0,1)$ is the discount factor.
Also, the two teams form a zero-sum Markov game, i.e., at time $t$,  $\sum_{i\in\cN}r^{1,i}_t+\sum_{j\in\cM}r^{2,j}_t=0$, with $r^{1,i}_t\sim R^{1,i}(s_t,a_t,b_t)$ and $r^{2,j}_t\sim R^{2,j}(s_t,a_t,b_t)$ for all $i\in\cN$ and $j\in\cM$.  
Moreover,  
$\cS $ is a compact subset of $\RR^d$, both $\cA$ and $\cB$ have finite cardinality $A=|\cA|$ and $B=|\cB|$, and the rewards have absolute values  uniformly bounded by $R_{\max}$. 
All the agents in Team $1$ (resp.  in Team $2$), are connected by the communication network $G^1_{\tau}$  (resp. $G^2_{\tau}$) at time $\tau$. 
The states and the joint actions are globally observable while the    rewards are observed only  locally by each agent.
\end{definition}

We note that  this model generalizes  the most common competitive MARL setting,  which is usually modeled as a  two-player zero-sum Markov game {\citep{littman1994markov,perolat2015approximate}}. Additionally, we allow collaboration among  agents within the same team, which establishes a type of  mixed MARL setting with both  cooperative and competitive  agents. This mixed  setting  finds broad practical applications, including { team-battle video games \citep{do2017moba}, robot soccer games \citep{kitano1997robocup}, and security for cyber-physical systems \citep{cardenas2009challenges}.}
 
For two-team competitive MARL, each team now aims to find the team's joint policy that optimizes  the  average of the cumulative rewards over all agents  in that team.  
With the zero-sum assumption, 
this goal    
can also be viewed as one team, e.g., Team $1$, maximizing the  globally averaged return of its agents; whereas the other team (Team $2$) minimizing the same  globally averaged return. Accordingly, we refer to Team $1$ (resp. Team $2$) as the \emph{maximizer} (resp. \emph{minimizer}) team, without loss of generality. 
Let $\pi: \cS\to\cP(\cA)$ and $\sigma: \cS\to\cP(\cB)$ be the team-joint policies of Team $1$ and Team $2$, respectively. Then, one can similarly define the   action-value function $Q_{\pi,\sigma}:\cS\times\cA\times\cB\to\RR$ as 
\$
Q_{\pi,\sigma}(s,a,b) = \frac{1}{N}\sum_{i\in\cN}\EE_{\pi(\cdot\given s_t),\sigma(\cdot\given s_t)} \bigg [\sum_{t=0}^\infty \gamma^t \cdot r^{1,i}_t\bigggiven s_0 = s ,  a_0=a,b_0=b   \bigg ],
\$
for any $(s,a,b)\in\cS\times\cA\times\cB$, where   $\EE_{\pi(\cdot\given s_t),\sigma(\cdot\given s_t)}$ means that $a_t\sim \pi(\cdot\given s_t)$ and $b_t\sim \sigma(\cdot\given s_t)$ for all $t\geq 1$. We also define the 
value function $V_{\pi,\sigma}:\cS\to\RR$  by $$V_{\pi,\sigma}(s)=\EE_{a \sim \pi(\cdot\given s),b \sim \sigma(\cdot\given s)}[Q_{\pi,\sigma}(s,a,b)].$$
Note that $Q_{\pi,\sigma}$ and $V_{\pi,\sigma}$ are both bounded by $Q_{\max}=R_{\max}/(1-\gamma)$ in absolute values, i.e., $Q_{\pi,\sigma}\in\cF(\cS\times\cA\times\cB,Q_{\max})$ and $V_{\pi,\sigma}\in\cF(\cS,Q_{\max})$. Formally, the  collective goal of Team $1$ is to solve   
$
\max_{\pi}\min_{\sigma} ~ V_{\pi,\sigma}.
$
The collective goal of Team $2$ is thus to solve $\min_{\sigma}\max_{\pi}V_{\pi,\sigma}
$. 
For finite action spaces $\cA$ and $\cB$ (for either $\cS$ being finite or continuous and compact), 
there exists a minimax value $V^*\in\cF(\cS,Q_{\max})$ of the game  such that \citep{shapley1953stochastic,maitra1970stochastic,maitra1971stochastic,parthasarathy1973discounted} 
\$
V^*=\max_{\pi}\min_{\sigma} ~ V_{\pi,\sigma}
=\min_{\sigma}\max_{\pi}~V_{\pi,\sigma}.
\$
Note that the interchange of $\max$ and $\min$ above is over the \emph{team-joint} policies $\pi$ and $\sigma$, and not the individual policy of each agent. In the two-team zero-sum setting, even for matrix games, restricting the team-joint policy to individual policies that are independent to each agent, i.e., $\pi=\prod_{i\in\cN}\pi^i$ for some local policy $\pi^i: \cS\to\cP(\cA^i)$ (as in the cooperative setting in \S\ref{sec:colla_bkg}),  leads to non-existence of the value  \citep{schulman2017duality}. We leave the study of this setting in our future work.  Fortunately, as our algorithm to be introduced is Q-learning-based, the team-joint equilibrium policies can be recovered from the team Q-function estimate at each agent, without using the independence property. 
One can thus define the minimax $Q$-value of the game as $Q^*=\max_{\pi}\min_{\sigma}Q_{\pi,\sigma}=\min_{\sigma}\max_{\pi}Q_{\pi,\sigma}$.
We also define the \emph{optimal} $Q$-value of Team $1$ under policy $\pi$ as $Q_{\pi}=\min_{\sigma}Q_{\pi,\sigma}$, where the opponent Team $2$ is assumed to  be performing the  best (minimizing) response to $\pi$.

Moreover, under fixed joint policy $(\pi,\sigma)$ of two teams, one can define the operators $P_{\pi,\sigma}, P^*: \cF(\cS\times\cA\times\cB,Q_{\max})\to \cF(\cS\times\cA\times\cB,Q_{\max})$ and the Bellman operators ${\cT}_{\pi,\sigma} ,\cT: \cF(\cS\times\cA\times\cB,Q_{\max})\to \cF(\cS\times\cA\times\cB,Q_{\max})$  by
\#
(P_{\pi,\sigma} Q) (s, a,b) =&~  \EE_{s' \sim P(\cdot \given s,a,b), a' \sim \pi(\cdot \given s'),b' \sim \sigma(\cdot \given s') } \bigl [Q(s', a',b') \bigr ],\notag\\
(P^*Q) (s, a,b) =&~  \EE_{s' \sim P(\cdot \given s,a,b)} \Bigl\{\max_{\pi'\in\cP(\cA)}\min_{\sigma'\in\cP(\cB)}\EE_{\pi',\sigma' }\Bigl [Q(s', a',b') \bigr ]\Bigl\},\label{eq:operator_P_opt}\\
 ({\cT}_{\pi,\sigma} Q) (s, a,b) =&~ \overline{r}^1(s, a,b) + \gamma \cdot (P_{\pi,\sigma} Q) (s, a,b), \notag \\
 ({\cT}Q) (s, a,b) =&~ \overline{r}^1(s, a,b) + \gamma \cdot(P^*Q) (s, a,b), \notag
\#
where 
$\overline{r}^1(s, a,b)=N^{-1} \cdot \sum_{i\in\cN} r^{1,i}(s,a,b)  $ denotes the globally averaged reward of Team $1$, where we write $r^{1,i}(s,a,b)=\int r R^{1,i}(dr\given s,a,b)$. 
Note that the operators  $P_{\pi,\sigma}$, $P^*$, ${\cT}_{\pi,\mu}$, and $\cT$ are all defined corresponding to the globally averaged reward of Team $1$, i.e., $\overline{r}^1$. One can also define all the quantities above based on that of Team $2$, i.e., $\overline{r}^2(s,a,b)=\sum_{j\in\cM} r^{2,j}(s,a,b)\cdot M^{-1}$, with the  $\max$ and $\min$ operators interchanged. Also note that  the $\max\min$ problem on the right-hand side of \eqref{eq:operator_P_opt}, which involves  a \emph{matrix game} given by $Q(s',\cdot,\cdot)$, may not admit  pure-strategy solutions. For notational brevity, we also define  the following two operators $\cT_{\pi}$ as 
$
\cT_{\pi}Q=\min_{\sigma}~\cT_{\pi,\sigma}Q,
$
for any $Q\in\cF(\cS\times\cA\times\cB,Q_{\max})$.
Moreover, 
with a slight abuse of notation, we also define the average Bellman operator $\tcT:[\cF(\cS\times\cA\times\cB,Q_{\max})]^N \to \cF(\cS\times\cA\times\cB,Q_{\max})$ for the maximizer team, i.e., Team $1$, similar to the definition in \eqref{equ:operator_average_T_opt} \footnote{For convenience, we will write $\EE_{a\sim\pi',b\sim\sigma'} \big[ Q (s, a,b)\big]$ as $\EE_{\pi',\sigma'} \big[ Q (s, a,b)\big]$ hereafter.}: 
\#\label{equ:operator_average_T_opt_comp}
(\tcT\Qb)(s,a,b) =&~ \overline{r}^1(s, a,b) + \gamma\cdot  \EE \biggl[\frac{1}{N}\sum_{i\in\cN}\max_{\pi'\in\cP(\cA)}\min_{\sigma'\in\cP(\cB)}\EE_{\pi',\sigma' }\bigl [Q^i(s', a',b') \bigr ]\biggl].
\#

In addition, in zero-sum Markov games, one can also define the \emph{greedy}  policy or \emph{equilibrium} policy of one team with respect to a value or action-value function, where the policy 
acts optimally based on the best response of the  opponent team. Specifically, given any $Q\in\cF(\cS\times\cA\times\cB,Q_{\max})$, the equilibrium joint policy   of Teams  $1$, 
denoted by $\pi_{Q}$, 
is defined as
\#
\pi_{Q}(\cdot\given s)=\argmax_{\pi'\in\cP(\cA)}\min_{\sigma'\in\cP(\cB)} \EE_{\pi',\sigma'} \big[ Q (s, a,b)\big], \label{equ:greedy_comp_pi}
\# 
which can be efficiently solved by solving a linear program \citep{osborne1994course}. 
With this definition, we can define an   operator $\cE^1$  
that   
generates the average equilibrium  policy with respect to a vector of $Q$-function for agents in Team $1$.
Specifically, with $\Qb=[Q^i]_{i\in\cN}$, we define   
$
\cE^1(\Qb)=N^{-1}\cdot\sum_{i\in\cN}\pi_{Q^i},
$
where $
\pi_{Q^i}$ 
is the equilibrium policy as defined in \eqref{equ:greedy_comp_pi}.


\section{Algorithms}\label{sec:algorithms}

In this section, we introduce the  decentralized batch MARL algorithms proposed for  both the cooperative and the  competitive settings.

\subsection{Decentralized Cooperative Batch MARL}\label{subsec:algorithms_collab}
Our  decentralized MARL algorithm is an extension of  the fitted-Q iteration algorithm for single-agent RL \citep{riedmiller2005neural}. In particular, all agents in a team have access to a  dataset $\cD=\{(s_t,\{a^i_t\}_{i\in\cN},s_{t+1})\}_{t=1,\cdots,T}$ that records the transition of the multi-agent system along the trajectory under a fixed joint behavior policy.  The local reward function, however,  is only available to each agent itself. 

At iteration $k$, each agent $i$ maintains  an estimate of the globally averaged $Q$-function denoted by $\tilde Q^i_{k}$.
Then,  agent $i$ samples local reward $\{r^i_t\}_{t=1,\cdots,T}$ along the trajectory $\cD$, and calculates the local target data $\{Y^i_t\}_{t=1,\cdots,T}$ following $
Y^i_t = r^i_t + \gamma \cdot \max _{a\in \cA} \tilde Q^i_k (s_{t+1}, a)$.
With the local data available, all agents hope to cooperatively find a common estimate of the global $Q$-function, by  
solving the following least-squares fitting problem 
\#\label{equ:fitted_least_squares}
\min_{f \in \cH}\quad \frac{1}{N}\sum_{i\in\cN}\frac{1}{T}\sum_{t=1}^T \bigl [ Y^i_t - f(s_t, a_t) \bigr ]^2,
\#
where $\cH\subseteq \cF(\cS\times\cA,Q_{\max})$ denotes the function class used for $Q$-function approximation.
The exact solution to \eqref{equ:fitted_least_squares}, denoted by $\tilde Q_{k+1}$,  can be viewed as an improved estimate of the global $Q$-function, which can be used to generate the targets for the next iteration $k+1$. However, in practice, since agents have to solve \eqref{equ:fitted_least_squares} in a distributed fashion, then with a finite number of iterations of any distributed optimization algorithms, the estimate at each agent may not reach exactly  consensual. Instead, each agent $i$ may have an estimate $\tilde Q^i_{k+1}$ that is different from the exact solution $\tilde Q_{k+1}$. This mismatch will then propagate to next iteration since agents can only use the local $\tilde Q^i_{k+1}$ to generate the target for iteration $k+1$. This is in fact one of the departures of our finite-sample analysis for MARL  from  the analysis for the single-agent setting \citep{munos2008finite,lazaric2010finite}.
After $K$ iterations, each agent $i$ finds the local greedy policy with respect to $\tilde Q_K^i$ and the local estimate of the global $Q$-function. To obtain a consistent joint greedy policy, all agents together output the average of their local greedy policies, i.e., output $\pi_{K}=\cG(\tilde \Qb_K)$.  The proposed  decentralized  algorithm for cooperative  MARL is summarized in Algorithm \ref{algo:fit_Q_Collab}.

When a parametric function class is considered, we denote $\cH$ by $\cH_{\Theta}$, where  
$\cH_{\Theta}=\{f(\cdot,\cdot;\theta)\in  \cF(\cS\times\cA,Q_{\max}): \theta\in\RR^{d}\}$. 
In this case, \eqref{equ:fitted_least_squares} becomes a vector-valued  optimization problem  with a separable objective function  among the agents. 
For notational convenience, we let  
$
g^i(\theta)={T}^{-1}\cdot\sum_{t=1}^T \bigl [ Y^i_t - f(s_t, a_t;\theta) \bigr ]^2,
$
and thus write   \eqref{equ:fitted_least_squares} can be written as 
\#\label{equ:para_fitted_least_squares}
\min_{\theta \in \RR^{d}}\quad \frac{1}{N}\sum_{i\in\cN}g^i(\theta).
\#

Since target data are distributed, i.e., each agent $i$ only has access to  its own $g^i(\theta)$, 
the agents need to exchange local information over the network $G_\tau$ to solve   \eqref{equ:para_fitted_least_squares}, which admits a  decentralized optimization algorithm. 
Note that problem \eqref{equ:para_fitted_least_squares} may be nonconvex with respect to $\theta$ when  $\cH_{\Theta}$ is a nonlinear function class, such deep neural networks, which makes computation of the exact minimum of \eqref{equ:para_fitted_least_squares} intractable. In addition, even if  $\cH_{\Theta}$ is a linear function class, which turns  \eqref{equ:para_fitted_least_squares} into  a  convex problem, with only a finite number of steps in practical implementation,  
decentralized optimization algorithms can at best   converge to a neighborhood of the global minimizer. 
Thus, the mismatch between $\tilde Q^i_k$ and $\tilde Q_k$ mentioned above is inevitable for our finite iteration analysis.

A rich family of decentralized/consensus optimization algorithms exists for   solving  the vector-valued optimization problem  \eqref{equ:para_fitted_least_squares}, and can all be incorporated into our algorithmic framework. 
For the  setting with a \emph{time-varying}  communication network, which is a general scenario in decentralized optimization,    several recent work \cite{nedic2017achieving,tatarenko2017non,hong2017stochastic} can apply. When the overall objective function is strongly-convex, \cite{nedic2017achieving} is the most advanced algorithm that is guaranteed to achieve geometric/linear  convergence rate to the best of our knowledge. Thus, we use the \emph{DIGing} algorithm proposed in \cite{nedic2017achieving} as a possible computational algorithm 
to solve \eqref{equ:para_fitted_least_squares}. 
In particular, each agent $i$ maintains two vectors in \emph{DIGing}, i.e., the solution estimate $\theta^i_l\in\RR^d$, and the average gradient estimate $\gamma^i_l\in\RR^d$, at iteration $l$. Each agent exchanges these two vectors with its  neighbors  over the time-varying network $\{G_l\}_{l\geq 0}$, weighted by  some  consensus matrix $\Cb_l=[c_l(i,j)]_{N\times N}$ that respects the topology of the graph $G_l$\footnote{Note that here we allow the communication graph to be time-varying even within each iteration $k$ of Algorithm \ref{algo:fit_Q_Collab}. Thus, we use $l$ as the time index used in the decentralized optimization algorithm instead of $\tau$, the general time index in Definition \ref{def:NMMDP}.}. Details on choosing the consensus matrix $\Cb_l$ will be provided in \S\ref{sec:theory}.
The updates of the \emph{DIGing} algorithm are  summarized in Algorithm \ref{algo:DIGing}. 
If $\cH_{\Theta}$ represents a linear function class, then \eqref{equ:para_fitted_least_squares} can be  strongly-convex under mild conditions.   
In this case, one can quantify the mismatch  between the global minimizer of \eqref{equ:para_fitted_least_squares} and the output of  Algorithm \ref{algo:DIGing} after a finite number of iterations, thanks to the linear convergence rate of the algorithm.

For general nonlinear function class $\cH_{\Theta}$, however, some existing  algorithms for nonconvex decentralized optimization \citep{zhu2013approximate,hong2016convergence,tatarenko2017non} can be applied. Nonetheless, in that case,  the mismatch between the algorithm output and the global minimizer is very difficult to quantify, which is a fundamental issue in general nonconvex optimization problems.

\begin{algorithm} 
\caption{Decentralized Fitted Q-Iteration Algorithm for Cooperative MARL} 
\label{algo:fit_Q_Collab} 
\begin{algorithmic} 
\STATE{{\textbf{Input:}} 
Function class $\cH$, trajectory data $\cD=\big\{\big(s_t,\{a^i_t\}_{i\in\cN},s_{t+1}\big)\big\}_{t=1,\cdots,T}$, number of  iterations $K$, number of samples $n$, the initial  estimator vector $\tilde \Qb_0=[\tilde Q^i_0]_{i\in\cN}$.} 
\FOR{$k = 0, 1, 2, \ldots, K-1$}
\FOR{agent $i\in\cN$} 
\STATE{Sample $r^i_t\sim R^i(\cdot\given s_t,a_t)$ and compute local target $Y^i_t = r^i_t + \gamma \cdot \max _{a\in \cA} \tilde Q^i_k (s_{t+1}, a)$, for all data $\big(s_t,\{a^i_t\}_{i\in\cN},s_{t+1}\big)\in\cD$. }
\ENDFOR
\STATE{Solve \eqref{equ:fitted_least_squares} for all agents $i\in\cN$, by decentralized optimization algorithms, e.g., by {\bf Algorithm  \ref{algo:DIGing}},   if $\cH$ is a parametric function class  $\cH_{\Theta}$}. 
\STATE{Update the estimator $\tilde Q^i_{k+1}$ for all agents $i\in\cN$. }
\ENDFOR
\STATE{{\textbf{Output:}} The vector of estimator $\tilde \Qb_{K}=[\tilde Q^i_{K}]_{i\in\cN}$ of $Q^*$ and joint  greedy policy $\pi_{K}=\cG(\tilde \Qb_{K})$.}
\end{algorithmic}
\end{algorithm}

\begin{algorithm}[!h] 
\caption{\emph{DIGing}: A Decentralized Optimization Algorithm for Solving \eqref{equ:para_fitted_least_squares}} 
\label{algo:DIGing} 
\begin{algorithmic} 
\STATE{{\textbf{Input:}} 
Parametric function class $\cH_\Theta$,  stepsize $\alpha>0$, initial consensus matrix $\Cb_0=[c_0(i,j)]_{N\times N}$, local target data $\{Y^i_t\}_{t=1,\cdots,T}$,  initial parameter $\theta^i_0\in\RR^{d}$, and initial vector $\gamma^i_0=\nabla g^i\big(\theta^i_0\big)$ for all agent $i\in\cN$. }
\FOR{$l = 0, 1, 2, \ldots, L-1$}  
\FOR{agent $i\in\cN$}
\STATE{$\theta^i_{l+1}=\sum_{j\in\cN}c_l(i,j)\cdot \theta^j_{l}-\alpha \cdot \gamma^i_l$} 
\STATE{$\gamma^i_{l+1}=\sum_{j\in\cN}c_l(i,j)\cdot \gamma^j_{l}+\nabla g^i\big(\theta^i_{l+1}\big)-\nabla g^i\big(\theta^i_{l}\big)$} 
\ENDFOR
\ENDFOR
\STATE{{\textbf{Output:}} The vector of functions $[\tilde Q^i]_{i\in\cN}$ with $\tilde Q^i=f\big(\cdot,\cdot;\theta^i_L\big)$ for all agent $i\in\cN$.}
\end{algorithmic}
\end{algorithm}

\subsection{Two-team Competitive  Batch MARL}
The proposed algorithm for two-team competitive MARL is also based on the fitted-Q iteration algorithm. Similarly, agents in both teams receive their rewards following the single trajectory data in $\cD=\{(s_t,\{a^i_t\}_{i\in\cN},\{b^j_t\}_{j\in\cM},s_{t+1})\}_{t=1,\cdots,T}$. 
To avoid repetition, in the sequel, we focus on the  update and analysis for  the  agents in Team $1$. Those for Team $2$ can be derived in a similar  way.

At iteration $k$, each agent $i\in\cN$ in Team $1$ maintains an estimate $\tilde Q^{1,i}_k$ of the globally averaged $Q$-function of its team.
With the local reward $r^{1,i}_t\sim R^{1,i}(s_t,a_t,b_t)$ available, agent $i$ computes the target $Y^{i}_t=r^{1,i}_t+\gamma\cdot \max_{\pi'\in\cP(\cA)}\min_{\sigma'\in\cP(\cB)}\EE_{\pi',\sigma'} \big[\tilde Q^{1,i}_k (s_{t+1}, a,b)\big]$. Then, all agents in Team $1$ aim to improve the estimate of the  minimax global $Q$-function 
by cooperatively  solving the following least-squares fitting problem
\#\label{equ:fitted_least_squares_2}
\min_{f^1 \in \cH}\quad \frac{1}{N}\sum_{i\in\cN}\frac{1}{T}\sum_{t=1}^T \bigl [ Y^{i}_t - f^1(s_t, a_t,b_t) \bigr ]^2.
\# 
Here with a slight abuse of notation,  we also use  $\cH\subset \cF(\cS\times\cA\times\cB,Q_{\max})$ to denote the function class for $Q$-function approximation.
Similar to the discussion in \S\ref{subsec:algorithms_collab}, with  decentralized algorithms that solve   \eqref{equ:fitted_least_squares_2},  agents in Team $1$ may not reach consensus on the estimate within a finite number of iterations. Thus, the output of the algorithm at iteration $k$ is a vector of $Q$-functions, i.e., $\Qb^1_k=[Q^{1,i}_{k}]_{i\in\cN}$, which will be used to compute the target at the next iteration $k+1$. Thus, the final output of the algorithm after $K$ iterations is the average greedy  policy with respect to the vector $\Qb^1_K$, i.e., $\cE^1(\Qb^1_K)$, which can  differ from the exact minimizer of \eqref{equ:fitted_least_squares_2} $\tilde Q^1_{K}$.
The  decentralized algorithm for competitive two-team MARL is summarized in Algorithm \ref{algo:fit_Q_Compet}.  
Moreover, if the function class $\cH$ is parameterized as $\cH_{\Theta}$, especially as a linear function class, then the mismatch between $\tilde Q^{1,i}_k$ and $\tilde Q^1_{k}$ after  a finite number of iterations of the distributed optimization algorithms, e.g., Algorithm \ref{algo:DIGing},  can be quantified.
We provide a detailed  discussion on this in \S\ref{subsubsec:colla_LFA}.

\begin{algorithm} [!h]
\caption{Decentralized Fitted Q-Iteration Algorithm for Two-team Competitive MARL} 
\label{algo:fit_Q_Compet} 
\begin{algorithmic} 
\STATE{{\textbf{Input:}} 
Function class $\cH$, trajectory data $\cD=\big\{\big(s_t,\{a^i_t\}_{i\in\cN},\{b^j_t\}_{j\in\cM},s_{t+1}\big)\big\}_{t=1,\cdots,T}$, number of  iterations $K$, number of samples $n$, the initial estimator vectors $\tilde \Qb^1_{0}=[\tilde Q^{1,i}_0]_{i\in\cN}$.
} 
\FOR{$k = 0, 1, 2, \ldots, K-1$}
\FOR{agent $i\in\cN$ in Team $1$} 
\STATE{Solve a matrix game 
	\$\max_{\pi'\in\cP(\cA)}\min_{\sigma'\in\cP(\cB)}\EE_{\pi',\sigma'} \big[\tilde Q^{1,i}_k (s_{t+1}, a,b)\big]
	\$ to obtain   equilibrium policies $(\pi'_k,\sigma'_k)$. }
\STATE{Sample $r^{1,i}_t\sim R^{1,i}(\cdot\given s_t,a_t,b_t)$ and compute local target 
\$
Y^{i}_t = r^{1,i}_t + \gamma \cdot \EE_{\pi'_k,\sigma'_k}\big[\tilde Q^{1,i}_k (s_{t+1}, a,b)\big],
\$ for all data $\big(s_t,\{a^i_t\}_{i\in\cN},\{b^j_t\}_{j\in\cM},s_{t+1}\big)\in\cD$. }
\ENDFOR
\STATE{Solve \eqref{equ:fitted_least_squares_2} for  agents in Team $1$, by decentralized optimization algorithms, e.g., by {\bf Algorithm  \ref{algo:DIGing}}, if $\cH$ is a  parametric function class  $\cH_{\Theta}$}. 
\STATE{Update the estimate $\tilde Q^{1,i}_{k+1}$ for all agents $i\in\cN$ in Team $1$.
}
\ENDFOR
\STATE{{\textbf{Output:}} The vector of estimates $\tilde \Qb^1_{K}=[\tilde Q^{1,i}_{K}]_{i\in\cN}$,  
and joint  equilibrium  policy $\pi_{K}=\cE^1(\tilde \Qb^1_{K})$.
}
\end{algorithmic}
\end{algorithm}
 
\section{Theoretical Results} \label{sec:theory}
In this section, we provide a finite-sample analysis on the algorithms proposed in  \S\ref{sec:algorithms}.
We first introduce several common assumptions for both the cooperative and competitive settings.

The function class $\cH$ used for action-value function approximation greatly influences the performance of the algorithm. Here we use the concept of \emph{pseudo-dimension} \citep{munos2008finite,antos2008fitted,antos2008learning} to capture the capacity of function classes, as assumed below.
 
\begin{assumption}[Capacity of Function Classes]\label{assum:capacity_func}
Let $V_{\cH^+}$  denote the  \emph{pseudo-dimension}  of a function class $\cH$, i.e., the VC-dimension of the  subgraphs of functions in $\cH$. Then the function class $\cH$ used in both Algorithm \ref{algo:fit_Q_Collab} and Algorithm \ref{algo:fit_Q_Compet} has  finite pseudo-dimension, i.e., $V_{\cH^+}<\infty$.
\end{assumption}

In our  decentralized setting, each agent may not have access to the simulators for  the overall MDP (Markov game) model transition. Thus, the data $\cD$ have to be collected from an actual  trajectory of the M-MDP (or the Markov game) with networked agents, under some joint  behavior policy of all agents. Note that the behavior  policy of other agents are not required to be known in order to generate such a sample path.  Our assumption regarding the sample path is as follows.

\begin{assumption}[Sample Path]\label{assum:sample_path}
	The sample path used in the cooperative (resp. the competitive team) setting, i.e.,  $\cD=\{(s_t,\{a^i_t\}_{i\in\cN},s_{t+1})\}_{t=1,\cdots,T}$ (resp.  $\cD=\{(s_t,\{a^i_t\}_{i\in\cN},\{b^j_t\}_{j\in\cM},s_{t+1})\}_{t=1,\cdots,T}$), are collected from a  sample path of the M-MDP (resp. Markov game) under 
	 some stochastic behavior  policy. Moreover, the process $\{(s_t,a_t)\}$ (resp. $\{(s_t,a_t,b_t)\}$) is  stationary, i.e., $(s_t,a_t)\sim \nu\in\cP(\cS\times\cA)$) (resp. $(s_t,a_t,b_t)\sim \nu\in\cP(\cS\times\cA\times\cB)$), and exponentially $\beta$-mixing\footnote{See Definition \ref{def:mixing} in Appendix \S \ref{sec:append_term_def} on $\beta$-mixing and exponentially $\beta$-mixing of a stochastic process.} with a rate defined by $(\overline{\beta},g,\zeta)$. 
\end{assumption}

Here we assume a mixing property of the random process along the sample path. Informally, this means that  the \emph{future} of the process depends weakly on the \emph{past}, which allows us to derive  tail inequalities for certain empirical processes. Note that 
Assumption \ref{assum:sample_path} is standard in the literature  for finite-sample analysis of batch RL using a single trajectory data \citep{antos2008learning,lazaric2010finite}. We also note that the  mixing coefficients do not need to be known when  implementing the proposed algorithms. 

In addition, we also make the following standard assumption on the \emph{concentrability coefficient} of the  M-MDP and the  zero-sum Markov game with networked agents, as in \cite{munos2008finite,antos2008learning}. The definitions  of concentrability coefficients follow from those in \cite{munos2008finite,perolat2015approximate}. For completeness, we provide the formal definitions in Appendix \S \ref{sec:append_term_def}. 

\begin{assumption}[Concentrability Coefficient]\label{assume:concentrability}
Let $\nu$ be the stationary distribution of the samples $\{(s_t,a_t)\}$ (resp. $\{(s_t,a_t,b_t)\}$) in $\cD$ from  the M-MDP (resp. Markov game) in Assumption \ref{assum:sample_path}. 
Let $\mu$  be a fixed distribution on $\cS \times \cA$ (resp. on $\cS \times \cA\times\cB$).  We assume that there exist constants $\phi_{\mu, \nu}^{\text{MDP}},\phi_{\mu, \nu}^{\text{MG}}< \infty$ such that 
  \#
  ( 1- \gamma)^{2}\cdot \sum_{m \geq 1} \gamma^{m-1} \cdot m \cdot \kappa^{\text{MDP}} ( m; \mu, \nu) \leq &~\phi_{\mu, \nu}^{\text{MDP}},\label{eq:assume:concentrability_MDP}\\
  ( 1- \gamma)^{2}\cdot \sum_{m \geq 1} \gamma^{m-1} \cdot m \cdot \kappa^{\text{MG}} ( m; \mu, \nu) \leq &~\phi_{\mu, \nu}^{\text{MG}},\label{eq:assume:concentrability_MG}
  \#
where  $\kappa^{\text{MDP}}$ and $\kappa^{\text{MG}}$ are concentrability coefficients for the networked  M-MDP and  zero-sum  Markov game as defined in  \S\ref{sec:append_term_def}, respectively.
\end{assumption}

The {concentrability coefficient}   measures the similarity between $\nu$ and  the distribution of the future states of the  M-MDP (or zero-sum Markov game) with networked agents when starting from $\mu$. 
The boundedness of the coefficient can be interpreted as the controllability of the underlying system, and holds in many regular   Markov games. It essentially means that the distribution of the sampled data has a good coverage over the states, leading to a valid change of measure from the distribution induced by the sampled data to some fixed performance evaluation distribution $\mu$. 
More interpretations on this quantity can be found  in \cite{munos2008finite,perolat2015approximate,chen2019information}. 

As mentioned in \S \ref{subsec:algorithms_collab}, in practice, at iteration $k$ of Algorithm \ref{algo:fit_Q_Collab}, with a finite number of iterations of the decentralized optimization algorithm, the output $\tilde Q^i_k$ is different from the exact  minimizer of  \eqref{equ:fitted_least_squares}. 
Such mismatches between the output of the decentralized optimization algorithm and the  exact solution to the fitting problem \eqref{equ:fitted_least_squares_2}  also exist in Algorithm \ref{algo:fit_Q_Compet}. 
Thus, 
we  make the following assumption on this one-step computation error in both cases. 
 
\begin{assumption}[One-step Decentralized  Computation Error]\label{assum:one_step_comp_error}
	At iteration $k$ of Algorithm \ref{algo:fit_Q_Collab}, the computation error from solving \eqref{equ:fitted_least_squares} is uniformly bounded, i.e., for each $i$,    there exists a certain $\epsilon^i_k>0$, such that  for any $(s,a)\in\cS\times\cA$, it holds that $|\tilde Q^i_k(s,a)-\tilde Q_k(s,a)|\leq \epsilon^i_k$, 
where $\tilde Q_k$
is the exact minimizer of \eqref{equ:fitted_least_squares} and $\tilde Q^i_k$ is the output of the decentralized optimization algorithm at agent $i\in\cN$. Similarly, at iteration $k$ of Algorithm \ref{algo:fit_Q_Compet}, there exists  certain $\epsilon^{1,i}_k>0$, 
such that  for any $(s,a,b)\in\cS\times\cA\times\cB$,  $|\tilde Q^{1,i}_k(s,a,b)-\tilde Q^1_k(s,a,b)|\leq \epsilon^{1,i}_k$, 
  where $\tilde Q^1_k$
is   the exact minimizers of \eqref{equ:fitted_least_squares_2},   $\tilde Q^{1,i}_k$ is 
 the output of the decentralized optimization algorithm at agent $i\in\cN$. 
\end{assumption}

The computation error, which for example is $|\tilde Q^i_k(s,a)-\tilde Q_k(s,a)|$ in the cooperative setting, generally comes  from two sources: 1) the  error caused by the  finiteness of the  number of iterations of the decentralized optimization algorithm; 2) the error caused by the nonconvexity of \eqref{equ:para_fitted_least_squares} with nonlinear parametric function class $\cH_\Theta$.
The error is always bounded for function class $\cH\subset \cF(\cS\times\cA,Q_{\max})$ with bounded absolute values. Moreover, the error can be further quantified when $\cH_\Theta$ is a linear function class, as to be detailed in  \S \ref{subsubsec:colla_LFA}.

\subsection{Decentralized Cooperative Batch  MARL}

Now we are ready to lay out the main results on the finite-sample error bounds for  decentralized cooperative MARL.

\begin{theorem}[Finite-sample Bounds for  Decentralized Cooperative MARL]\label{thm:main_thm_collab}
	Recall that $\{\tilde \Qb_k\}_{ 0\leq k \leq K} $ are the estimator vectors generated from  Algorithm \ref{algo:fit_Q_Collab}, and $\pi_{K}=\cG(\tilde \Qb_K)$ is the joint average greedy policy  with respect to the  estimate vector $\tilde \Qb_K$. Let $Q_{\pi_K}$ be the $Q$-function corresponding to $\pi_K$, $Q^*$ be the optimal $Q$-function, and  $\tilde R_{\max}=(1+\gamma)Q_{\max}+R_{\max}$. Also,  recall that $A=|\cA|,~N=|\cN|$, and $T=|\cD|$. Then,  under Assumptions   \ref{assum:capacity_func}-
	  \ref{assum:one_step_comp_error}, for any fixed distribution $\mu\in\cP(\cS\times\cA)$ and $\delta\in(0,1]$, there exist constants  $K_1$ and $K_2$ with 
 \$
 &K_1=K_1\big(V_{\cH^+}\log(T),\log(1/\delta),\log(\tilde R_{\max}),V_{\cH^+}\log(\overline{\beta})\big),\\ &K_2=K_2\big(V_{\cH^+}\log(T),V_{\cH^+}\log(\overline{\beta}),V_{\cH^+}\log[\tilde R_{\max}(1+\gamma)],V_{\cH^+}\log(Q_{\max}), V_{\cH^+}\log(A)\big),
 \$  
and $\Lambda_T(\delta)=K_1+K_2\cdot N$, such that  with probability at least $1-\delta$  
\$
\|  Q^* - Q_{\pi_K }\|_{\mu}\leq & \underbrace{C^{\text{MDP}}_{\mu, \nu}\cdot E(\cH)}_{\text{Approximation error}}+\underbrace{C^{\text{MDP}}_{\mu, \nu}\cdot\bigg\{\frac{\Lambda_T(\delta/K)[\Lambda_T(\delta/K)/b\vee 1]^{1/\zeta}}{T/(2048\cdot \tilde R^4_{\max})}\bigg\}^{1/4}}_{\text{Estimation error}}\notag\\
& \quad +\underbrace{\sqrt{2}\gamma\cdot C^{\text{MDP}}_{\mu, \nu}\cdot\overline{\epsilon}+ \frac{2\sqrt{2}\gamma}{1-\gamma}\cdot \overline{\epsilon}_K}_{\text{Decentralized computation error}}+{\frac{4\sqrt{2}\cdot Q_{\max} }{(1- \gamma)^2}   \cdot \gamma^{K/2}},
\$
where $\overline{\epsilon}_{K}=[N^{-1}\cdot\sum_{i\in\cN}(\epsilon^i_{K})^2]^{1/2}$,  and 
\$
C^{\text{MDP}}_{\mu, \nu}=\frac{4 \gamma \cdot\big(\phi^{\text{MDP}}_{\mu,\nu}\big)^{1/2} }{\sqrt{2}(1- \gamma)^2},\quad E(\cH)=\sup_{\Qb\in\cH^N}\inf_{f\in\cH}\|f-\tcT{\Qb}\|_{\nu},\quad  \overline{\epsilon}=\max_{0\leq k\leq K-1}~\bigg[\frac{1}{N}\sum_{i\in\cN}(\epsilon^i_{k})^2\bigg]^{1/2}.
\$
Moreover, $\phi^{\text{MDP}}_{\mu, \nu}$, given in \eqref{eq:assume:concentrability_MDP},  is a constant that only depends on the distributions $\mu$ and $\nu$. 
\end{theorem}
\begin{proof}
The proof is mainly based on the  following theorem that quantifies the propagation of one-step errors as Algorithm \ref{algo:fit_Q_Collab} proceeds.

 \begin{theorem}[Error Propagation for  Decentralized Cooperative MARL] \label{thm:err_prop_collab}
Under Assumptions  \ref{assume:concentrability} and \ref{assum:one_step_comp_error}, for any fixed distribution $\mu\in\cP(\cS\times\cA)$, we have 
\$
\|  Q^* - Q_{\pi_K }\|_{\mu}\leq \underbrace{C^{\text{MDP}}_{\mu, \nu}\cdot \|\varrho\|_\nu}_{\text{Statistical error}} +\underbrace{\sqrt{2}\gamma\cdot C^{\text{MDP}}_{\mu, \nu}\cdot\overline{\epsilon}+ \frac{2\sqrt{2}\gamma}{1-\gamma}\cdot \overline{\epsilon}_K}_{\text{Decentralized computation error}}+{\frac{4\sqrt{2}\cdot Q_{\max} }{(1- \gamma)^2}   \cdot \gamma^{K/2}},
\$
where we define 
\$
\|\varrho\|_\nu=\max_{k\in[K]}~\|\tcT\tilde{\Qb}_{k-1} - \tilde Q_k\|_{\nu},
\$
and $\overline{\epsilon}_{K},C^{\text{MDP}}_{\mu, \nu}$, and $\overline{\epsilon}$ are as defined in Theorem \ref{thm:main_thm_collab}.  
\end{theorem} 

Theorem \ref{thm:err_prop_collab}  shows that both the one-step statistical error and the decentralized computation error will propagate, which constitute  the fundamental error that will not vanish even when the iteration  $K\to\infty$.   See \S\ref{proof:thm:err_prop} for the proof of Theorem \ref{thm:err_prop_collab}.  
To obtain the main results, now it suffices to characterize the one-step statistical error $\|\varrho\|_{\nu}$. The following theorem  establishes a high probability bound for this  statistical error. 

 \begin{theorem}[One-step Statistical Error  for  Decentralized Cooperative MARL] \label{thm:err_one_step_approx} 
 Let $\Qb=[Q^i]_{i\in\cN}\in\cH^N$ be a vector of real-valued random functions  (may not be  independent from the sample path), let $(s_t,\{a^i_t\}_{i\in\cN},s_{t+1})$ be the samples from the trajectory data $\cD$ and $\{r^i_t\}_{i\in\cN}$ be the rewards sampled from $(s_t,\{a^i_t\}_{i\in\cN},s_{t+1})$. We also define $Y^i_t=r^i_t + \gamma \cdot \max _{a\in \cA}  Q^i (s_{t+1}, a)$, and define $f'$  by
 \#\label{equ:fitting_obj}
f'\in\argmin_{f \in \cH}~ \frac{1}{N}\sum_{i\in\cN}\frac{1}{T}\sum_{t=1}^T \bigl [ Y^i_t - f(s_t, a_t) \bigr ]^2.
 \#
 Then, under Assumptions    \ref{assum:capacity_func}  and   \ref{assum:sample_path}, for  $\delta\in(0,1]$, $T\geq 1$, there exists some  $\Lambda_T(\delta)$ as defined in Theorem \ref{thm:main_thm_collab}, such that 
 with probability at least $1-\delta$, 
 \#\label{equ:statistic_error_collab}
 \|f'-\tcT{\Qb}\|^2_{\nu}\leq \inf_{f\in\cH} \|f-\tcT{\Qb}\|^2_{\nu}+\sqrt{\frac{\Lambda_T(\delta)[\Lambda_T(\delta)/b\vee 1]^{1/\zeta}}{T/(2048\cdot \tilde R^4_{\max})}}.
 \#
\end{theorem}

The proof of Theorem \ref{thm:err_one_step_approx} is provided in \S\ref{proof:thm:err_one_step_approx}. 
Similar to the existing results in the single-agent setting (e.g., Lemma $10$ in \cite{antos2008learning}), the one-step statistical error consists of two parts, the approximation error that depends on the richness of the function class $\cH$, and the estimation error that vanishes with the number of samples $T$.  
 
By replacing $\Qb$ by $\tilde \Qb_{k-1}$ and $f'$ by $\tilde Q_k$, the results in Theorem \ref{thm:err_one_step_approx}  can characterize the one-step statistical error $\|\tcT\tilde{\Qb}_{k-1} - \tilde Q_k\|_{\nu}$. Together with Theorem \ref{thm:err_prop_collab}, we  conclude the main results in Theorem \ref{thm:main_thm_collab}. 
\end{proof}

Theorem \ref{thm:main_thm_collab} establishes  a high probability  bound on the quality of the output policy $\pi_K$ obtained from Algorithm  \ref{algo:fit_Q_Collab} after $K$ iterations. Here we use the $\mu$-weighted norm of the difference between $Q^*$ and $Q_{\pi_K}$ as the performance metric. Theorem \ref{thm:main_thm_collab} shows that the finite-sample error of decentralized MARL is precisely controlled by three fundamental   terms: 
1) the approximation error that depends on the richness of the function class $\cH$, i.e., how well $\cH$ preserves the average Bellman operator $\tcT$;  
2) the estimation error incurred by the fitting step \eqref{algo:fit_Q_Collab}, which vanishes with increasing number of samples $T$; 3) the  computation error in solving the least-squares problem \eqref{equ:fitted_least_squares} in a decentralized way with a finite number of updates.  
Note that the estimation error, after some simplifications and suppression of constant and logarithmic terms, { has the form
\#\label{equ:bound_clean_form_collab}
\bigg\{\frac{[V_{\cH^+}(N+1)\log(T)+V_{\cH^+}N\log(A)+\log(K/\delta)]^{1+1/\zeta}}{T}\bigg\}^{1/4}.
\#}
Compared with the existing results in the single-agent setting, e.g., \cite[Theorem $4$]{antos2008learning},  our results have  an additional dependence on $O(N\log(A))$,  where $N=|\cN|$ is the number of agents in the team and $A=|\cA|$ is cardinality  of the joint action set.
This dependence on $N$ is due to the fact that   the target data used  in the fitting step are collections of local target data from $N$ agents; while the dependence on  $\log(A)$ characterizes the difficulty of estimating $Q$-functions, each of which has $A$ choices to find the maximum given any state $s$. Similar terms of order   $\log(A)$   
also show up in the single-agent setting \citep{antos2008fitted,antos2008learning}, which is induced by the capacity of the action space. 
 In addition, a close examination of the proof shows that the \emph{effective dimension} \citep{antos2008learning}  is $(N+1)V_{\cH^+}$, which is because we allow $N$ agents to have their own  estimates of $Q$-functions, each of which lies in the function class $\cH$ with pseudo-dimension $V_{\cH^+}$. We note that it is possible to sharpen the dependence of the  rate and the {effective dimension} on $N$  via different proof techniques from here, which is left as our future work. 

 
\subsection{Two-team Competitive Batch MARL}
We now  establish  the finite-sample error bounds for the  decentralized competitive batch MARL.

\begin{theorem}[Finite-sample Bounds for  Decentralized Two-team Competitive MARL]\label{thm:main_thm_compet}
	Recall that $\{\tilde \Qb^1_k\}_{ 0\leq k \leq K} $ are the estimator vectors obtained by Team $1$ via   Algorithm \ref{algo:fit_Q_Compet}, and $\pi_{K}=\cE^1(\tilde \Qb^1_K)$ is the joint average equilibrium  policy  with respect to the  estimate vector $\tilde \Qb^1_K$. Let $Q_{\pi_K}$ be the optimal $Q$-function corresponding to $\pi_K$, $Q^*$ be the minimax  $Q$-function of the game, and   $\tilde R_{\max}=(1+\gamma)Q_{\max}+R_{\max}$. Also,  recall that $A=|\cA|,~B=|\cB|,~N=|\cN|$, and $T=|\cD|$. Then,  under Assumptions   \ref{assum:capacity_func}-
\ref{assum:one_step_comp_error}, for any fixed distribution $\mu\in\cP(\cS\times\cA\times\cB)$ and $\delta\in(0,1]$, there exist constants $K_1$ and $K_2$ with \footnote{Note that  the constants $K_1,K_2$  here may have different values from those in  Theorem \ref{thm:main_thm_collab}.}
 \$
 &K_1=K_1\big(V_{\cH^+}\log(T),\log(1/\delta),\log(\tilde R_{\max}),V_{\cH^+}\log(\overline{\beta})\big),\\ &K_2=K_2\big(V_{\cH^+}\log(T),V_{\cH^+}\log(\overline{\beta}),V_{\cH^+}\log[\tilde R_{\max}(1+\gamma)],V_{\cH^+}\log(Q_{\max}), V_{\cH^+}\log(AB)\big),
 \$  
 {and $\Lambda_T(\delta)=K_1+K_2\cdot N$, such that  with probability at least $1-\delta$  
\$
\|  Q^* - Q_{\pi_K }\|_{\mu}\leq & \underbrace{C^{\text{MG}}_{\mu, \nu}\cdot E(\cH)}_{\text{Approximation error}}+\underbrace{C^{\text{MG}}_{\mu, \nu}\cdot\bigg\{\frac{\Lambda_T(\delta/K)[\Lambda_T(\delta/K)/b\vee 1]^{1/\zeta}}{T/(2048\cdot \tilde R^4_{\max})}\bigg\}^{1/4}}_{\text{Estimation error}}\notag\\
& \quad +\underbrace{\sqrt{2}\gamma\cdot C^{\text{MG}}_{\mu, \nu}\cdot\overline{\epsilon}^1+ \frac{2\sqrt{2}\gamma}{1-\gamma}\cdot \overline{\epsilon}^1_K}_{\text{Decentralized computation error}}+{\frac{4\sqrt{2}\cdot Q_{\max} }{(1- \gamma)^2}   \cdot \gamma^{K/2}},
\$
where $\overline{\epsilon}^1_{K}=[N^{-1}\cdot\sum_{i\in\cN}(\epsilon^{1,i}_{K})^2]^{1/2}$,  and 
\$
C^{\text{MG}}_{\mu, \nu}=\frac{4 \gamma \cdot\big(\phi^{\text{MG}}_{\mu,\nu}\big)^{1/2} }{\sqrt{2}(1- \gamma)^2},\quad E(\cH)=\sup_{\Qb\in\cH^N}\inf_{f\in\cH}\|f-\tcT{\Qb}\|_{\nu},\quad  \overline{\epsilon}^1=\max_{0\leq k\leq K-1}~\bigg[\frac{1}{N}\sum_{i\in\cN}(\epsilon^{1,i}_{k})^2\bigg]^{1/2}.
\$
Moreover, $\phi^{\text{MG}}_{\mu, \nu}$, given in \eqref{eq:assume:concentrability_MG},  is a constant that only depends on the distributions $\mu$ and $\nu$.}
\end{theorem}
\begin{proof}
Note that by slightly abusing the  notation, the operator $\tcT$ here follows the definition in \eqref{equ:operator_average_T_opt_comp}.
The proof is established mainly upon the following error propagation bound, whose proof is provided in \S\ref{proof:thm:err_prop_compet}.  
{
 \begin{theorem}[Error Propagation for  Decentralized Two-team Competitive MARL] \label{thm:err_prop_compet}
Under Assumptions  \ref{assume:concentrability} and \ref{assum:one_step_comp_error}, for any fixed distribution $\mu\in\cP(\cS\times\cA\times\cB)$, we have 
\$
\|  Q^* - Q_{\pi_K }\|_{\mu}\leq \underbrace{C^{\text{MG}}_{\mu, \nu}\cdot \|\varrho^1\|_\nu}_{\text{Statistical error}}+\underbrace{\sqrt{2}\gamma\cdot C^{\text{MG}}_{\mu, \nu}\cdot\overline{\epsilon}^1+ \frac{2\sqrt{2}\gamma}{1-\gamma}\cdot \overline{\epsilon}^1_K}_{\text{Decentralized computation error}}+{\frac{4\sqrt{2}\cdot Q_{\max} }{(1- \gamma)^2}   \cdot \gamma^{K/2}},
\$
where we define 
\$
 \|\varrho^1\|_\nu=\max_{k\in[K]}~\|\tcT\tilde{\Qb}^1_{k-1} - \tilde Q^1_k\|_{\nu}, 
\$
and $\overline{\epsilon}^1_{K},C^{\text{MG}}_{\mu, \nu}$, and $\overline{\epsilon}^1$ are as defined in Theorem \ref{thm:main_thm_compet}.
\end{theorem}
}

Similar to the error propagation results in Theorem \ref{thm:err_prop_collab}, as iteration $K$ increases, the fundamental error of the $Q$-function under policy $\pi_K$ is bounded by two terms, the statistical error and the decentralized computation error. The former is characterized by the following theorem.

 \begin{theorem}[One-step Statistical Error  for  Decentralized Two-team Competitive MARL] \label{thm:err_one_step_approx_2} 
 Let $\Qb^1=[Q^{1,i}]_{i\in\cN}\in\cH^N$ be a vector of real-valued random functions  (may not be  independent from the sample path), let $(s_t,\{a^i_t\}_{i\in\cN},\{b^j_t\}_{j\in\cM}, s_{t+1})$ be the samples from the trajectory data $\cD$ and $\{r^{1,i}_t\}_{i\in\cN}$ be the rewards sampled  by agents in Team $1$ from $(s_t,\{a^i_t\}_{i\in\cN},\{b^j_t\}_{j\in\cM},s_{t+1})$. We also define $Y^{i}_t=r^{1,i}_t + \gamma\cdot \max_{\pi'\in\cP(\cA)}\min_{\sigma'\in\cP(\cB)}\EE_{\pi',\sigma'} \big[Q^{1,i} (s_{t+1}, a,b)\big]$, and define $f'$  as the minimizer of \eqref{equ:fitting_obj}. 
 Then, under Assumptions    \ref{assum:capacity_func}  and   \ref{assum:sample_path}, for  $\delta\in(0,1]$, $T\geq 1$, there exists some  $\Lambda_T(\delta)$ as defined in Theorem \ref{thm:main_thm_compet}, such that  the bound that has the form of  \eqref{equ:statistic_error_collab} holds. 
\end{theorem}

By substituting the results of  Theorem \ref{thm:err_one_step_approx_2} into Theorem \ref{thm:err_prop_compet}, we obtain the desired results and conclude  the proof. 
\end{proof}

Theorem \ref{thm:main_thm_compet} characterizes the quality  of the output policy $\pi_K$ for Team $1$ obtained from Algorithm \ref{algo:fit_Q_Compet} in a high probability sense. 
We use the same  performance metric, i.e., the weighted norm of the difference between  $Q_{\pi_K}$ and the minimax action-value $Q^*$, as in the literature  \cite{patek1997stochastic,perolat2015approximate}. 
For brevity, we only include the error bound for Team $1$ as in \cite{perolat2015approximate}, and note that the bound for Team $2$ can be obtained immediately by changing the order of the $\max$ and $\min$ operators and some notations in the proof. 

Similar to the results for the cooperative setting, the error  bound is composed of three main terms, the inherent approximate error depending on the function class $\cH$, the estimation error vanishing with the increasing number of samples, and the decentralized computation error.  The simplified estimation error has a nearly identical form as in \eqref{equ:bound_clean_form_collab}, except that the { dependence on $N\log(A)$ is replaced by $N\log(AB)$. } Moreover, the effective dimension remains $(N+1)V_{\cH^+}$ as in \eqref{equ:bound_clean_form_collab}. These observations further substantiate   the discussions right  after Theorem \ref{thm:main_thm_collab}, i.e.,  the dependence on  $N$ is due to the local target data distributed at $N$ agents in the team, and the dependence on $\log(AB)$ follows from the capacity of the action space. Also note that the number of agents $M=|\cM|$ in Team $2$ does not show up in the bound, thanks to the zero-sum assumption on the  rewards.  The decentralized computation error is controlled by the mean-square average of the one-step decentralized computation error defined in Assumption \ref{assum:one_step_comp_error}. In contrast, in  a centralized setting, where the two teams are coordinated by two central controllers and the problem  reduces to a two-player zero-sum Markov game, this decentralized computation error would disappear and the bound in Theorem \ref{thm:main_thm_compet} then provides, for the first time, a finite-sample result for this reduced setting. In other words, this decentralized computation error   precisely characterizes the decentralized nature of the problem and the algorithm.

\subsection{Using Linear Function Approximation}\label{subsubsec:colla_LFA}
Now we provide more concrete finite-sample bounds for both settings  above when a linear function class for $Q$-function approximation is used.  
In particular, we quantify the one-step computation error bound assumed in  Assumption \ref{assum:one_step_comp_error}, after $L$ iterations of the decentralized optimization algorithm that solves  \eqref{equ:fitted_least_squares} or \eqref{equ:fitted_least_squares_2}.  
We first make the following assumption on the features of the linear  function class used in both settings.

\begin{assumption}\label{assum:linear_features}
For cooperative MARL, 
	the function class $\cH_{\Theta}$ used in Algorithm \ref{algo:fit_Q_Collab} is a parametric linear function class, i.e.,  $\cH_\Theta\subset  \cF(\cS\times\cA,Q_{\max})$ and $\cH_\Theta=\{f(s,a;\theta)=\theta^\top\varphi(s,a): \theta\in\RR^{d}\}$, where for any $(s,a)\in\cS\times\cA$, $\varphi(s,a)\in\RR^d$  is the feature vector. Moreover, let $\Mb^{\text{MDP}}=T^{-1}\cdot \sum_{t=1}^T\varphi(s_t,a_t)\varphi^\top(s_t,a_t)$ with $\{(s_t,a_t)\}_{t\in[T]}$ being samples from the data set $\cD$, then 
	the matrix $\Mb^{\text{MDP}}$ is full rank. Similarly, for two-team competitive MARL, the function class $\cH_\Theta\subset  \cF(\cS\times\cA\times\cB,Q_{\max})$ used in Algorithm \ref{algo:fit_Q_Compet} is a parametric linear function class, with $\varphi(s,a,b)\in\RR^d$  being  the feature vector for any $(s,a,b)\in\cS\times\cA\times\cB$. 
	Moreover, letting $\Mb^{\text{MG}}=T^{-1}\cdot \sum_{t=1}^T\varphi(s_t,a_t,b_t)\varphi^\top(s_t,a_t,b_t)$ with $\{(s_t,a_t,b_t)\}_{t\in[T]}$ being samples from the data set $\cD$, we have that  
	the matrix $\Mb^{\text{MG}}$ is full rank.
\end{assumption}

Since $f\in\cH_\Theta$ has bounded absolute value $Q_{\max}$, 
Assumption \ref{assum:linear_features} implies that the norm of the features are uniformly bounded. 
The second assumption on the rank of the matrices $\Mb^{\text{MDP}}$ and $\Mb^{\text{MG}}$ ensures that the least-squares problems \eqref{equ:fitted_least_squares} and \eqref{equ:fitted_least_squares_2} are strongly-convex, which enables the  \emph{DIGing} algorithm to achieve  the desirable geometric convergence rate over time-varying communication networks. We note that this assumption can be readily satisfied in practice. Let $\varphi(s,a)=[\varphi_1(s,a),\cdots,\varphi_d(s,a)]^\top$. Then,    the functions $\{\varphi_1(s,a),\cdots,\varphi_d(s,a)\}$ (or vectors of dimension $|\cS|\times|\cA|$ if the state space $\cS$ is finite), are required to be linearly independent, in the conventional  RL with linear function approximation \citep{tsitsiklis1997analysis,geramifard2013tutorial}. Thus, with a rich enough data set $\cD$, it is not difficult to find $d\ll T$ samples from $\cD$, such that the  matrix $[\varphi(s_1,a_1),\cdots,\varphi(s_d,a_d)]^\top$ has full rank, i.e., rank $d$. In this case, with some  algebra (see Lemma \ref{lemma:append_rank_1} in Appendix \S\ref{sec:append_proofs}), one can  show that the matrix $\Mb^{\text{MDP}}$ is also full-rank. The similar argument applies to  $\Mb^{\text{MG}}$. These arguments justify the rationale  behind  Assumption \ref{assum:linear_features}.

Moreover, we make the following assumption on the time-varying consensus matrix $\Cb_l$ used in the \emph{DIGing} algorithm (see also Assumption 1 in \cite{nedic2017achieving}). 

\begin{assumption}[Consensus Matrix Sequence $\{\Cb_l\}$]\label{assum:consensus_mat}
For any $l=0,1,\cdots$, the consensus matrix $\Cb_l=[c_l(i,j)]_{N\times N}$ satisfies the following relations:

1) (Decentralized property) If $i\neq j$, and edge $(j,i)\notin E_l$, then $c_l(i,j)=0$;

2) (Double stochasticity) $\Cb_l\bm1=\bm1$ and $\bm1^\top \Cb_l=\bm1^\top$;

3) (Joint spectrum property) There exists a positive integer $B$ such that 
\$
 \chi<1,~~\text{where}~~\chi=\sup_{l\geq B-1}\sigma_{\max}\Big\{\Cb_l\Cb_{l-1}\cdots\Cb_{l-B+1}-\frac{1}{N}\bm1\bm1^\top\Big\} \text{~~for all~~} l=0,1,\cdots,
\$
and $\sigma_{\max}(\cdot)$ denotes the largest singular value of a matrix.
\end{assumption}

Assumption \ref{assum:consensus_mat} is standard and can be satisfied by many matrix sequences used in decentralized optimization. Specifically, condition 1) states the restriction on the  physical connectivity  of the network; 2) ensures the convergent vector is consensual for all agents; 3) concerns the connectivity of the time-varying graph $\{G_l\}_{l\geq 0}$. See more  discussions on this assumption  in \cite[Section 3]{nedic2017achieving}.

Now we are ready to present the following corollary on the sample and iteration complexity of Algorithms \ref{algo:fit_Q_Collab} and \ref{algo:fit_Q_Compet}, when Algorithm \ref{algo:DIGing} and   linear function approximation is used.  

\begin{corollary}[Sample and Iteration Complexity with  Linear Function Approximation]\label{coro:complex_LFA}
Suppose Assumptions    \ref{assum:capacity_func}- \ref{assum:one_step_comp_error}, and \ref{assum:linear_features}-\ref{assum:consensus_mat} hold, and  Algorithm \ref{algo:DIGing} is used in the fitting steps  \eqref{equ:fitted_least_squares} and \eqref{equ:fitted_least_squares_2}  for decentralized optimization. Let $\pi_K$ be the output policy of Algorithm \ref{algo:fit_Q_Collab}, then for any $\delta\in(0,1]$,  $\epsilon>0$, and  fixed distribution $\mu\in\cP(\cS\times\cA)$, there exist integers $K,T$, and $L$, where $K$  is linear in $\log(1/\epsilon)$, $\log[1/(1-\gamma)]$,  and $\log(Q_{\max})$; $T$ is polynomial  in $1/\epsilon$, $\gamma/(1-\gamma)$, $1/{\tilde R_{\max}}$, $\log(1/\delta)$, $\log(\overline{\beta})$, and $N\log(A)$; and $L$ is linear in $\log(1/\epsilon)$, $\log[\gamma/(1-\gamma)]$, such that 
\$
\|  Q^* - Q_{\pi_K }\|_{\mu}\leq C^{\text{MDP}}_{\mu, \nu}\cdot E(\cH)+\epsilon
\$
holds with probability at least $1-\delta$. If $\pi_K$ is the output  policy of Team $1$ from Algorithm \ref{algo:fit_Q_Compet}, the same arguments also hold for any fixed distribution $\mu\in\cP(\cS\times\cA\times\cB)$, but with $T$ being polynomial in $N\log(AB)$ instead of $N\log(A)$, and $C^{\text{MDP}}_{\mu, \nu}$ replaced by $C^{\text{MG}}_{\mu, \nu}$.
\end{corollary}

Corollary \ref{coro:complex_LFA}, whose proof is  deferred to \S\ref{proof:coro:LFA}, is established upon Theorems \ref{thm:main_thm_collab} and \ref{thm:main_thm_compet}.
It shows that the proposed Algorithms \ref{algo:fit_Q_Collab} and \ref{algo:fit_Q_Compet} are efficient with the aid of Algorithm \ref{algo:DIGing} under some mild  assumptions, in the sense that finite number of samples and iterations, which  scale at most polynomially with the parameters of the problem, are needed  to achieve arbitrarily small $Q$-value errors, provided the inherent approximation error is small. 

We note that if the full-rank condition in Assumption \ref{assum:linear_features} does not hold,  the fitting problems   \eqref{equ:fitted_least_squares} and \eqref{equ:fitted_least_squares_2} are simply convex. Then, over time-varying communication network, it is also possible to establish convergence rate of $O(1/l)$ using the  proximal-gradient consensus algorithm  \citep{hong2017stochastic}. We will skip the detailed discussion on various decentralized optimization algorithms since it is beyond the scope of this paper.

\section{Proofs of the Main Results} \label{sec:proof}
In this section, we provide proofs for the main results presented in \S\ref{sec:theory}.

\subsection{Proof of Theorem \ref{thm:err_prop_collab}}    \label{proof:thm:err_prop}
  \begin{proof}
 We start our proof by introducing some notations. For any $k \in \{1, \ldots, K\}$, we   define 
 \#\label{equ:rho_eta_k}
 \varrho_{k} = \tcT\tilde{\Qb}_{k-1} - \tilde Q_k, \quad \eta_k = \cT\tilde {Q}_{k-1} -\tcT\tilde{\Qb}_{k-1},
 \#  
 where recall that $\tilde Q_k$ is the exact minimizer of the least-squares  problem \eqref{equ:fitted_least_squares}, and  $\tilde{\Qb}_{k}=[\tilde Q^i_k]_{i\in\cN}$ are the output of $Q$-function estimators at each agent  from Algorithm \ref{algo:fit_Q_Collab}, both at iteration $k$. Also   by definition  of $\tcT$ in \eqref{equ:operator_average_T_opt}, the  expression  $\tcT\tilde{\Qb}_{k}$ has the form 
 \$
(\tcT\tilde{\Qb}_{k})(s,a) =\overline{r}(s, a) + \gamma\cdot  \EE_{s' \sim P(\cdot  \given s,a)}\biggl[\frac{1}{N}\cdot\sum_{i=1}^N\max_{a' \in \cA} \tilde Q^i_k(s', a') \biggr].
\$
The term $\rho_k$ captures the one-step approximation error of the fitting  problem \eqref{equ:fitted_least_squares}, which is caused by the finite number of samples used, and the capacity along with the expressive power of the function class $\cH$. The term  $\eta_k$ captures the computational error of the decentralized optimization algorithm after a finite number of updates. 
Also, we denote by  $\pi_k$ the average greedy policy obtained from the estimator vector $\tilde{\Qb}_{k}$, i.e., $\pi_k=\cG(\tilde{\Qb}_{k})$.

The proof mainly  contains the following three  steps. 
  
\vskip4pt
 {\noindent \bf Step (i):}  First, we establish a recursion between the errors of the exact minimizers of \eqref{equ:fitted_least_squares} at consecutive iterations with respect to the optimal $Q$-function, i.e., the recursion between  $Q^* - \tilde{Q}_{k+1}$ and $Q^* - \tilde{Q}_k$. To this end, we first split  $Q^* - \tilde{Q}_{k+1}$ as follows by the  definitions of $\varrho_{k+1}$	and $\eta_{k+1}$
 \#\label{equ:one_step1}
	  Q^* - \tilde{Q}_{k+1}  
	&=  Q^* - \big(\tcT\tilde{\Qb}_{k} - \varrho_{k+1}\big)  
	= \big(Q^* - \cT_{\pi^*} \tilde{Q}_k\big) +   \big(\cT_{\pi^*} \tilde{Q}_k - \cT\tilde {Q}_k \big)+\eta_{k+1}+ \varrho_{k+1},
	\#
where we denote by $\pi^*$ the greedy policy  with respect to $Q^*$.

First  note that for any $s' \in \cS$ and $a' \in \cA$, $\max_{a'} \tilde{Q}_k(s', a') \geq \tilde{Q}_k(s', a')$, which yields
\$
	(\cT  \tilde{Q}_k)(s,a) & = \overline{r}(s,a) + \gamma \cdot \EE_{s' \sim P(\cdot  \given s,a)}\bigl[\max_{a'} \tilde Q_k(s', a')\bigr]\\
	& \geq \overline{r}(s, a) + \gamma \cdot  \EE_{s' \sim P(\cdot  \given s,a),a' \sim \pi^*(\cdot \given s')}\bigl[\tilde Q_k(s', a') \bigr] = (\cT_{\pi^*} \tilde Q_k)(s,a).
\$
Thus, it follows that $\cT_{\pi^*} \tilde{Q}_k \leq  \cT\tilde {Q}_k $. Combined with \eqref{equ:one_step1}, we further obtain
\#\label{equ:one_step2}
	  Q^* - \tilde{Q}_{k+1}   
	\leq  \big(Q^* - \cT_{\pi^*} \tilde{Q}_k\big) +\eta_{k+1}+ \varrho_{k+1}.
	\#

Similarly, we can establish a lower bound for $Q^* - \tilde Q_{k+1}$ based on $Q^* - \tilde Q_{k}$. Note that 
\$
	  Q^* - \tilde{Q}_{k+1}   
	= \big(Q^* - \cT_{\tilde{\pi}_k}  Q^*\big) +   \big(\cT_{\tilde{\pi}_k}  Q^* - \cT\tilde {Q}_k \big)+\eta_{k+1}+ \varrho_{k+1},
	\$
	where $\tilde{\pi}_k$ is the  greedy policy with respect to $\tilde Q_k$, i.e., $\cT \tilde Q_k=\cT_{\tilde{\pi}_k}\tilde Q_k$. Since $Q^*=TQ^*\geq \cT_{\tilde{\pi}_k}  Q^*$, it holds  that
	\#\label{equ:one_step4}
  Q^* - \tilde{Q}_{k+1}  \geq  \big(\cT_{\tilde{\pi}_k}  Q^* - \cT_{\tilde{\pi}_k}\tilde {Q}_k \big)+\eta_{k+1}+ \varrho_{k+1} .
	\#
By combining \eqref{equ:one_step2} and \eqref{equ:one_step4}, we obtain that  for any $k \in \{0,\ldots, K-1\}$,
	  \#\label{equ:one_step_err}
 \big(\cT_{\tilde{\pi}_k}  Q^* - \cT_{\tilde{\pi}_k}\tilde {Q}_k \big)+\eta_{k+1}+ \varrho_{k+1} \leq 	Q^* - \tilde{Q}_{k+1} \leq \big(T_{\pi^*}Q^* - \cT_{\pi^*} \tilde{Q}_k\big) +\eta_{k+1}+ \varrho_{k+1}. 
  \#
\eqref{equ:one_step_err} shows that one  can both upper and lower bound the error $Q^* - \tilde{Q}_{k+1}$ using terms related to $Q^* - \tilde{Q}_{k}$, plus  two error terms $\eta_{k+1}$ and $\varrho_{k+1}$ as defined in \eqref{equ:rho_eta_k}. 
With the definition of $P_{\pi}$  in \eqref{eq:operator_P}, we can write \eqref{equ:one_step_err} in a more compact form as 
\# \label{equ:one_step_err_compact}
	\gamma  \cdot  P_{\tilde{\pi}_k} (Q^* - \tilde{Q}_k) + \eta_{k+1}+ \varrho_{k+1} \leq Q^* - \tilde{Q}_{k+1} \leq \gamma \cdot   P_{\pi^*} (Q^* - \tilde{Q}_k) + \eta_{k+1}+\varrho_{k+1}.
\#  
Note that since $P_{\pi}$  is a linear operator, we can derive the following  bounds for multi-step error propagation. 

\begin{lemma}
[Multi-step Error Propagation in Cooperative MARL]\label{lemma:multi_step_err}  
For any  $k, \ell  \in \{0, 1, \ldots, K-1 \}$ with $k < \ell$, we have
\$ 
	& Q^* - \tilde{Q}_{\ell} \geq \sum_{j=k}^{\ell-1} \gamma^{\ell-1-j} \cdot (P_{\tilde{\pi}_{\ell-1}}\cdots,P_{\tilde{\pi}_{j+1}})  (\eta_{j+1}+\varrho_{j+1}) + \gamma^{\ell-k} \cdot   (P_{\tilde{\pi}_{\ell-1}}\cdots,P_{\tilde{\pi}_{k}}) ( Q^* - \tilde{Q}_k),\\
	& Q^* - \tilde{Q}_{\ell} \leq \sum_{j=k}^{\ell-1} \gamma^{\ell-1-j} \cdot (P_{\pi^*} )^{\ell-1-j} (\eta_{j+1}+\varrho_{j+1}) + \gamma^{\ell -k}\cdot  (P_{\pi^*})^{\ell-k} ( Q^* - \tilde{Q}_k),
\$
where  $\varrho_{j+1}$ and $\eta_{j+1}$ are  defined in \eqref{equ:rho_eta_k}, and we use  $P_{\pi} P_{\pi'} $ and $(P_{\pi})^{k}$ to denote the composition of operators.
\end{lemma}
\begin{proof}
By the linearity  of the operator $P_{\pi}$, we can obtain the desired results by  applying the inequalities in \eqref{equ:one_step_err_compact} multiple times.
\end{proof}

The bounds for multi-step error propagation in  Lemma \ref{lemma:multi_step_err} conclude the first step of our proof.

\vskip4pt
{\noindent \bf Step (ii):} Step (i) only establishes the propagation of error $Q^* - \tilde{Q}_k$. To evaluate the output of Algorithm \ref{algo:fit_Q_Collab}, we need to further derive the propagation of error $Q^* -  Q_{\pi_k}$, where $Q_{\pi_k}$ is the $Q$-function corresponding to the output joint policy $\pi_k$ from Algorithm \ref{algo:fit_Q_Collab}.
The error $Q^* -  Q_{\pi_k}$ quantifies the sub-optimality of the output policy $\pi_k$ at iteration $k$. 

By definition of $Q^*$, we have $Q^* \geq Q_{\pi_k}$ and  
 $ Q^* = \cT_{\pi^*} Q^*$. Also note $
	Q_{\pi_k} =\cT_{\pi_k} Q_{\pi_k}
$ and $\cT\tilde Q^i_k=\cT_{\tilde \pi^i_k}\tilde Q^i_k$, where we denote the greedy policy with respect to $\tilde Q^i_k$ by $\tilde \pi^i_k$, i.e., $\tilde \pi^i_k=\pi_{\tilde Q^i_k}$, for notational convenience. 
Hence, it follows  that 
\# \label{equ:greedy1}
	Q^* - Q_{\pi_k} =  &~  \cT_{\pi^*} Q^*  - \cT_{\pi_k} Q_{\pi_k} = \bigg(\cT_{\pi^*} Q^*- \frac{1}{N}\sum_{i\in\cN}\cT_{\pi^*} \tilde{Q}^i_k\bigg) +\frac{1}{N}\sum_{i\in\cN}
	\bigg(\cT_{\pi^*} -\cT_{\tilde \pi^i_k}\bigg) \tilde{Q}^i_k\notag\\
	&~\quad +\bigg(\frac{1}{N}\sum_{i\in\cN}
	\cT_{\tilde \pi^i_k} \tilde{Q}^i_k-\cT_{\pi_k} \tilde Q_k\bigg)+\big(\cT_{\pi_k} \tilde Q_k-\cT_{\pi_k} Q_{\pi_k}\big)
%
%
\#
Now we show that the four terms on the right-hand side of \eqref{equ:greedy1} can be bounded, respectively.  
First, by definition of $\tilde \pi^i_k$, we have
\#\label{equ:two_step1}
\cT_{\pi^*} \tilde{Q}^i_k  - \cT_{\tilde \pi^i_k} \tilde{Q}^i_k=\cT_{\pi^*} \tilde{Q}^i_k  - \cT\tilde{Q}^i_k\leq 0,\quad \text{for~all~} i\in\cN.
\#
Moreover, since $\pi_k=\cG(\Qb_k)$, it holds that for any $Q$,  $\cT_{\pi_k} Q={N}^{-1}\sum_{i\in\cN}\cT_{\tilde \pi^i_k}Q$ where $\tilde \pi^i_k$ is the greedy policy with respect to $\tilde Q^i_k$.
Then, by definition of the operator $P_{\pi}$, we have 
\#\label{equ:two_step2}
\cT_{\pi^*} Q^* - \cT_{\pi^*} \tilde{Q}^i_k = \gamma \cdot P_{\pi^*}   ( Q^* - \tilde{Q}^i_k ), \quad \cT_{\tilde \pi^i_k} \tilde{Q}^i_k -\cT_{\tilde \pi^i_k}  \tilde Q_k=\gamma \cdot P_{\tilde \pi^i_k}   ( \tilde{Q}^i_k - \tilde Q_k),
\# 
By substituting \eqref{equ:two_step1} and \eqref{equ:two_step2} into \eqref{equ:greedy1}, we obtain   
\$
	 &Q^* - Q_{\pi_k}  \leq \gamma \cdot P_{\pi^*} \bigg(Q^* - \frac{1}{N}\sum_{i\in\cN}\tilde{Q}^i_k \bigg) + \gamma \cdot \frac{1}{N}\sum_{i\in\cN} P_{\tilde \pi^i_k} ( \tilde{Q}^i_k - \tilde Q_k) +\big(\cT_{\pi_k} \tilde Q_k-\cT_{\pi_k} Q_{\pi_k}\big)\\
	& \quad = \gamma \cdot ( P_{\pi^*} - P_{\pi_k})  ( Q^* - \tilde{Q}_k )  + \gamma \cdot \frac{1}{N}\sum_{i\in\cN} \big(P_{\tilde \pi^i_k}-P_{\pi^*}\big) ( \tilde{Q}^i_k - \tilde Q_k) +\gamma \cdot P_{\pi_k}  ( Q^* - Q_{\pi_k}).
\$
This  further implies  
\$
	( I - \gamma \cdot P_{\pi_k} ) ( Q^* - Q_{\pi_k}) \leq &~ \gamma \cdot ( P_{\pi^*} - P_{\pi_k})  ( Q^* - \tilde{Q}_k )  + \gamma \cdot \frac{1}{N}\sum_{i\in\cN} \big(P_{\tilde \pi^i_k}-P_{\pi^*}\big) ( \tilde{Q}^i_k - \tilde Q_k),
\$
where  $I$ is the identity operator.
Note that for any policy $\pi$, the operator $\cT_{\pi}$ is  $\gamma$-contractive. Thus the operator $I - \gamma \cdot P_{\pi}$ is invertible and it follows  that
\#\label{equ:greedy_err}
0 \leq Q^* - Q_{\pi_k} \leq &~ \gamma \cdot  ( I - \gamma\cdot  P_{\pi_k} )^{-1} \bigl  [ P_{\pi^*}  (Q^* - \tilde{Q}_k ) - P_{\pi_k}   (Q^* - \tilde{Q}_k )\bigr ]\notag\\
&~\quad +\gamma \cdot  ( I - \gamma\cdot  P_{\pi_k} )^{-1}\biggl[\frac{1}{N}\sum_{i\in\cN} \big(P_{\tilde \pi^i_k}-P_{\pi^*}\big) ( \tilde{Q}^i_k - \tilde Q_k)\biggr ].
\#
With the expression $(Q^* - \tilde{Q}_k )$ on the right-hand side, we can further bound \eqref{equ:greedy_err} by applying Lemma \ref{lemma:multi_step_err}. To this end, we first  note that for any $f_1,f_2\in\cF(\cS\times\cA,Q_{\max})$ such  that $f_1\geq f_2$, it holds that $P_{\pi}f_1\geq P_{\pi}f_2$ by definition of $P_{\pi}$. Thus,  for any $k < \ell$, we obtain the following upper and lower bounds from Lemma \ref{lemma:multi_step_err}
\#
P_{\pi_{\ell}}  (Q^* - \tilde{Q}_\ell ) & \geq \sum_{j=k}^{\ell-1} \gamma^{\ell-1-j} \cdot (P_{ \pi_{\ell}}P_{\tilde\pi_{\ell-1} }\cdots P_{\tilde\pi_{j+1}}) (\eta_{j+1}+\varrho_{j+1}) \notag\\ 
&\qquad\qquad\qquad\qquad\qquad+ \gamma^{\ell-k} \cdot  ( P_{\pi_{\ell}} P_{\tilde\pi_{\ell-1}}\cdots P_{\tilde\pi_{k}}) ( Q^* - \tilde{Q}_k). \label{equ:upper_seq_1}\\
P_{\pi^*}(Q^* - \tilde{Q}_\ell) & \leq  \sum_{j=k}^{\ell-1} \gamma^{\ell-1-j} \cdot (P_{\pi^*} )^{\ell-j} (\eta_{j+1}+\varrho_{j+1}) + \gamma^{\ell -k}\cdot  (P_{\pi^*})^{\ell-k+1} ( Q^* - \tilde{Q}_k).  \label{equ:upper_seq_2}
\#

Moreover, we denote the second term on the right-hand side of \eqref{equ:greedy_err} by $\xi_k$, i.e., 
\$
\xi_k=\gamma \cdot  ( I - \gamma\cdot  P_{\pi_k} )^{-1}\biggl[\frac{1}{N}\sum_{i\in\cN} \big(P_{\tilde \pi^i_k}-P_{\pi^*}\big) ( \tilde{Q}^i_k - \tilde Q_k)\biggr ].
\$
Note that $\xi_k$ 
depends on the accuracy of the output of the decentralized optimization algorithm at iteration $k$, i.e., the error $\tilde Q^i_k-\tilde Q_k$, which vanishes with the number of updates of the decentralized optimization algorithm,   when the least-squares  problem \eqref{equ:fitted_least_squares} is convex. 
By this definition, together with \eqref{equ:upper_seq_1}-\eqref{equ:upper_seq_2}, we obtain the bound for the error $Q^* - Q_{\pi_{K} }$ at the final iteration $K$ as    
\#\label{equ:greedy_err_multiple}
	Q^* - Q_{\pi_{K} } \leq & ~     ( I - \gamma\cdot  P_{\pi_{K} } )^{-1} \bigg \{ \sum_{j=0}^{K-1} \gamma^{K  - j} \cdot \bigl [  (P_{\pi^*} )^{K-j} -  (P_{\pi_{K}}P_{\tilde\pi_{K-1} }\cdots P_{\tilde\pi_{j+1}}) \bigr ] (\eta_{j+1}+\varrho_{j+1}) \notag\\
	&~\quad  +  \gamma^{K + 1}\cdot \bigl [ (P_{\pi^*} )^{K+1} -  (P_{\pi_{K}}P_{\tilde\pi_{K-1} }\cdots P_{\tilde\pi_{0}}) \bigr ] ( Q^* - \tilde{Q}_0) \bigg\} +\xi_{K}.
\#
To simplify the notation, we introduce the coefficients 
\#\label{equ:define_alpha_param}
\alpha_{j} & = \frac{(1-\gamma) \gamma^{K-j-1}}{1-\gamma^{K+1}}, ~~\text{for}~~ 0 \leq j \leq  K-1, ~~\text{and}~~  
	\alpha_{K } = \frac{(1-\gamma) \gamma^{K}}{1-\gamma^{K+1}}.  
\#
Also, we introduce  $K+1$ linear operators $\{\cL_k \}_{ k=0}^K $ that  are defined as 
\$
\cL_j &= \frac{(1 - \gamma)}{2} \cdot ( I - \gamma P_{\pi_K})^{-1}  \bigl [  (P_{\pi^*} )^{K -j} + (P_{\pi_{K}}P_{\tilde\pi_{K-1} }\cdots P_{\tilde\pi_{j+1}}) \bigr ], ~~\text{for}~~ 0 \leq j \leq K-1,\\
\cL_K & = \frac{(1 - \gamma)}{2} \cdot  ( I - \gamma P_{\pi_K})^{-1} \bigl [ (P_{\pi^*} )^{K+1} + (P_{\pi_{K}}P_{\tilde\pi_{K-1} }\cdots P_{\tilde\pi_{0}}) \bigr ].
\$
Then, by taking the absolute values of both sides of  \eqref{equ:greedy_err_multiple}, we obtain   that 
for any $(s, a) \in \cS \times \cA$
\#\label{equ:absolute_value_bound}
\bigl | Q^* (s,a) - Q_{\pi_K} (s,a) \bigr | \leq &~ \frac{2 \gamma ( 1 - \gamma^{K+1} ) }{(1- \gamma)^2}  \cdot \biggl [ \sum_{j=0}^{K-1} \alpha_j \cdot \bigl  ( \cL_j |\eta_{j+1}+ \varrho_{j+1} | \bigr ) (s,a) \notag\\
&~\quad + \alpha_K  \cdot \bigl ( \cL_K | Q^* - \tilde Q_0 | \bigr ) (s,a) \biggr ]+|\xi_K(s,a)|,
\#
where functions   $\cL_j | \eta_{j+1}+\varrho_{j+1} | $ and $\cL_K | Q^* - \tilde Q_0 |$ are both defined over $\cS \times \cA$.  The upper bound in  \eqref{equ:absolute_value_bound}   concludes the second step of the proof.

\vskip4pt
{\noindent \bf Step (iii):} Now we establish the final step to complete  the proof. In particular,  we upper bound the weighted norm  $\| Q^* - Q_{\pi_{K}} \|_{\mu}$ for some probability distribution $\mu\in \cP(\cS \times \cA)$,  based on the point-wise bound of  $| Q^*- Q_{\pi_K}|$ from \eqref{equ:absolute_value_bound}. 
For notational simplicity, we define $\mu (f) $ to be the expectation of $f$ under $\mu$, that is, 
$
	\mu(f)  = \int_{\cS \times \cA} f(s, a) \ud\mu(s,a).
$
By taking square of both sides of  \eqref{equ:absolute_value_bound}, we obtain
\$
\bigl | Q^* (s,a) - Q_{\pi_K} (s,a) \bigr |^2\leq &~  2\cdot\bigg[\frac{2 \gamma ( 1 - \gamma^{K+1} ) }{(1- \gamma)^2}\bigg]^2  \cdot \biggl [ \sum_{j=0}^{K-1} \alpha_j \cdot \bigl  ( \cL_j |\eta_{j+1}+ \varrho_{j+1} | \bigr ) (s,a) \notag\\
&~\quad + \alpha_K  \cdot \bigl ( \cL_K | Q^* - \tilde Q_0 | \bigr ) (s,a) \biggr ]^2+2\cdot|\xi_K(s,a)|^2.
\$
Then, by  applying Jensen's inequality twice, we arrive at 
\#\label{equ:lp_norm_p}
 \| Q^* - Q_{\pi_{K}} \|_{\mu}^2  = \mu\big(| Q^* -  Q_{\pi_{K}} |^2\big)  \leq &~ 2\cdot\bigg[\frac{2 \gamma ( 1 - \gamma^{K+1} ) }{(1- \gamma)^2} \bigg]^2   \cdot  \mu  \biggl ( \sum_{j=0}^{K-1} \alpha_j \cdot \bigl  ( \cL_j | \eta_{j+1}+\varrho_{j+1} |^2  \bigr ) \notag \\
&~\quad   + \alpha_K  \cdot \bigl ( \cL_K | Q^* - \tilde Q_0 |^2 \bigr )   \biggr )+2\cdot\mu\big(|\xi_K|^2\big),
\#
where we also use the fact that $\sum_{j=0}^K\alpha_j=1$ and for all $j=0,\cdots,K$, the linear operators $\cL_j$ are positive and satisfy $\cL_j\bm1=\bm1$. 
Since   both $Q^*$ and $\tilde Q_0$ are bounded by  $Q_{\max} = R_{\max} / ( 1-\gamma )$ in absolute value, we have  
\#\label{equ:second_term}
  \mu \Bigr(\bigl ( \cL_{K } | Q^* - \tilde Q_0 |   \bigr )^2\Bigr) \leq     (2 Q_{\max})^2.
\#
Also, by definition of the concentrability coefficients $\kappa^{\text{MDP}}$ in \S\ref{sec:append_term_def}, we have 
\#\label{equ:first_term}
\mu(\cL_j)\leq (1-\gamma)\sum_{m\geq 0}\gamma^m \kappa^{\text{MDP}}(m+K-j)\nu,
\#
where recall that $\nu$ is the distribution over $\cS\times\cA$ from which  the data $\{(s_t,a_t)\}_{t=1,\cdots,T}$  in trajectory $\cD$ are sampled.  
Moreover, we can also bound $\mu(|\xi_K|^2)$ by Jensen's inequality   as
\#\label{equ:last_term}
\mu\big(|\xi_K|^2\big)&\leq \bigg(\frac{2\gamma}{1-\gamma}\bigg)^2\cdot \mu\biggl(\frac{1-\gamma}{2N}\cdot( I - \gamma\cdot  P_{\pi_k} )^{-1}\sum_{i\in\cN} \big(P_{\tilde \pi^i_k}+P_{\pi^*}\big) \big| \tilde{Q}^i_k - \tilde Q_k\big|\biggr )^2\notag\\
&\leq \bigg(\frac{2\gamma}{1-\gamma}\bigg)^2\cdot \mu\biggl(\frac{1-\gamma}{2N}\cdot( I - \gamma\cdot  P_{\pi_k} )^{-1}\sum_{i\in\cN} \big(P_{\tilde \pi^i_k}+P_{\pi^*}\big) \big| \tilde{Q}^i_k - \tilde Q_k\big|^2\biggr ).
\#
By Assumption \ref{assum:one_step_comp_error}, we can further bound the right-hand side of \eqref{equ:last_term} as
\#\label{equ:last_term_2}
\mu\big(|\xi_K|^2\big)\leq \bigg(\frac{2\gamma}{1-\gamma}\bigg)^2\cdot \frac{1}{N}\sum_{i\in\cN}\bigl\| \tilde{Q}^i_k - \tilde Q_k\bigl\|^2_{\mu}\leq \bigg(\frac{2\gamma}{1-\gamma}\bigg)^2\cdot \frac{1}{N}\sum_{i\in\cN}(\epsilon^i_K)^2.
\#

Therefore, by plugging \eqref{equ:second_term}, \eqref{equ:first_term}, and \eqref{equ:last_term_2} into  \eqref{equ:lp_norm_p}, we obtain
\#\label{equ:p_norm_plug_alpha}
\| Q^* - Q_{\pi_{K}} \|_{\mu}^2   \leq &~ \bigg[\frac{4 \gamma ( 1 - \gamma^{K+1} ) }{\sqrt{2}(1- \gamma)^2} \bigg]^2   \cdot  \biggl [ \sum_{j=0}^{K-1} \frac{(1-\gamma)^2 }{1-\gamma^{K+1}}\cdot\sum_{m\geq 0}\gamma^{m+K-j-1} \kappa^{\text{MDP}}(m+K-j) \notag \\
& \quad\cdot\big\| \eta_{j+1}+\varrho_{j+1} \big\|^2_{\nu}  + \frac{(1-\gamma) \gamma^{K}}{1-\gamma^{K+1}}  \cdot  (2 Q_{\max})^2   \biggr ]+\bigg(\frac{2\sqrt{2}\gamma}{1-\gamma}\bigg)^2\cdot \frac{1}{N}\sum_{i\in\cN}(\epsilon^i_K)^2. 
\# 
Furthermore, from Assumption \ref{assume:concentrability} and the definition of $\phi_{\mu,\nu}^{\text{MDP}}$, and letting $\overline{\epsilon}_K=[N^{-1}\cdot\sum_{i\in\cN}(\epsilon^i_K)^2]^{1/2}$, it follows from \eqref{equ:p_norm_plug_alpha}  that 
\$
\| Q^* - Q_{\pi_{K}} \|_{\mu}^2   &\leq  \bigg[\frac{4 \gamma ( 1 - \gamma^{K+1} ) }{\sqrt{2}(1- \gamma)^2} \bigg]^2   \cdot  \biggl [ \frac{(1-\gamma)^2 }{1-\gamma^{K+1}}\cdot\sum_{m\geq 0}m\cdot\gamma^{m} \cdot\kappa^{\text{MDP}}(m)\notag \\
&\qquad  \cdot \max_{j=0,\cdots,K-1}\big\| \eta_{j+1}+\varrho_{j+1} \big\|^2_{\nu}  + \frac{(1-\gamma) \gamma^{K}(2 Q_{\max})^2}{1-\gamma^{K+1}}    \biggr ]+\bigg(\frac{2\sqrt{2}\gamma}{1-\gamma}\cdot \overline{\epsilon}_K\bigg)^2,\notag\\
& \leq  \bigg[\frac{4 \gamma ( 1 - \gamma^{K+1} ) }{\sqrt{2}(1- \gamma)^2} \bigg]^2   \cdot  \biggl [ \frac{\phi^{\text{MDP}}_{\mu,\nu} }{1-\gamma^{K+1}} \cdot \max_{j=0,\cdots,K-1}\big\| \eta_{j+1}+\varrho_{j+1} \big\|_{\nu}^2 \notag\\
&\quad\quad + \frac{(1-\gamma) \gamma^{K}(2 Q_{\max})^2}{1-\gamma^{K+1}} \biggr ]+\bigg(\frac{2\sqrt{2}\gamma}{1-\gamma}\cdot \overline{\epsilon}_K\bigg)^2.
\$
This further yields that
\#\label{equ:p_norm_concentra_coeff}
&\| Q^* - Q_{\pi_{K}} \|_{\mu} \notag\\  &\quad\leq  \frac{4 \gamma ( 1 - \gamma^{K+1} ) }{\sqrt{2}(1- \gamma)^2}   \cdot  \biggl [ \frac{\big(\phi^{\text{MDP}}_{\mu,\nu}\big)^{1/2} }{1-\gamma^{K+1}} \cdot (\|\eta\|_\nu+\|\varrho\|_\nu)  + \frac{\gamma^{K/2}}{1-\gamma^{K+1}}  \cdot  (2 Q_{\max})   \biggr ]+\frac{2\sqrt{2}\gamma}{1-\gamma}\cdot \overline{\epsilon}_K,\notag\\
&\quad\leq  \frac{4 \gamma \cdot\big(\phi^{\text{MDP}}_{\mu,\nu}\big)^{1/2} }{\sqrt{2}(1- \gamma)^2}   \cdot (\|\eta\|_\nu+\|\varrho\|_\nu)+\frac{4\sqrt{2}\cdot Q_{\max} }{(1- \gamma)^2}   \cdot \gamma^{K/2} + \frac{2\sqrt{2}\gamma}{1-\gamma}\cdot \overline{\epsilon}_K,
\#
where we denote by $\|\eta\|_\nu =\max_{j=0,\cdots,K-1}\| \eta_{j+1}\|_{\nu}$ and $\|\varrho\|_\nu =\max_{j=0,\cdots,K-1}\| \varrho_{j+1}\|_{\nu}$.
Recall that $\eta_{j+1}$ is defined as $\eta_{j+1}=\cT\tilde {Q}_{j} -\tcT\tilde{\Qb}_{j}$, which can be further bounded by the one-step decentralized computation error from  Assumption \ref{assum:one_step_comp_error}. Specifically, we have the following lemma regarding the difference between $\cT\tilde {Q}_{j}$ and $\tcT\tilde{\Qb}_{j}$.

\begin{lemma}\label{lemma:max_minus_max}
	Under Assumption \ref{assum:one_step_comp_error}, for any $j=0,\cdots,K-1$, it holds that
	$
	\|\eta_{j+1}\|_\nu\leq \sqrt{2}\gamma\cdot\overline{\epsilon}_{j}
	$, where $\overline{\epsilon}_{j}=[N^{-1}\cdot\sum_{i\in\cN}(\epsilon^i_{j})^2]^{1/2}$ and $\epsilon^i_{j}$ is defined as in Assumption \ref{assum:one_step_comp_error}.
\end{lemma}
\begin{proof}
By definition, we have
\$
|\eta_{j+1}(s,a)|&=\Big|\big(\cT\tilde {Q}_{j}\big)(s,a) -\big(\tcT\tilde{\Qb}_{j}\big)(s,a)\Big|\notag\\
&\leq \gamma\cdot\frac{1}{N}\sum_{i\in\cN}\EE\Bigl[\big|\max_{a' \in \cA} \tilde Q_j(s', a')-\max_{a' \in \cA} \tilde Q^i_j(s', a')\big| \Biggiven s' \sim P(\cdot  \given s,a)\Bigr].
\$
Now we claim that for any $s'\in\cS$
\#\label{equ:pf_max_minus_max_1}
\big|\max_{a' \in \cA} \tilde Q_j(s', a')-\max_{a' \in \cA} \tilde Q^i_j(s', a')\big|\leq (1+C_0)\cdot \epsilon^i_j,
\# for any constant $C_0>0$. Suppose \eqref{equ:pf_max_minus_max_1} does not hold.  Then, either $\max_{a' \in \cA} \tilde Q^i_j(s', a')\geq \max_{a' \in \cA} \tilde Q_j(s', a')+(1+C_0)\cdot \epsilon^i_j$  or $\max_{a' \in \cA} \tilde Q^i_j(s', a')\leq \max_{a' \in \cA} \tilde Q_j(s', a')-(1+C_0)\cdot \epsilon^i_j$. In the first case, let $a'_*\in\argmax _{a' \in \cA} \tilde Q^i_j(s', a')$, then by Assumption \ref{assum:one_step_comp_error}, the values of $\tilde Q_j(s',\cdot)$ and $\tilde Q^i_j(s',\cdot)$ are close at $a'_*$ up to a small error $\epsilon^i_j$, i.e.,  $\tilde Q_j(s',a'_*)\geq \tilde Q^i_j(s',a'_*)-\epsilon^i_j$. This implies that $\tilde Q_j(s',a'_*)\geq \max_{a' \in \cA} \tilde Q_j(s', a')+C_0\cdot \epsilon^i_j$, which cannot hold since $\max_{a' \in \cA} \tilde Q_j(s', a')\geq \tilde Q_j(s',a')$ for any $a'$ including $a'_*$. Similarly, one can show that the second case cannot occur. Thus, the claim \eqref{equ:pf_max_minus_max_1} is proved. Letting $C_0=\sqrt{2}-1$, we obtain that 
\$
|\eta_{j+1}(s,a)|^2\leq \gamma^2  \bigg(\frac{1}{N}\sum_{i\in\cN}\sqrt{2}\epsilon^i_j\bigg)^2\leq \big(\sqrt{2}\gamma\big)^2  \frac{1}{N}\sum_{i\in\cN} \big(\epsilon^i_j\big)^2, 
\$
where the second inequality follows from Jensen's inequality. Taking expectation over $\nu$, we  obtain the desired bound. 
\end{proof}

From Lemma \ref{lemma:max_minus_max}, we can further simplify \eqref{equ:p_norm_concentra_coeff} to obtain the desired bound in Theorem \ref{thm:err_prop_collab}, which concludes the proof. 
\end{proof}

\subsection{Proof of Theorem \ref{thm:err_one_step_approx}}\label{proof:thm:err_one_step_approx}
\begin{proof}
	The proof integrates the proof ideas in \cite{munos2008finite} and \cite{antos2008learning}.
	First, for any fixed  $\Qb=[Q^i]_{i\in\cN}$ and $f$, we define   
	\$
	d(\Qb)(s,a,s')&=\overline{r}(s,a) + \gamma \cdot \frac{1}{N}\sum_{i\in\cN}\max _{a'\in \cA}  Q^i (s', a'),\\
	\ell_{f,\Qb}(s,a,s')&=[d(\Qb)(s,a,s')-f(s,a)]^2,
	\$
	where $\overline{r}(s,a)=N^{-1}\cdot\sum_{i\in\cN}r^i(s,a)$ with $r^i(s,a)\sim R^i(s,a)$. We also define $\cL_{\cH}=\{\ell_{f,\Qb}:f\in\cH,\Qb\in\cH^N\}$. For convenience, we denote   $d_t(\Qb)=d(\Qb)(s_t,a_t,s_{t+1})$, for any data $(s_t,a_t,s_{t+1})$ drawn from $\cD$.
	Then,   
	we define $\hat{L}_T(f;\Qb)$ and ${L}(f;\Qb)$ as
	\#
	\hat{L}_T(f;\Qb)&=\frac{1}{T}\sum_{t=1}^T \bigl [d_t(\Qb) - f(s_t, a_t) \bigr ]^2=\frac{1}{T}\sum_{t=1}^T\ell_{f,\Qb}(s_t,a_t,s_{t+1}),\notag\\
	{L}(f;\Qb)&=\|f- \tcT\Qb \|^2_{\nu}+\EE_{\nu}\{\Var[d_1(\Qb)\given s_1,a_1]\}.\label{equ:def_L}
	\#
	Obviously  ${L}(f;\Qb)=\EE[\hat{L}_T(f;\Qb)]$, and $\argmin_{f\in\cH}\hat{L}_T(f;\Qb)$ is exactly the minimizer  of the  objective defined in \eqref{equ:fitting_obj}. Also, note that  the second term on the right-hand side  of \eqref{equ:def_L} does not depend on $f$, thus $\argmin_{f\in\cH}{L}(f;\Qb)=\argmin_{f\in\cH}\|f- \tcT\Qb \|^2_{\nu}$. 
	Letting $f'\in\argmin_{f\in\cH}\hat{L}_T(f;\Qb)$, we have
	\#
	&\|f'-\tcT \Qb\|^2_{\nu}-\inf_{f\in\cH}\|f-\tcT \Qb\|^2_{\nu}=L(f';\Qb)-\hat{L}_{T}(f';\Qb)+\hat{L}_{T}(f';\Qb)-\inf_{f\in\cH}L(f;\Qb)\notag\\
	&\quad\leq |\hat{L}_{T}(f';\Qb)-L(f';\Qb)|+\inf_{f\in\cH}\hat{L}_{T}(f;\Qb)-\inf_{f\in\cH}L(f;\Qb)\leq 2\sup_{f\in\cH}|\hat{L}_{T}(f;\Qb)-L(f;\Qb)|,\notag\\
	&\quad\leq 2\sup_{f\in\cH,\Qb\in\cH^N}|\hat{L}_{T}(f;\Qb)-L(f;\Qb)|=2\sup_{\ell_{f,\Qb}\in\cL_{\cH}}\bigg|\frac{1}{T}\sum_{t=1}^T\ell_{f,\Qb}(Z_t)-\EE[\ell_{f,\Qb}(Z_1)]\bigg|,\label{equ:unform_dev_bnd}
	\#
	where we use the definition of $f'$ and $\ell_{f,\Qb}$, and let $Z_t=(s_t,\{a^i_t\}_{i\in\cN},s_{t+1})$ for notational convenience.
	
	In addition, we define two constants $C_1$ and $C_2$ as 
	\#\label{equ:def_C_1}
 C_1=16\cdot e^{N+1}(V_{\cH^+}+1)^{N+1}A^{NV_{\cH^+}}Q_{\max}^{(N+1)V_{\cH^+}} \big[{64\cdot e\tilde R_{\max}(1+\gamma)}\big]^{(N+1)V_{\cH^+}},
 \#
 and  $C_2=1/(2048\cdot \tilde R^4_{\max})$, and also define  $\Lambda_T(\delta)$  and $\epsilon$ as  
	\#\label{def:eps_Lambda}
	\Lambda_T(\delta)=\frac{V}{2}\log(T)+\log\Big(\frac{e}{\delta}\Big)+\log^+(C_1C_2^{V/2}\vee \overline{\beta}),\quad \epsilon &=\sqrt{\frac{\Lambda_T(\delta)[\Lambda_T(\delta)/b\vee 1]^{1/\zeta}}{C_2 T}}, 
	\#
	where $V=(N+1)V_{\cH^+}$ and $\log^+(x)=\max\{\log(x),0\}$. 
	Let 
	\#\label{equ:def_P_0}
	P_0=\PP\bigg(\sup_{\ell_{f,\Qb}\in\cL_{\cH}}\bigg|\frac{1}{T}\sum_{t=1}^T\ell_{f,\Qb}(Z_t)-\EE[\ell_{f,\Qb}(Z_1)]\bigg|>\frac{\epsilon}{2}\bigg).
	\#
	Then, from \eqref{equ:unform_dev_bnd}, it suffices to show $P_0<\delta$ in order  to conclude the proof. To this end, we use the  technique in \cite{antos2008learning} that splits  the $T$ samples in $\cD$ into $2m_T$ blocks that come in pairs, with each block having $k_T$ samples, i.e., $T=2m_Tk_T$. Then, we can  introduce the  ``ghost'' samples that have $m_T$ blocks, $H_1,\cdots,H_2,\cdots,H_{m_T}$, where  each block has the same marginal distribution as the every second blocks in $\cD$, but these new   $m_T$ blocks are independent of one another. We let $H=\bigcup_{i=1}^{m_T}H_i$. Recall that   $\tilde R_{\max}=(1+\gamma)Q_{\max}+R_{\max}$, then for any $f\in\cH$ and $\Qb\in\cH^N$, $\ell_{f,\Qb}$ has absolute  value bounded by $\tilde R_{\max}^2$. Thus, we can apply an extended version of Pollard's tail inequality to $\beta$-mixing sequences (Lemma $5$ in \cite{antos2008learning}), to obtain that 
	\#\label{equ:uniform_dev_prob}
	&\PP\bigg(\sup_{\ell_{f,\Qb}\in\cL_{\cH}}\bigg|\frac{1}{T}\sum_{t=1}^T\ell_{f,\Qb}(Z_t)-\EE[\ell_{f,\Qb}(Z_1)]\bigg|>\frac{\epsilon}{2}\bigg)\notag\\&\quad \leq 16\cdot\EE\big[\cN_1(\epsilon/16,\cL_{\cH},(Z'_t;t\in H))\big]e^{-\frac{m_T}{2}\cdot\frac{\epsilon^2}{(16 \tilde R_{\max}^2)^2}}+2m_T\beta_{k_T},
	\#
	where $\beta_m$ denotes the $m$-th $\beta$-mixing coefficient of  the sequence $Z_1,\cdots,Z_T$ in $\cD$ (see  Definition  \ref{def:mixing}), and $\cN_1(\epsilon/16,\cL_{\cH},(Z'_t;t\in H))$ is the \emph{empirical covering number}  (see Definition \ref{def:covering_num}) of the function class $\cL_{\cH}$ evaluated on the ghost samples $(Z'_t;t\in H)$.
	
	To bound the {empirical covering number} $\cN_1(\epsilon/16,\cL_{\cH},(Z'_t;t\in H))$, we establish the  following technical lemma. 
	
	\begin{lemma}\label{lemma:cover_num_bnd}
		Let $Z^{1:T}=(Z_1,\cdots,Z_T)$, with $Z_t=(s_t,\{a^i_t\}_{i\in\cN},s_{t+1})$. Recall that  $\tilde R_{\max}=(1+\gamma)Q_{\max}+R_{\max}$, and  $A=|\cA|$ is the cardinality of the joint action set. Then, under Assumption  \ref{assum:capacity_func}, it holds that
		\$
		\cN_1(\epsilon,\cL_{\cH},Z^{1:T})\leq e^{N+1}(V_{\cH^+}+1)^{N+1}A^{NV_{\cH^+}}Q_{\max}^{(N+1)V_{\cH^+}} \bigg(\frac{4e\tilde R_{\max}(1+\gamma)}{\epsilon}\bigg)^{(N+1)V_{\cH^+}}.
		\$
	\end{lemma}
	\begin{proof}
	For any  $l_{f,\Qb}$ and $l_{\tilde f,\tilde \Qb}$ in $\cL_{\cH}$, 
		the empirical $\ell^1$-distance between them can be bounded as
	\#
	&\frac{1}{T}\sum_{t=1}^T\big|l_{f,\Qb}(Z_t)-l_{\tilde f,\tilde \Qb}(Z_t)\big|\leq \frac{1}{T}\sum_{t=1}^T \bigg|\big|d_t(\Qb)-f(s_t,a_t)\big|^2-\big|d_t(\tilde \Qb)-\tilde f(s_t,a_t)\big|^2\bigg|\notag\\
	&\quad\leq \frac{2\tilde R_{\max}}{T}\sum_{t=1}^T \bigg[\big|f(s_t,a_t)-\tilde f(s_t,a_t)\big|+\frac{\gamma}{N}\sum_{i\in\cN}\big|\max_{a'\in\cA}Q^i(s_{t+1},a')-\max_{a'\in\cA}\tilde Q^i(s_{t+1},a')\big|\bigg]\notag\\
	&\quad= 2\tilde R_{\max}\bigg[\frac{1}{T}\sum_{t=1}^T \big|f(s_t,a_t)-\tilde f(s_t,a_t)\big|+ \frac{\gamma}{T}\sum_{t=1}^T \Big|\frac{1}{N}\sum_{i\in\cN}\max_{a'\in\cA}Q^i(s_{t+1},a')\notag\\
	&\qquad\qquad\qquad\qquad\qquad-\frac{1}{N}\sum_{i\in\cN}\max_{a'\in\cA}\tilde Q^i(s_{t+1},a')\Big|\bigg].\label{equ:cov_num_1}
	\#
	Let $\cD_Z=[(s_1,\{a^i_1\}_{i\in\cN}),\cdots,(s_T,\{a^i_T\}_{i\in\cN})]$ and $y_Z=(s_2,\cdots,s_{T+1})$. Then, 
	 the first  term in the bracket of \eqref{equ:cov_num_1} is the $\cD_Z$-based $\ell^1$-distance of functions in $\cH$, while the second term is the $y_Z$-based $\ell^1$-distance of functions in the set $\cH^{\vee}_N=\{V:V(\cdot)=N^{-1}\cdot\sum_{i\in\cN}\max_{a'\in\cA}Q^i(\cdot,a')\text{~and~}\Qb\in\cH^N\}$ (times $\gamma$). This implies that 
	\#\label{equ:empirical_cov_num_1}
	\cN_1\big(2\tilde R_{\max}(1+\gamma)\epsilon,\cL_{\cH},Z^{1:T}\big)\leq \cN_1(\epsilon,\cH^{\vee}_N,y_Z)\cdot\cN_1(\epsilon,\cH,\cD_Z).
	\#
	The first empirical covering number $\cN_1(\epsilon,\cH^{\vee}_N,y_Z)$ can be further bounded by the following lemma, whose proof is deferred to \S\ref{sec:append_proofs}.
	
	\begin{lemma}\label{lemma:empirical_cov_num_H_vee}
		For any fixed $y_{Z}=(y_1,\cdots,y_T)$, let 
		$\cD_y=\{(y_t,a_j)\}_{t\in[T],j\in[A]},
		$ where recall that $A=|\cA|$ and $\cA=\{a_1,\cdots,a_A\}$. Then, under Assumption  \ref{assum:capacity_func}, it holds that 
		\$
		\cN_1(\epsilon,\cH^{\vee}_N,y_Z)\leq \big[\cN_1({\epsilon}/{A},\cH,\cD_y)\big]^N\leq \bigg[e(V_{\cH^+}+1)\bigg(\frac{2eQ_{\max}A}{\epsilon}\bigg)^{V_{\cH^+}}\bigg]^N.
		\$
	\end{lemma}
	
	\vspace{3pt}
	In addition, the second empirical covering number 
	$\cN_1(\epsilon,\cH,\cD_Z)$ in \eqref{equ:empirical_cov_num_1} can be bounded directly by  Corollary $3$ in \cite{haussler1995sphere} (see also Proposition \ref{prop:hauss_empirical_cov} in \S\ref{sec:append_proofs}).   Combined with the bound from Lemma \ref{lemma:empirical_cov_num_H_vee}, we finally obtain
	\#\label{equ:lemma_cov_bnd_final}
	&\cN_1\big(2\tilde R_{\max}(1+\gamma)\epsilon,\cL_{\cH},Z^{1:T}\big)\notag\\
	&\quad\leq \bigg[e(V_{\cH^+}+1)\bigg(\frac{2eQ_{\max}A}{\epsilon}\bigg)^{V_{\cH^+}}\bigg]^N\cdot \bigg[e(V_{\cH^+}+1)\bigg(\frac{2eQ_{\max}}{\epsilon}\bigg)^{V_{\cH^+}}\bigg]\notag\\
	&\quad =e^{N+1}(V_{\cH^+}+1)^{N+1}A^{NV_{\cH^+}}Q_{\max}^{(N+1)V_{\cH^+}} \bigg(\frac{2e}{\epsilon}\bigg)^{(N+1)V_{\cH^+}}.
	\#
	Replacing $2\tilde R_{\max}(1+\gamma)\epsilon$ by $\epsilon$ in \eqref{equ:lemma_cov_bnd_final}, we arrive at the desired bound and  complete the proof.
	\end{proof} 

By Lemma \ref{lemma:cover_num_bnd}, we can  bound $\cN_1(\epsilon/16,\cL_{\cH},(Z'_t;t\in H))$	 in \eqref{equ:uniform_dev_prob} as
\$
&\cN_1(\epsilon/16,\cL_{\cH},(Z'_t;t\in H))\\
&\quad\leq e^{N+1}(V_{\cH^+}+1)^{N+1}A^{NV_{\cH^+}}Q_{\max}^{(N+1)V_{\cH^+}} \bigg[\frac{64e\tilde R_{\max}(1+\gamma)}{\epsilon}\bigg]^{(N+1)V_{\cH^+}} = \frac{C_1}{16}\bigg(\frac{1}{\epsilon}\bigg)^V,
\$
where $V=(N+1)V_{\cH^+}$ and $C_1=C_1(V_{\cH^+},Q_{\max},\tilde R_{\max},\gamma)$ is as defined in \eqref{equ:def_C_1}. 
	
	{Now we are ready to bound the probability $P_0$ in \eqref{equ:uniform_dev_prob}, based on the following  lemma. 

\begin{lemma}[\cite{antos2008learning}, Lemma $14$]\label{lemma:high_prob_to_dev}
	Recall that the parameters $(\overline{\beta},g,\zeta)$ define  the rate of the exponential $\beta$-mixing sequence $Z_1,\cdots,Z_T$ in $\cD$. Let $\beta_m\leq \overline{\beta}\exp(-gm^{\zeta})$, $T\geq 1$, $k_T=\lceil (C_2 T\epsilon^2/b)^\frac{1}{1+\zeta}\rceil$, $m_T=T/(2K_T)$, $\delta\in(0,1]$, $V\geq 2$, and $C_1, C_2, \overline{\beta}, g, \zeta>0$.  Further, with $\epsilon$ and $\Lambda_T$   defined in \eqref{def:eps_Lambda},
	\$
	C_1\bigg(\frac{1}{\epsilon}\bigg)^Ve^{-4C_2m_T\epsilon^2}+2m_T\beta_{k_T}<\delta.
	\$
\end{lemma}
	
	In particular, let $C_2=[{8(16\tilde R^2_{\max})^2}]^{-1}=(2048 \tilde R^4_{\max})^{-1}$, we obtain that $P_0<\delta$ by applying Lemma \ref{lemma:high_prob_to_dev} onto the right-hand side of \eqref{equ:uniform_dev_prob}.
	Note that 
	\$
	\Lambda_T= &~ {V}/{2}\log(T)+\log({e}/{\delta})+\max\big\{\log(C_1)+{V}/{2}\log(C_2),\log(\overline{\beta}),0\big\}\\=&~K_1+K_2\cdot N,
	\$
	where \$
 &K_1=K_1\big(V_{\cH^+}\log(T),\log(1/\delta),\log(\tilde R_{\max}),V_{\cH^+}\log(\overline{\beta})\big),\\ &K_2=K_2\big(V_{\cH^+}\log(T),V_{\cH^+}\log(\overline{\beta}),V_{\cH^+}\log[\tilde R_{\max}(1+\gamma)],V_{\cH^+}\log(Q_{\max}), V_{\cH^+}\log(A)\big), 
 \$ 
are some constants that depend on the parameters in the brackets. 
This yields the desired bound  and 
	  completes the proof.
	}
\end{proof}


\subsection{Proof of Theorem \ref{thm:err_prop_compet}}    \label{proof:thm:err_prop_compet}
  \begin{proof}
 The proof is similar to the proof of  Theorem \ref{thm:err_prop_collab}. For brevity, we will only emphasize the  the difference between them. 
We first define  two quantities   $\varrho^1_{k}$ and $\eta^1_k$ as follows
 \#\label{equ:rho_eta_k_2}
 \varrho^1_{k} = \tcT\tilde{\Qb}^1_{k-1} - \tilde Q^1_k, \quad \eta^1_k = \cT\tilde {Q}^1_{k-1} -\tcT\tilde{\Qb}^1_{k-1},
 \#  
 where recall that at iteration $k$, $\tilde Q^1_k$ is the exact minimizer of the least-squares  problem \eqref{equ:fitted_least_squares_2} among  $f^1\in\cH$,  and  $\tilde{\Qb}^1_{k}=[\tilde Q^{1,i}_k]_{i\in\cN}$ are the output of $Q$-function estimators at all  agents in Team $1$ from Algorithm \ref{algo:fit_Q_Compet}. Recall from \eqref{equ:operator_average_T_opt_comp} that  $\tcT\tilde{\Qb}^1_{k}$ here  has the form 
 \$
(\tcT\tilde{\Qb}^1_{k})(s,a,b) =\overline{r}(s, a,b) + \gamma\cdot  \EE_{s' \sim P(\cdot  \given s,a,b)}\biggl[\frac{1}{N}\sum_{i=1}^N\max_{\pi' \in \cP(\cA)} \min_{\sigma' \in \cP(\cB)}\EE_{\pi',\sigma'}\big[\tilde Q^{1,i}_k(s', a',b')\big]\biggr].
\$
The term $\rho^1_k$ captures the approximation error of the first fitting  problem  in \eqref{equ:fitted_least_squares_2}, which can be characterized  using tools from nonparametric regression.  The term  $\eta^1_k$, on the other  hand,  captures the computational error of the decentralized optimization algorithm after a finite number of updates. 
Also, we denote by  $\pi_k$ the average equilibrium  policy obtained from the estimator vector $\tilde{\Qb}^1_{k}$, i.e., $\pi_k=\cE^1(\tilde{\Qb}_{k})$. 
For notational convenience, 
we also introduce the following notations for several different   minimizer policies, {$\tilde\sigma_k^i$, 
$\tilde{\sigma}_k^{i,*}, \tilde{\sigma}_{k}, \tilde{\sigma}^*_{k}$, $\sigma^*_k$, $\hat \sigma^i_k$, and $\overline \sigma_k$  }, which satisfy 
\#\label{equ:nu_def}
&\cT\tilde Q^{1,i}_k=\cT_{\tilde \pi^i_{k}}\tilde Q^{1,i}_k=\cT_{\tilde \pi^i_{k},\tilde \sigma^i_{k}}\tilde Q^{1,i}_k,\quad \cT_{\pi^*}\tilde Q^{1,i}_k=\cT_{\pi^*,\sigma_k^{i,*}}\tilde Q^{1,i}_k,\quad \cT\tilde Q^1_k=\cT_{\tilde \pi_k}\tilde Q^1_k=\cT_{\tilde \pi_k,\tilde \sigma_k}\tilde Q^1_k
\\
&\cT_{\tilde \pi_k}Q^*=\cT_{\tilde \pi_k,\tilde \sigma^*_k}Q^*,\quad  
\cT_{\pi^*}\tilde Q^1_k=\cT_{\pi^*,\sigma_k^*}\tilde Q^1_k,\quad \cT_{\tilde \pi^i_k} \tilde{Q}^1_{k}=\cT_{\tilde \pi^i_k,\hat \sigma^i_k} \tilde{Q}^1_{k},\quad  \cT_{\pi_k} Q_{\pi_k}=\cT_{\pi_k,\overline{\sigma}_k}  Q_{\pi_k}.\notag
\#

We also separate  the proof into three main steps, similar to the procedure in the proof of  Theorem \ref{thm:err_prop_collab}.

\vskip4pt
 {\noindent \bf Step (i):}  The first step is  to  establish a recursion between the errors of the exact minimizers of the  least-squares problem  \eqref{equ:fitted_least_squares_2} with respect to $Q^*$, the minimax $Q$-function of the game. The error   $Q^* - \tilde{Q}^1_{k+1}$ can be written as
 \#\label{equ:one_step1_2}
	  Q^* - \tilde{Q}^1_{k+1}    
	= \big(Q^* - \cT_{\pi^*} \tilde{Q}^1_k\big) +   \big(\cT_{\pi^*} \tilde{Q}^1_k - \cT\tilde {Q}^1_k \big)+\eta^1_{k+1}+ \varrho^1_{k+1},
	\#
where we denote by $\pi^*$ the equilibrium  policy  with respect to $Q^*$.
Then, by definition of $\cT_{\pi^*}$ and $\cT$ in zero-sum Markov games, we have   $\cT_{\pi^*} \tilde{Q}^1_k \leq  \cT\tilde {Q}^1_k $.
Also, from \eqref{equ:nu_def} and by  relation between $\cT_{\pi,\sigma}$ and  $P_{\pi,\sigma}$ for any $(\pi,\sigma)$, it holds that
\$
Q^* - \cT_{\pi^*} \tilde{Q}^1_k=\cT_{\pi^*} Q^* - \cT_{\pi^*} \tilde{Q}^1_k=\cT_{\pi^*,\sigma^*} Q^* - \cT_{\pi^*,\sigma^*_k} \tilde{Q}^1_k\leq \cT_{\pi^*,\sigma^*_k} Q^* - \cT_{\pi^*,\sigma^*_k} \tilde{Q}^1_k.
\$
Thus, \eqref{equ:one_step1_2} can be upper bounded by 
\#\label{equ:one_step2_2}
	  Q^* - \tilde{Q}^1_{k+1}   
	\leq  \cT_{\pi^*,\sigma^*_k} (Q^* -\tilde{Q}^1_k) +\eta^1_{k+1}+ \varrho^1_{k+1}.
	\#
	
Moreover, we can also establish a lower bound for $Q^* - \tilde Q^1_{k+1}$. Note that 
\$
	  Q^* - \tilde{Q}^1_{k+1}   
	= \big(Q^* - \cT_{\tilde{\pi}_k}  Q^*\big) +   \big(\cT_{\tilde{\pi}_k}  Q^* - \cT\tilde {Q}^1_k \big)+\eta^1_{k+1}+ \varrho^1_{k+1},
	\$
	where $\tilde{\pi}_k$ is the  $\max\min$ policy with respect to $\tilde Q^1_k$. Since $Q^*=TQ^*\geq \cT_{\tilde{\pi}_k}  Q^*$ and $\cT_{\tilde{\pi}_k}  Q^* - \cT_{\tilde{\pi}_k}\tilde {Q}^1_k=\cT_{\tilde{\pi}_k,\tilde{\sigma}_k^*}  Q^* - \cT_{\tilde{\pi}_k,\tilde{\sigma}_k}\tilde {Q}^1_k\geq \cT_{\tilde{\pi}_k,\tilde{\sigma}_k^*}  Q^* - \cT_{\tilde{\pi}_k,\tilde{\sigma}_k^*}\tilde {Q}^1_k$, it follows  that
	\#\label{equ:one_step4_2}
  Q^* - \tilde{Q}^1_{k+1}  \geq  \cT_{\tilde{\pi}_k,\tilde{\sigma}_k^*}(Q^* - \tilde {Q}^1_k )+\eta^1_{k+1}+ \varrho^1_{k+1} .
	\#
Thus, from   the notations in \eqref{equ:nu_def},  combining  \eqref{equ:one_step2_2} and	\eqref{equ:one_step4_2} yields  
\#\label{equ:one_step_err_2}
\gamma  \cdot  P_{\tilde{\pi}_k,\tilde{\sigma}_k^*} (Q^* - \tilde{Q}^1_k) + \eta^1_{k+1}+ \varrho^1_{k+1} \leq Q^* - \tilde{Q}^1_{k+1} \leq \gamma \cdot   P_{\pi^*,\sigma^*_k} (Q^* - \tilde{Q}^1_k) + \eta^1_{k+1}+\varrho^1_{k+1},
\#
from which we obtain the following multi-step error propagation bound and conclude the first step of our proof.

\begin{lemma}
[Multi-step Error Propagation in Competitive MARL]\label{lemma:multi_step_err_2}  
For any  $k, \ell  \in \{0, 1, \ldots, K-1 \}$ with $k < \ell$, we have
\$
	Q^* - \tilde{Q}^1_{\ell} \geq & \sum_{j=k}^{\ell-1} \gamma^{\ell-1-j} \cdot (P_{\tilde{\pi}_{\ell-1},\tilde{\sigma}_{\ell-1}^*}P_{\tilde\pi_{\ell-2},\tilde{\sigma}_{\ell-2}^*}\cdots P_{\tilde\pi_{j+1},\tilde{\sigma}_{j+1}^*}) (\eta^1_{j+1}+\varrho^1_{j+1})\\
	&\qquad\qquad\qquad + \gamma^{\ell-k} \cdot  ( P_{\tilde\pi_{\ell-1},\tilde{\sigma}_{\ell-1}^*} P_{\tilde\pi_{\ell-2},\tilde{\sigma}_{\ell-2}^*}\cdots P_{\tilde\pi_{k},\tilde{\sigma}_{k}^*}) ( Q^* - \tilde{Q}^1_k), \\
	Q^* - \tilde{Q}^1_{\ell} \leq & \sum_{j=k}^{\ell-1} \gamma^{\ell-1-j} \cdot (P_{\pi^*,{\sigma}_{\ell-1}^*}P_{\pi^*,{\sigma}_{\ell-2}^*}\cdots P_{\pi^*,{\sigma}_{j+1}^*}) (\eta^1_{j+1}+\varrho^1_{j+1})\\
	&\qquad\qquad\qquad + \gamma^{\ell-k} \cdot  ( P_{\pi^*,{\sigma}_{\ell-1}^*} P_{\pi^*,{\sigma}_{\ell-2}^*}\cdots P_{\pi^*,{\sigma}_{k}^*}) ( Q^* - \tilde{Q}^1_k), 
\$
where  $\varrho^1_{j+1}$ and $\eta^1_{j+1}$ are  defined in \eqref{equ:rho_eta_k_2}, and we use  $P_{\pi,\sigma} P_{\pi',\sigma'} $  to denote the composition of operators.
\end{lemma}
\begin{proof}
By the linearity  of the operator $P_{\pi,\sigma}$, we can obtain the desired results by  applying the inequalities in \eqref{equ:one_step_err_2} multiple times.
\end{proof}

\vskip4pt
{\noindent \bf Step (ii):} Now we quantify the sub-optimality of the output policy at iteration $k$ of Algorithm \ref{algo:fit_Q_Compet} for Team $1$, i.e., the error $Q^*- Q_{\pi_k}$. Note that $ Q_{\pi_k}$ here represents  the action-value   when the maximizer Team $1$ plays $\pi_k$ and the minimizer Team $2$ plays the optimal counter-policy against $\pi_k$. As argued in \cite{perolat2015approximate}, this is a natural measure of the quality of the policy $\pi_k$. 
The error $Q^*- Q_{\pi_k}$ can be separated as 
\# \label{equ:greedy1_2}
	Q^* - Q_{\pi_k} =   \cT Q^*  - \cT_{\pi_k} Q_{\pi_k} = &~\bigg(\cT_{\pi^*} Q^*- \frac{1}{N}\sum_{i\in\cN}\cT_{\pi^*} \tilde{Q}^{1,i}_k\bigg) +\frac{1}{N}\sum_{i\in\cN}
	\bigg(\cT_{\pi^*} -\cT_{\tilde \pi^i_k}\bigg) \tilde{Q}^{1,i}_k\notag\\
	&~\quad +\bigg(\frac{1}{N}\sum_{i\in\cN}
	\cT_{\tilde \pi^i_k} \tilde{Q}^{1,i}_k-\cT_{\pi_k} \tilde Q^1_k\bigg)+\big(\cT_{\pi_k} \tilde Q^1_k-\cT_{\pi_k} Q_{\pi_k}\big).
%
%
\#
Now we bound the four terms   on the right-hand side of \eqref{equ:greedy1_2} as follows.  
First, by definition of $\tilde \pi^i_k$, we have
\#\label{equ:two_step1_2}
\cT_{\pi^*} \tilde{Q}^{1,i}_k  - \cT_{\tilde \pi^i_k} \tilde{Q}^{1,i}_k=\cT_{\pi^*} \tilde{Q}^{1,i}_k  - \cT\tilde{Q}^{1,i}_k\leq 0,\quad \text{for~all~} i\in\cN.
\#
Moreover, by definition we have
\#\label{equ:trash_buding_1}
\cT_{\pi_k} Q=\cT_{\cE^1(\Qb_k)} Q=\cT_{N^{-1}\cdot\sum_{i\in\cN}\tilde \pi^i_k} Q\geq\frac{1}{N}\sum_{i\in\cN}\cT_{\tilde \pi^i_k}Q
\# 
for any $Q$, where 
$\tilde \pi^i_k$ is the equilibrium  policy of Team $1$ with respect to $\tilde Q^i_k$. 
Also,  we have  from \eqref{equ:nu_def}  that 
\#
\cT_{\pi^*} Q^* - \cT_{\pi^*} \tilde{Q}^{1,i}_k \leq &\gamma \cdot P_{\pi^*,\sigma^{i,*}_k}   ( Q^* - \tilde{Q}^{1,i}_k ),\quad \cT_{\tilde \pi^i_k} \tilde{Q}^{1,i}_k -\cT_{\tilde \pi^i_k} \tilde{Q}^1_{k}\leq \gamma \cdot P_{\tilde \pi^i_k,\hat \sigma^i_k}   ( \tilde{Q}^{1,i}_k - \tilde Q^1_k),\label{equ:two_step2_2_2}\\
&\cT_{\pi_k} \tilde Q^1_k-\cT_{\pi_k} Q_{\pi_k}\leq \gamma\cdot P_{\pi_k,\overline{\sigma}_k}(\tilde Q^1_k- Q_{\pi_k}),\label{equ:two_step2_2_3}
\# 
where the inequalities follow from the fact that $\cT_{\pi^*} Q^*\leq \cT_{\pi^*,\sigma} Q^*$, $\cT_{\tilde \pi^i_k} \tilde{Q}^{1,i}_k\leq \cT_{\tilde \pi^i_k,\sigma} \tilde{Q}^{1,i}_k$, and $\cT_{\pi_k} \tilde Q^1_k\leq \cT_{\pi_k,\sigma} \tilde Q^1_k$ for any $\sigma\in\cP(\cB)$.

By substituting \eqref{equ:two_step1_2},  \eqref{equ:trash_buding_1}, \eqref{equ:two_step2_2_2}, and \eqref{equ:two_step2_2_3} into \eqref{equ:greedy1_2}, we obtain    
\$
	  Q^* - Q_{\pi_k} \leq &~ \frac{\gamma}{N}\sum_{i\in\cN} P_{\pi^*,\sigma^{i,*}_k}(Q^* - \tilde{Q}^{1,i}_k ) + \frac{\gamma}{N}\sum_{i\in\cN} P_{\tilde \pi^i_k,\hat \sigma^i_k} ( \tilde{Q}^{1,i}_k - \tilde Q^1_k) +\gamma\cdot P_{\pi_k,\overline{\sigma}_k}(\tilde Q^1_k- Q_{\pi_k})\\
	 =&~ \frac{\gamma}{N}\sum_{i\in\cN}( P_{\pi^*,\sigma^{i,*}_k} - P_{\pi_k,\overline{\sigma}_k})  ( Q^* - \tilde{Q}^1_k ) +\gamma \cdot P_{\pi_k,\overline{\sigma}_k}  ( Q^* - Q_{\pi_k})\notag\\ 
	&~\quad+  \frac{\gamma}{N}\sum_{i\in\cN} \big(P_{\tilde \pi^i_k,\hat \sigma^i_k}-P_{\pi^*,\sigma^{i,*}_k}\big) ( \tilde{Q}^{1,i}_k - \tilde Q^1_k).
\$
Since $I - \gamma \cdot P_{\pi_k,\overline{\sigma}_k}$ is invertible, we further obtain  
\#\label{equ:greedy_err_2}
0 \leq Q^* - Q_{\pi_k} \leq &~ \frac{\gamma}{N} \cdot  ( I - \gamma\cdot  P_{\pi_k,\overline{\sigma}_k} )^{-1} \cdot \sum_{i\in\cN}( P_{\pi^*,\sigma^{i,*}_k} - P_{\pi_k,\overline{\sigma}_k})  ( Q^* - \tilde{Q}^1_k )\notag\\
&~\quad +\frac{\gamma}{N} \cdot  ( I - \gamma\cdot  P_{\pi_k,\overline{\sigma}_k} )^{-1}\cdot\sum_{i\in\cN} \big(P_{\tilde \pi^i_k,\hat \sigma^i_k}-P_{\pi^*,\sigma^{i,*}_k}\big) ( \tilde{Q}^{1,i}_k - \tilde Q^1_k).
\#
Moreover, by  setting $\ell=K$ and $k=0$ in Lemma \ref{lemma:multi_step_err_2}, we obtain that for any $i\in\cN$
\#\label{equ:greedy_err_2_step2}
&\gamma\cdot( P_{\pi^*,\sigma^{i,*}_K} - P_{\pi_K,\overline{\sigma}_K})  ( Q^* - \tilde{Q}^1_K )\notag\\
	& \quad \leq  \sum_{j=0}^{K-1} \gamma^{K  - j} \cdot \bigl [  (P_{\pi^*,\sigma^{i,*}_K}P_{\pi^*,\sigma^*_{K-1}}\cdots P_{\pi^*,\sigma^*_{j+1}}) -  (P_{\pi_{K},\overline{\sigma}_K}P_{\tilde\pi_{K-1},\tilde\sigma^*_{K-1} }\cdots P_{\tilde\pi_{j+1},\tilde\sigma^*_{j+1} }) \bigr ] \notag\\
	&\qquad\quad\cdot(\eta^1_{j+1}+\varrho^1_{j+1})  +  \gamma^{K + 1}\cdot \bigl [ (P_{\pi^*,\sigma^{i,*}_K}P_{\pi^*,\sigma^*_{K-1}}\cdots P_{\pi^*,\sigma^*_{0}})\notag\\
	& \qquad\quad  -  (P_{\pi_{K},\overline{\sigma}_K}P_{\tilde\pi_{K-1},\tilde\sigma^*_{K-1} }\cdots P_{\tilde\pi_{0},\tilde\sigma^*_{0} }) \bigr ] ( Q^* - \tilde{Q}^1_0).
\#

Also, we denote the second term on the right-hand side of \eqref{equ:greedy_err_2} by $\xi^1_k$, i.e., 
\#\label{equ:xi_def_2}
\xi^1_k=\frac{\gamma}{N} \cdot  ( I - \gamma\cdot  P_{\pi_k,\overline{\sigma}_k} )^{-1}\cdot\sum_{i\in\cN} \big(P_{\tilde \pi^i_k,\hat \sigma^i_k}-P_{\pi^*,\sigma^{i,*}_k}\big) ( \tilde{Q}^{1,i}_k - \tilde Q^1_k),
\#
which depends on the quality of the solution to \eqref{equ:fitted_least_squares_2}  returned by the decentralized optimization algorithm.
By combining \eqref{equ:greedy_err_2_step2} and \eqref{equ:xi_def_2}, we obtain the bound of \eqref{equ:greedy_err_2} at the final iteration $K$ as  
\#\label{equ:greedy_err_multiple_2}
	Q^* - Q_{\pi_{K} } \leq &~\frac{(  I - \gamma\cdot  P_{\pi_K,\overline{\sigma}_K} )^{-1}}{N} \cdot   \sum_{i\in\cN}\bigg \{ \sum_{j=0}^{K-1} \gamma^{K  - j} \cdot \bigl [  (P_{\pi^*,\sigma^{i,*}_K}P_{\pi^*,\sigma^*_{K-1}}\cdots P_{\pi^*,\sigma^*_{j+1}}) \notag\\
	&~\quad -  (P_{\pi_{K},\overline{\sigma}_K}P_{\tilde\pi_{K-1},\tilde\sigma^*_{K-1} }\cdots P_{\tilde\pi_{j+1},\tilde\sigma^*_{j+1} }) \bigr ] (\eta^1_{j+1}+\varrho^1_{j+1}) \notag\\
	&~\quad +  \gamma^{K + 1}\cdot \bigl [ (P_{\pi^*,\sigma^{i,*}_K}P_{\pi^*,\sigma^*_{K-1}}\cdots P_{\pi^*,\sigma^*_{0}}) -  (P_{\pi_{K},\overline{\sigma}_K}P_{\tilde\pi_{K-1},\tilde\sigma^*_{K-1} }\cdots P_{\tilde\pi_{0},\tilde\sigma^*_{0} }) \bigr ]\notag\\
	&~\quad\cdot ( Q^* - \tilde{Q}^1_0)\bigg\} +\xi^1_{K}.
\#

For simplicity, we define the coefficients $\{\alpha_{j}\}_{ k=0}^K$  as  in \eqref{equ:define_alpha_param}, and    linear operators $\{\cL^1_k \}_{ k=0}^K $ as 
\$
\cL^1_j =&~ \frac{(1 - \gamma)}{2N} \cdot (  I - \gamma\cdot  P_{\pi_k,\overline{\sigma}_k} )^{-1} \sum_{i\in\cN} \bigl [   (P_{\pi^*,\sigma^{i,*}_K}P_{\pi^*,\sigma^*_{K-1}}\cdots P_{\pi^*,\sigma^*_{j+1}}) \notag\\
	&~\quad + (P_{\pi_{K},\overline{\sigma}_K}P_{\tilde\pi_{K-1},\tilde\sigma^*_{K-1} }\cdots P_{\tilde\pi_{j+1},\tilde\sigma^*_{j+1} }) \bigr ], ~~\text{for}~~ 0 \leq j \leq K-1,\\
\cL^1_K  =&~ \frac{(1 - \gamma)}{2N} \cdot  (  I - \gamma\cdot  P_{\pi_k,\overline{\sigma}_k} )^{-1}\sum_{i\in\cN} \bigl [   (P_{\pi^*,\sigma^{i,*}_K}P_{\pi^*,\sigma^*_{K-1}}\cdots P_{\pi^*,\sigma^*_{0}}) \notag\\
	&~\quad + (P_{\pi_{K},\overline{\sigma}_K}P_{\tilde\pi_{K-1},\tilde\sigma^*_{K-1} }\cdots P_{\tilde\pi_{0},\tilde\sigma^*_{0} }) \bigr ].
\$
Then, we take absolute value on both sides of  \eqref{equ:greedy_err_multiple_2} to obtain 
\#\label{equ:absolute_value_bound_2}
\bigl | Q^* (s,a,b) - Q_{\pi_K} (s,a,b) \bigr | \leq &~ \frac{2 \gamma ( 1 - \gamma^{K+1} ) }{(1- \gamma)^2}  \cdot \biggl [ \sum_{j=0}^{K-1} \alpha_j \cdot \bigl  ( \cL^1_j |\eta^1_{j+1}+ \varrho^1_{j+1} | \bigr ) (s,a,b) \notag\\
&~\quad + \alpha_K  \cdot \bigl ( \cL^1_K | Q^* - \tilde Q_0 | \bigr ) (s,a,b) \biggr ]+|\xi^1_K(s,a,b)|,
\# 
for any $(s, a, b) \in \cS \times \cA\times \cB$. 
This completes the second step of the proof.

\vskip4pt
{\noindent \bf Step (iii):} We note that \eqref{equ:absolute_value_bound_2} has almost the identical form as \eqref{equ:absolute_value_bound}, with the fact that  $\sum_{j=0}^K\alpha_j=1$ and for all $j=0,\cdots,K$, the linear operators $\cL^1_j$ are positive and satisfy $\cL^1_j\bm1=\bm1$. 
 Hence, the proof here follows directly from the {\bf{Step (iii)}} in \S\ref{proof:thm:err_prop}, from which we obtain that for any fixed $\mu\in\cP(\cS\times\cA\times\cB)$
\#\label{equ:p_norm_concentra_coeff_2}
&\| Q^* - Q_{\pi_{K}} \|_{\mu} \notag\\  &\quad\leq  \frac{4 \gamma \cdot\big(\phi^{\text{MG}}_{\mu,\nu}\big)^{1/2} }{\sqrt{2}(1- \gamma)^2}   \cdot (\|\eta^1\|_\nu+\|\varrho^1\|_\nu)+\frac{4\sqrt{2}\cdot Q_{\max} }{(1- \gamma)^2}   \cdot \gamma^{K/2} + \frac{2\sqrt{2}\gamma}{1-\gamma}\cdot \overline{\epsilon}^1_K,
\#
where we denote by $\|\eta^1\|_\nu =\max_{j=0,\cdots,K-1}\| \eta^1_{j+1}\|_{\nu}$, $\|\varrho^1\|_\nu =\max_{j=0,\cdots,K-1}\| \varrho^1_{j+1}\|_{\nu}$, and $\overline{\epsilon}^1_K=[N^{-1}\cdot\sum_{i\in\cN}(\epsilon^{1,i}_K)^2]^{1/2}$, with $\epsilon^{1,i}_K$ being the one-step decentralized computation error from Assumption \ref{assume:concentrability}. 
To obtain  \eqref{equ:p_norm_concentra_coeff_2}, 
  we use the definition of concentrability coefficients $\kappa^{\text{MG}}$ and  Assumption \ref{assume:concentrability}, similarly  as the usage of the definition of $\kappa^{\text{MDP}}$ and  Assumption \ref{assume:concentrability} in the {\bf{Step (iii)}} in \S\ref{proof:thm:err_prop}.

Recall that $\eta^1_{j+1}$ is defined as $\eta^1_{j+1}=\cT\tilde {Q}^1_{j} -\tcT\tilde{\Qb}^1_{j}$, which can also be further bounded by the one-step decentralized computation error from  Assumption \ref{assum:one_step_comp_error}. The following lemma characterizes such a relationship, playing the similar role as Lemma \ref{lemma:max_minus_max}.

\begin{lemma}\label{lemma:max_minus_max_2}
	Under Assumption \ref{assum:one_step_comp_error}, for any $j=0,\cdots,K-1$, it holds that
	$
	\|\eta^1_{j+1}\|_\nu\leq \sqrt{2}\gamma\cdot\overline{\epsilon}^1_{j}
	$, where $\overline{\epsilon}^1_{j}=[N^{-1}\cdot\sum_{i\in\cN}(\epsilon^{1,i}_{j})^2]^{1/2}$ and $\epsilon^{1,i}_{j}$ is defined as in Assumption  \ref{assum:one_step_comp_error}.
\end{lemma}
\begin{proof}
Note that 
\$
&|\eta^1_{j+1}(s,a,b)|=\Big|\big(\cT\tilde {Q}^1_{j}\big)(s,a,b) -\big(\tcT\tilde{\Qb}^1_{j}\big)(s,a,b)\Big|\notag\\
&\quad\leq \gamma\cdot\frac{1}{N}\sum_{i\in\cN}\EE_{s' \sim P(\cdot  \given s,a,b)}\Bigl[\Big|\max_{\pi'\in\cP(\cA)}\min_{\sigma'\in\cP(\cB)}\EE_{\pi',\sigma'} \big[\tilde Q^1_j(s', a',b')\big]\notag\\
&\quad\quad\quad -\max_{\pi'\in\cP(\cA)}\min_{\sigma'\in\cP(\cB)}\EE_{\pi',\sigma'} \big[\tilde Q^{1,i}_j(s', a',b')\big]\Big| \Bigr].
\$
Now we claim that for any $s'\in\cS$ 
\small
\#\label{equ:pf_max_minus_max_1_2}
\Big|\max_{\pi'\in\cP(\cA)}\min_{\sigma'\in\cP(\cB)}\EE_{\pi',\sigma'} \big[\tilde Q^1_j(s', a',b')\big] -\max_{\pi'\in\cP(\cA)}\min_{\sigma'\in\cP(\cB)}\EE_{\pi',\sigma'} \big[\tilde Q^{1,i}_j(s', a',b')\big]\Big|\leq C_1\cdot \epsilon^{1,i}_j,
\# 
\normalsize 
for any constant $C_1>1$. 
For notational simplicity, let $g_j(\pi',\sigma')=\EE_{\pi',\sigma'} \big[\tilde Q^1_j(s', a',b')\big]$ and $g^i_j(\pi',\sigma')=\EE_{\pi',\sigma'} \big[\tilde Q^{1,i}_j(s', a',b')\big]$. 
Suppose \eqref{equ:pf_max_minus_max_1_2} does not hold, then either 
\#\label{equ:lemma_case_1}
\max_{\pi'\in\cP(\cA)}\min_{\sigma'\in\cP(\cB)}g_j(\pi',\sigma')\geq \max_{\pi'\in\cP(\cA)}\min_{\sigma'\in\cP(\cB)}g^i_j(\pi',\sigma')+C_1\cdot \epsilon^i_j,
\#  or 
\$
\max_{\pi'\in\cP(\cA)}\min_{\sigma'\in\cP(\cB)}g_j(\pi',\sigma')\leq \max_{\pi'\in\cP(\cA)}\min_{\sigma'\in\cP(\cB)}g^i_j(\pi',\sigma')-C_1\cdot \epsilon^i_j.
\$
In the first case, let $(\pi'_*,\sigma'_*)$ be the minimax strategy pair for $g^i_j(\pi',\sigma')$, such that $g^i_j(\pi'_*,\sigma'_*)=\max_{\pi'\in\cP(\cA)}\min_{\sigma'\in\cP(\cB)}g_j(\pi',\sigma')$. 
Note that $\sigma'_*=\sigma'_*(\pi'_*)\in\argmin_{\sigma'\in\cP(\cB)}g_j(\pi'_*,\sigma')$ is a function of $\pi'_*$. 
By Assumption \ref{assum:one_step_comp_error}, $\tilde Q^1_j(s', \cdot,\cdot)$ and $\tilde Q^{1,i}_j(s', \cdot,\cdot)$ are close to each other uniformly over $\cA\times\cB$. Thus, by the linearity of $g_j$ and $g^i_j$, we have
$
g_j(\pi'_*,\sigma'_*)\geq g^i_j(\pi'_*,\sigma'_*)-\epsilon^{1,i}_j.
$ 
Together with \eqref{equ:lemma_case_1}, we obtain 
\#\label{equ:lemma_case_1_counter}
g_j(\pi'_*,\sigma'_*)\geq \max_{\pi'\in\cP(\cA)}\min_{\sigma'\in\cP(\cB)}g^i_j(\pi',\sigma')+(C_1-1)\cdot\epsilon^{1,i}_j.
\# 
However, \eqref{equ:lemma_case_1_counter} cannot hold with any $C_1>1$ since $\max_{\pi'\in\cP(\cA)}\min_{\sigma'\in\cP(\cB)}g^i_j(\pi',\sigma')$ should be no smaller than  $g_j(\pi,\sigma)$ with any $(\pi,\sigma)$ pair that satisfies $\sigma=\sigma(\pi)\in\argmin_{\sigma\in\cP(\cB)}g_j(\pi,\sigma)$, including $g_j(\pi'_*,\sigma'_*)$.  Similarly, one can show that the second case cannot occur. Thus, the claim \eqref{equ:pf_max_minus_max_1_2} is proved. Letting $C_1=\sqrt{2}$, we obtain that 
\$
|\eta^1_{j+1}(s,a)|^2\leq \gamma^2  \bigg(\frac{1}{N}\sum_{i\in\cN}\sqrt{2}\epsilon^{1,i}_j\bigg)^2\leq \big(\sqrt{2}\gamma\big)^2  \frac{1}{N}\sum_{i\in\cN} \big(\epsilon^{1,i}_j\big)^2, 
\$
where the second inequality follows Jensen's inequality. Taking expectation over $\nu$, we  obtain the desired bound. 
\end{proof}

From Lemma \ref{lemma:max_minus_max_2}, we can further simplify \eqref{equ:p_norm_concentra_coeff_2} to obtain the desired bound in Theorem \ref{thm:err_prop_compet}, which concludes the proof. 
\end{proof}

\subsection{Proof of Theorem \ref{thm:err_one_step_approx_2}}\label{proof:thm:err_one_step_approx_2}
\begin{proof}
	The proof is similar to the proof of Theorem \ref{thm:err_one_step_approx} in \S\ref{proof:thm:err_one_step_approx}. To avoid repetition, we will only emphasize the difference from there. 
	First for any fixed  $\Qb=[Q^i]_{i\in\cN}\in\cH^N$ and $f\in\cH$, we define  $\ell_{f,\Qb}$ as 
	\$
	\ell_{f,\Qb}(s,a,b,s')&=\Big\{\overline{r}(s,a,b) + \gamma/ N\cdot\sum_{i\in\cN}\max_{\pi'\in\cP(\cA)}\min_{\sigma'\in\cP(\cB)} \EE_{\pi',\sigma'}[Q^i (s', a',b')]-f(s,a,b)\Big\}^2,
	\$
	where $\overline{r}(s,a,b)=N^{-1}\cdot\sum_{i\in\cN}r^i(s,a,b)$ with $r^i(s,a,b)\sim R^i(s,a,b)$. 
	Thus, we can similarly define $\hat{L}_T(f;\Qb)=T^{-1}\cdot\sum_{t=1}^T \ell_{f,\Qb}(s_t,a_t,b_t,s_{t+1})$ and ${L}(f;\Qb)=\EE[\hat{L}_T(f;\Qb)]$. 
	Let $f'\in\argmin_{f\in\cH}\hat{L}_T(f;\Qb)$ and $Z_t=(s_t,\{a^i_t\}_{i\in\cN},\{b^j_t\}_{j\in\cM},s_{t+1})$, we have
	\#
	&\|f'-\tcT \Qb\|^2_{\nu}-\inf_{f\in\cH}\|f-\tcT \Qb\|^2_{\nu}\leq 2\sup_{\ell_{f,\Qb}\in\cL_{\cH}}\bigg|\frac{1}{T}\sum_{t=1}^T\ell_{f,\Qb}(Z_t)-\EE[\ell_{f,\Qb}(Z_1)]\bigg|.\label{equ:unform_dev_bnd_2}
	\#
	Now it suffices to show that the $P_0$ as defined in \eqref{equ:def_P_0} satisfies  $P_0<\delta$. 
	We will identify the choice of $C_1$ and $C_2$ in $\Lambda_T(\delta)$ shortly. To show $P_0<\delta$, from \eqref{equ:uniform_dev_prob}, we  need to bound the empirical covering number $\cN_1(\epsilon/16,\cL_{\cH},(Z'_t;t\in H))$, where $H$ is the set of ghost samples. The bound is characterized by the following lemma. 
		
	\begin{lemma}\label{lemma:cover_num_bnd_2}
		Let $Z^{1:T}=(Z_1,\cdots,Z_T)$, with $Z_t=(s_t,\{a^i_t\}_{i\in\cN},\{b^j_t\}_{j\in\cM},s_{t+1})$. Recall that  $\tilde R_{\max}=(1+\gamma)Q_{\max}+R_{\max}$, and   $A=|\cA|, B=|\cB|$. Then, under Assumption  \ref{assum:capacity_func}, it holds that { 
		\$
		\cN_1(\epsilon,\cL_{\cH},Z^{1:T})\leq e^{N+1}(V_{\cH^+}+1)^{N+1}(AB)^{NV_{\cH^+}}Q_{\max}^{(N+1)V_{\cH^+}} \bigg(\frac{4e\tilde R_{\max}(1+\gamma)}{\epsilon}\bigg)^{(N+1)V_{\cH^+}}.
		\$}
	\end{lemma}
	\begin{proof}
	We  first bound the empirical $\ell^1$-distance between any   $l_{f,\Qb}$ and $l_{\tilde f,\tilde \Qb}$ in $\cL_{\cH}$  as
	\small
	\#\label{equ:cov_num_1_2}
	&\frac{1}{T}\sum_{t=1}^T\big|l_{f,\Qb}(Z_t)-l_{\tilde f,\tilde \Qb}(Z_t)\big|\leq  2\tilde R_{\max}\bigg\{\frac{1}{T}\sum_{t=1}^T \big|f(s_t,a_t,b_t)-\tilde f(s_t,a_t,b_t)\big|+ \frac{\gamma}{T}\sum_{t=1}^T \Big|\frac{1}{N}\sum_{i\in\cN}\notag\\
	&\quad \max_{\pi'\in\cP(\cA)}\min_{\sigma'\in\cP(\cB)} \EE_{\pi',\sigma'}[Q^i (s_{t+1}, a',b')]-\frac{1}{N}\sum_{i\in\cN}\max_{\pi'\in\cP(\cA)}\min_{\sigma'\in\cP(\cB)} \EE_{\pi',\sigma'}[\tilde Q^i (s_{t+1}, a',b')]\Big|\bigg\}.
	\#
	\normalsize
	Let $\cD_Z=\{(s_1,\{a^i_1\}_{i\in\cN},\{b^j_1\}_{j\in\cM}),\cdots,(s_T,\{a^i_T\}_{i\in\cN},\{b^j_T\}_{j\in\cM})\}$ and $y_Z=(s_2,\cdots,s_{T+1})$, then from \eqref{equ:cov_num_1_2}, 
	the empirical covering number $\cN_1\big(2\tilde R_{\max}(1+\gamma)\epsilon,\cL_{\cH},Z^{1:T}\big)$ can be bounded by 
	\#\label{equ:empirical_cov_num_1_2}
	\cN_1\big(2\tilde R_{\max}(1+\gamma)\epsilon,\cL_{\cH},Z^{1:T}\big)\leq \cN_1(\epsilon,\cH^{\vee}_N,y_Z)\cdot\cN_1(\epsilon,\cH,\cD_Z),
	\#
	where $\cH^{\vee}_N$ here is defined as 
	\$
	\cH^{\vee}_N=\bigg\{V:V(\cdot)=N^{-1}\cdot\sum_{i\in\cN}\max_{\pi'\in\cP(\cA)}\min_{\sigma'\in\cP(\cB)} \EE_{\pi',\sigma'}[Q^i (\cdot, a',b')]\text{~and~}\Qb\in\cH^N\bigg\}.
	\$
	We further bound the first covering number $\cN_1(\epsilon,\cH^{\vee}_N,y_Z)$ by the  following lemma, which is proved later in  \S\ref{sec:append_proofs}.
	
	\begin{lemma}\label{lemma:empirical_cov_num_H_vee_2}
		For any fixed $y_{Z}=(y_1,\cdots,y_T)$, let 
		$\cD_y=\{(y_t,a_j,b_k)\}_{t\in[T],j\in[A],k\in[B]},
		$ where recall that $A=|\cA|, B=|\cB|$, and $\cA=\{a_1,\cdots,a_A\},\cB=\{b_1,\cdots,b_B\}$. Then, under Assumption  \ref{assum:capacity_func}, it holds that { 
		\$
		\cN_1(\epsilon,\cH^{\vee}_N,y_Z)\leq \big[\cN_1({\epsilon}/{(AB)},\cH,\cD_y)\big]^N\leq \bigg[e(V_{\cH^+}+1)\bigg(\frac{2eQ_{\max}AB}{\epsilon}\bigg)^{V_{\cH^+}}\bigg]^N.
		\$}
	\end{lemma}
	\vspace{3pt}
	Furthermore, we can bound $\cN_1(\epsilon,\cH,\cD_Z)$ 	by  Proposition \ref{prop:hauss_empirical_cov}, which together with Lemma \ref{lemma:empirical_cov_num_H_vee_2} yields  a bound for \eqref{equ:empirical_cov_num_1_2}
	\$
	&\cN_1\big(2\tilde R_{\max}(1+\gamma)\epsilon,\cL_{\cH},Z^{1:T}\big)  \leq e^{N+1}(V_{\cH^+}+1)^{N+1}(AB)^{NV_{\cH^+}}Q_{\max}^{(N+1)V_{\cH^+}} \bigg(\frac{2e}{\epsilon}\bigg)^{(N+1)V_{\cH^+}}, 
	\$
	which completes the proof by 
	replacing $2\tilde R_{\max}(1+\gamma)\epsilon$ by $\epsilon$.
	\end{proof} 

By choosing $C_1$ as 
\$
 C_1=16\cdot e^{N+1}(V_{\cH^+}+1)^{N+1}(AB)^{NV_{\cH^+}}Q_{\max}^{(N+1)V_{\cH^+}} \big[{64\cdot e\tilde R_{\max}(1+\gamma)}\big]^{(N+1)V_{\cH^+}},
\$
 and applying Lemma \ref{lemma:cover_num_bnd_2}, 
 we bound $\cN_1(\epsilon/16,\cL_{\cH},(Z'_t;t\in H))$	 as  
\$
&\cN_1(\epsilon/16,\cL_{\cH},(Z'_t;t\in H))\leq  \frac{C_1}{16}\bigg(\frac{1}{\epsilon}\bigg)^V,
\$
where $V=(N+1)V_{\cH^+}$. Further, choosing $C_2=1/(2048\cdot \tilde R^4_{\max})$, we obtain $P_0<\delta$ by  Lemma \ref{lemma:high_prob_to_dev}. In particular, $\Lambda_T$ can be written as $\Lambda_T=K_1+K_2\cdot N$, with 
\$
 &K_1=K_1\big(V_{\cH^+}\log(T),\log(1/\delta),\log(\tilde R_{\max}),V_{\cH^+}\log(\overline{\beta})\big),\\ &K_2=K_2\big(V_{\cH^+}\log(T),V_{\cH^+}\log(\overline{\beta}),V_{\cH^+}\log[\tilde R_{\max}(1+\gamma)],V_{\cH^+}\log(Q_{\max}), V_{\cH^+}\log(AB)\big), 
 \$ 
being some constants that depend on the parameters in the brackets. 
This  concludes  the proof.
\end{proof}

\subsection{Proof of Corollary \ref{coro:complex_LFA}}\label{proof:coro:LFA}
\begin{proof}
The proof proceeds by controlling the three error terms in the bound (except the inherent approximation error) in  Theorems \ref{thm:main_thm_collab} and \ref{thm:main_thm_compet} by $\epsilon/3$ for any $\epsilon>0$. In particular, to show the first argument for the cooperative  setting,  letting 
\$
{\frac{4\sqrt{2}\cdot Q_{\max} }{(1- \gamma)^2}   \cdot \gamma^{K/2}}\leq {\frac{4\sqrt{2}\cdot Q_{\max} }{(1- \gamma)^2}   \cdot \frac{(1- \gamma)^2\epsilon}{12\sqrt{2}\cdot Q_{\max} }}= \frac{\epsilon}{3},
\$
we immediately obtain that $K$ is linear in $\log(1/\epsilon)$, $\log[1/(1-\gamma)]$,  and $\log(Q_{\max})$. 
Letting the estimation error be controlled by $\epsilon/3$, we have
\$
&C^{\text{MDP}}_{\mu, \nu}\cdot\bigg\{\frac{\Lambda_T(\delta/K)[\Lambda_T(\delta/K)/b\vee 1]^{1/\zeta}}{T/(2048\cdot \tilde R^4_{\max})}\bigg\}^{1/4}\\
&\quad =\frac{4 \gamma \cdot\big(\phi^{\text{MDP}}_{\mu,\nu}\big)^{1/2} }{\sqrt{2}(1- \gamma)^2}\cdot\bigg\{\frac{\Lambda_T(\delta/K)[\Lambda_T(\delta/K)/b\vee 1]^{1/\zeta}}{T/(2048\cdot \tilde R^4_{\max})}\bigg\}^{1/4}\leq \frac{\epsilon}{3}.
\$
By definition of $\Lambda_T$ in  Theorem \ref{thm:main_thm_collab}, we obtain that  $T$ is polynomial in $1/\epsilon$, $\gamma/(1-\gamma)$, $1/{\tilde R_{\max}}$, $\log(1/\delta)$, $\log(\overline{\beta})$, and $N\log(A)$. 
As for the decentralized computation error,  let 
\#\label{equ:comp_epsilon}
\sqrt{2}\gamma\cdot C^{\text{MDP}}_{\mu, \nu}\cdot\overline{\epsilon}+ \frac{2\sqrt{2}\gamma}{1-\gamma}\cdot \overline{\epsilon}_K=\sqrt{2}\gamma\cdot\frac{4 \gamma \cdot\big(\phi^{\text{MDP}}_{\mu,\nu}\big)^{1/2} }{\sqrt{2}(1- \gamma)^2}\cdot \overline{\epsilon}+ \frac{2\sqrt{2}\gamma}{1-\gamma}\cdot \overline{\epsilon}_K\leq \frac{\epsilon}{3}.
\#
Note that  under Assumptions  \ref{assum:linear_features} and \ref{assum:consensus_mat}, we can apply 
 Lemma \ref{lemma:nedic_geo_rate}, which shows that there exist constants $\lambda\in[0,1)$ and $C_3>0$,   such that at each iteration $k$ of Algorithm \ref{algo:fit_Q_Collab}
 \#\label{equ:bound_DIGing_adapted}
 \sqrt{\sum_{i\in\cN}\|\theta_{k,l}^i-\theta^*_k\|^2}\leq C_3\cdot \lambda^l,
 \#
 where $\theta^*_k$ corresponds to the exact solution to \eqref{equ:fitted_least_squares} at this iteration $k$, i.e., $\tilde Q_k=(\theta^*_k)^\top \varphi$, and $\theta_{k,l}^i$ represents the estimate of $\theta^*_k$ of agent $i$  at iteration $l$ of Algorithm \ref{algo:DIGing}. Thus, if Algorithm \ref{algo:DIGing} terminates after  $L$ iterations, we have $\tilde Q^i_k=(\theta_{k,L}^i)^\top \varphi$, where recall that $\tilde Q^i_k$ denotes the output of the decentralized optimization step in Algorithm \ref{algo:fit_Q_Collab}. Since  the features $\varphi$ are uniformly bounded,  we obtain from  \eqref{equ:bound_DIGing_adapted} that there exists a constant $C_4>0$, such that for any $(s,a)\in\cS\times\cA$
 \$
 \sqrt{\frac{1}{N}\sum_{i\in\cN}|\tilde Q^i_k(s,a)-\tilde Q_k(s,a)|^2}\leq C_4\sqrt{\frac{1}{N}\sum_{i\in\cN}\|\theta_{k,L}^i-\theta^*_k\|^2}\leq \frac{C_4C_3}{\sqrt{N}}\cdot \lambda^L. 
 \$
 Thus, we can choose ${C_4C_3}/{\sqrt{N}}\cdot \lambda^L$ to bound the decentralized computation error $\overline{\epsilon}_k$. These arguments apply to all iterations $k\in[K]$, which means  that we can bound both $\overline{\epsilon}=\max_{0\leq k\leq K-1}\overline{\epsilon}_k$  and $\overline{\epsilon}_K$ in \eqref{equ:comp_epsilon} by $C_5\cdot \lambda^L$ for some constant $C_5>0$. Therefore, we conclude from \eqref{equ:comp_epsilon} that the number of iterations $L$ is linear in $\log(1/\epsilon)$, $\log[\gamma/(1-\gamma)]$. 
 The arguments above also hold for the terms of the bound in Theorem \ref{thm:main_thm_compet}, except that for $T$, the dependence on $N\log(A)$ changes to $N\log(AB)$.  This completes the proof. 
\end{proof}

\section{Numerical Results} \label{sec:simulation}
 
In this section, we present some numerical results to justify the effectiveness of our algorithms. As the two-team setting strictly covers the one-team setting by setting the opponent team dummy, we focus on the former in the simulations.  
  
\subsection{Simulation Setup}
For generality, we consider randomly generated Markov games in our simulations, similar to those in \cite{zhang2018fully,dann2014policy} with randomly generated MDPs. More specifically, we consider a two-team zero-sum Markov game, where Team $1$ and $2$ have $N=5$ and $M=9$ agents, respectively. Each agent has a binary action space, i.e., $\cA^i=\cB^j=\{0,1\}$ for all $i\in\cN$ and $j\in\cM$. Note that the cardinality of the joint action space $\cA\times\cB$ is thus $2^7\times 2^9= 65536$. The state space $\cS$ has in total $|\cS|=10$ states.  The elements in the transition probability matrix $P$ are uniformly sampled from the interval $[0,1]$ and normalized to make $P$ a stochastic matrix. We also add a small constant $10^{-4}$ onto each element in the matrix to ensure the geometric   ergodicity of the transition. We then  use uniform distribution over the action spaces $\cA$ and $\cB$ at all states $s\in\cS$ as the behavior policy to sample the batch data $\cD$. Note that this setup of $P$ and the behavior policies ensure  that the sample path in $\cD$ satisfy  Assumption \ref{assum:sample_path} if we sample the data after the Markov chain mixes, as geometric ergodicity implies exponential $\beta$-mixing \citep{davydov1973mixing,doukhan2012mixing}. 

 For each agent $i$ and each state-joint-action pair $(s,a,b)$, the mean reward $R^{1,i}(s,a,b)$ is sampled uniformly from $[0,5]$, with some different baseline $R^{1,i}_{\tt{b}}$ added for each agent $i$, where $R^{1,i}_{\tt{b}}$ is uniformly drawn from $[0,3]$ and varies  among agents. $R^{1,i}_{\tt{b}}$ ensures that the mean of $R^{1,i}$ over $s,a,b$ varies among agents.  The instantaneous reward is sampled uniformly around $R^{1,i}(s_t,a_t,b_t)$, i.e.,  $r^{1,i}_t\sim{\tt{Unif}}[R^{1,i}(s_t,a_t,b_t)-0.5,R^{1,i}(s_t,a_t,b_t)+0.5]$. The reward function $R^{2,j}$ and the instantaneous reward $r^{2,j}_t$ are generated in the same way. 
 For computational tractability, the $Q$-functions are approximated via linear function classes as $\cH_\Theta=\{f(s,a,b;\theta)=\theta^\top\varphi(s,a,b): \theta\in\RR^{d}\}$, where $\theta\in\RR^{d}$ is the parameter and $\varphi(s,a,b)\in\RR^{d}$ is the feature vector. 
Let $\Phi:=[\varphi(s_1,a_1,b_1),~\varphi(s_1,a_1,b_2),\cdots,~\varphi(s_{|\cS|},\\a_{A},b_{B})]^\top \in\RR^{|\cS|AB\times d}$ denote  the feature matrix, where each row of $\Phi$ corresponds to  each feature vector $\varphi(s,a,b)^\top$. We choose $\varphi$ such that 
\$
\Phi=\left[\begin{matrix}
	1 & 0  & \cdots & 0 & 1\\
	0 & 1  & \cdots & 0 & 1\\
	& & \cdots & & \\
	0 & 0  & \cdots & 1 & 1\\
	1 & 0  & \cdots & 0 & 1\\
	& & \cdots & & \\
\end{matrix}\right],
\$ 
which follows from the feature selection for random MDPs in \cite{joseph2018incremental}, by making  the last column be an all-one vector and to account for the bias. The parameter dimension $d$ is chosen to be $\ll |\cS||\cA||\cB|=655360$. The consensus matrix $\Cb_l$ is chosen independent and identically distributed along time $l$ 
by normalizing the absolute Laplacian matrix of a random graph, as in  \cite{zhang2018fully,nedic2017achieving}, in order to satisfy Assumption \ref{assum:consensus_mat}. Discount factor  $\gamma$ is set to be $0.9$. The performance evaluation distribution $\mu\in\cP(\cS\times\cA\times\cB)$ is a uniform distribution over $\cS\times\cA\times\cB$.

\subsection{Results}

\begin{figure*}[!t]
	\centering
	\begin{tabular}{ccc}
		\hskip-13pt\includegraphics[width=0.34\textwidth]{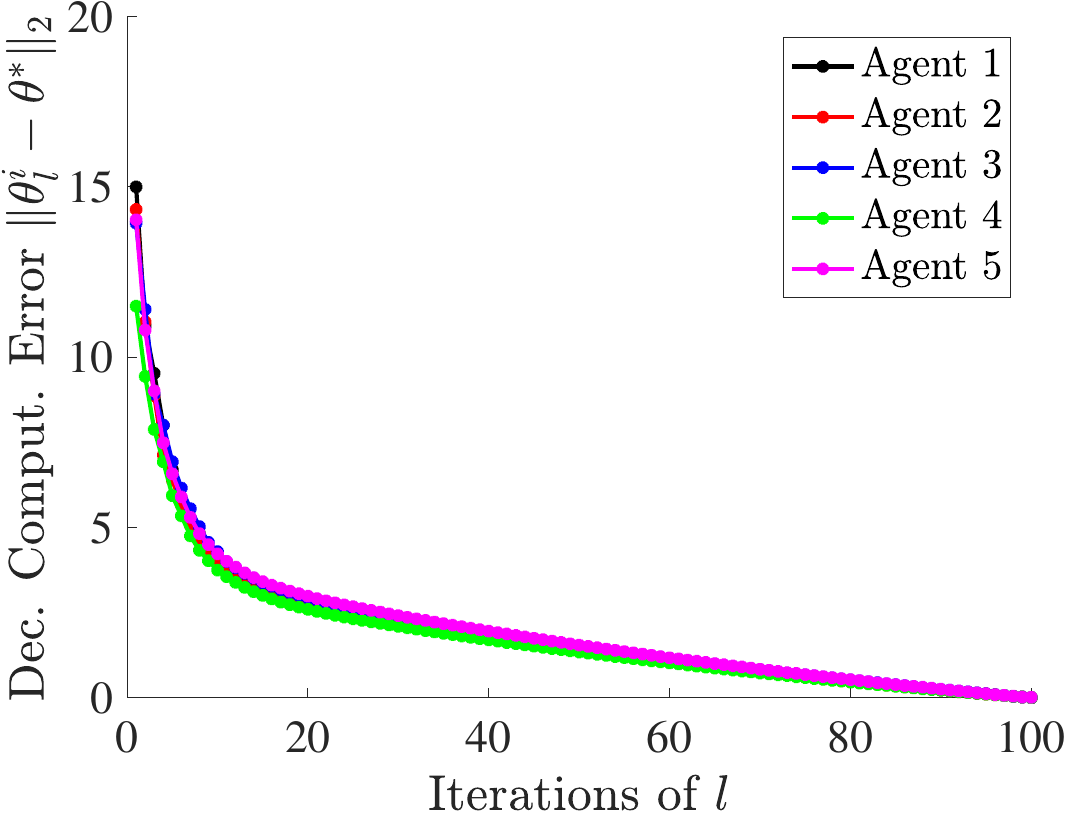} 
		&
		\hskip -10pt\includegraphics[width=0.34\textwidth]{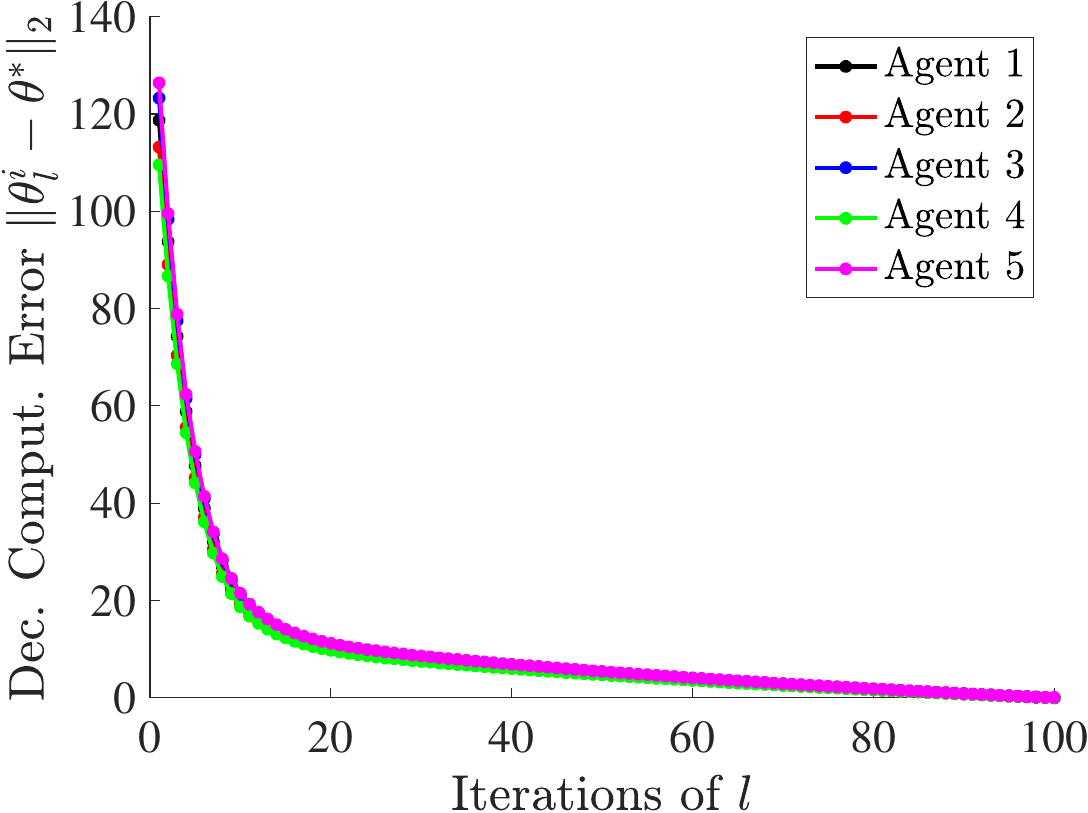}
		&
		\hskip -10pt\includegraphics[width=0.34\textwidth]{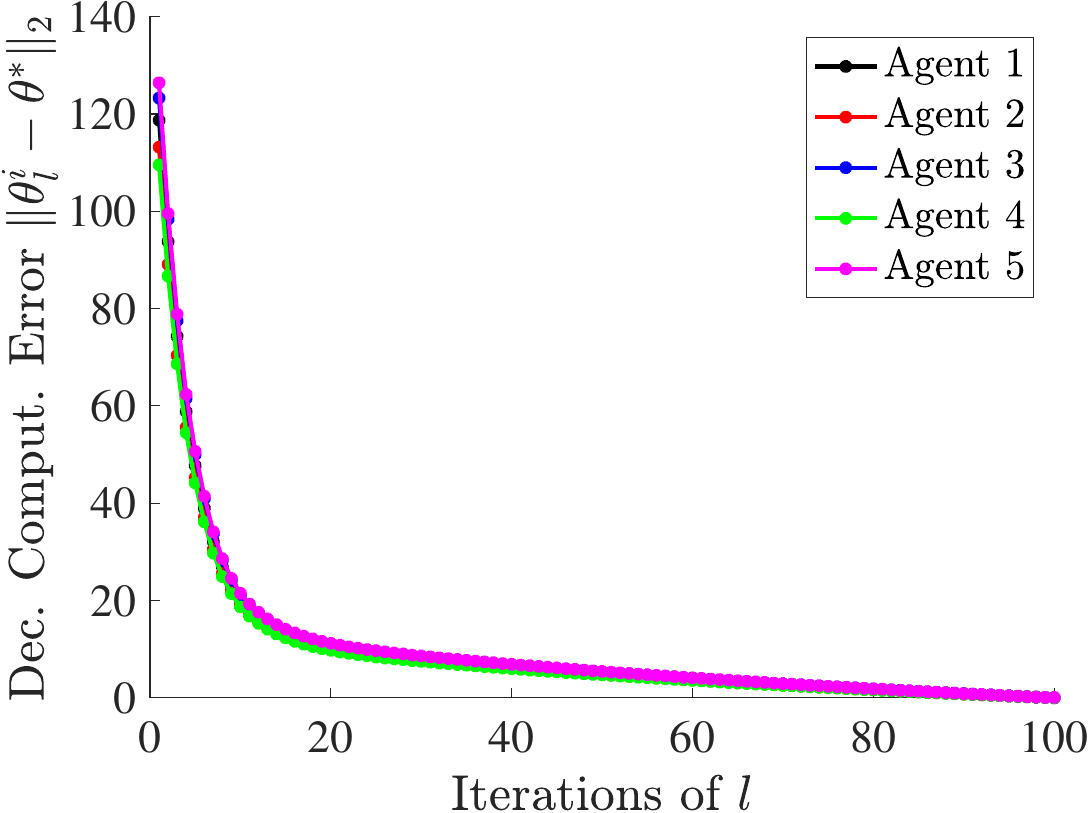}
		\\ 
		\hskip 0pt(a) at iteration $k=1$  & \hskip -1pt(b) at iteration $k=50$   & \hskip -1pt(c) at iteration $k=100$   
	\end{tabular}
	\caption{Convergence of the inner-loop of Algorithm \ref{algo:fit_Q_Compet}, which  solves the fitting problem \eqref{equ:para_fitted_least_squares}, using Algorithm \ref{algo:DIGing} in a decentralized fashion, at three outer-loop iterations $k=1,~50$ and $100$. }
\label{fig:consensus_error} 
\end{figure*} 

With linear function approximation, we follow the DIGing algorithm summarized in Algorithm \ref{algo:DIGing}. We choose stepsize $\alpha=0.1$. We first demonstrate the convergence of Algorithm \ref{algo:DIGing}. Let $T=1000$, $L=100$, $d=12$, we plot the convergence of the inner-loop using Algorithm \ref{algo:DIGing} for three examples of $k=1,50,100$.  We note that as $d\ll |\cS||\cA||\cB|$, the matrix $\Mb^{\text{MG}}=T^{-1}\cdot \sum_{t=1}^T\varphi(s_t,a_t,b_t)\varphi^\top(s_t,a_t,b_t)$ can easily become positive definite (full-rank), i.e., Assumption \ref{assum:linear_features} holds. 
It is shown in Fig. \ref{fig:consensus_error} that the inner-loop iterate converges successfully, with all agents reaching (almost) consensus eventually, towards the  unique optimal solution  $\theta^*$ to the fitting problem \eqref{equ:para_fitted_least_squares}. This justifies that with a large enough $L$, the one-step decentralized computation error can indeed be controlled and made small (cf. Assumption \ref{assum:one_step_comp_error}).

\begin{figure*}[!t]
	\centering
	\begin{tabular}{ccc}
		\hskip-13pt\includegraphics[width=0.34\textwidth]{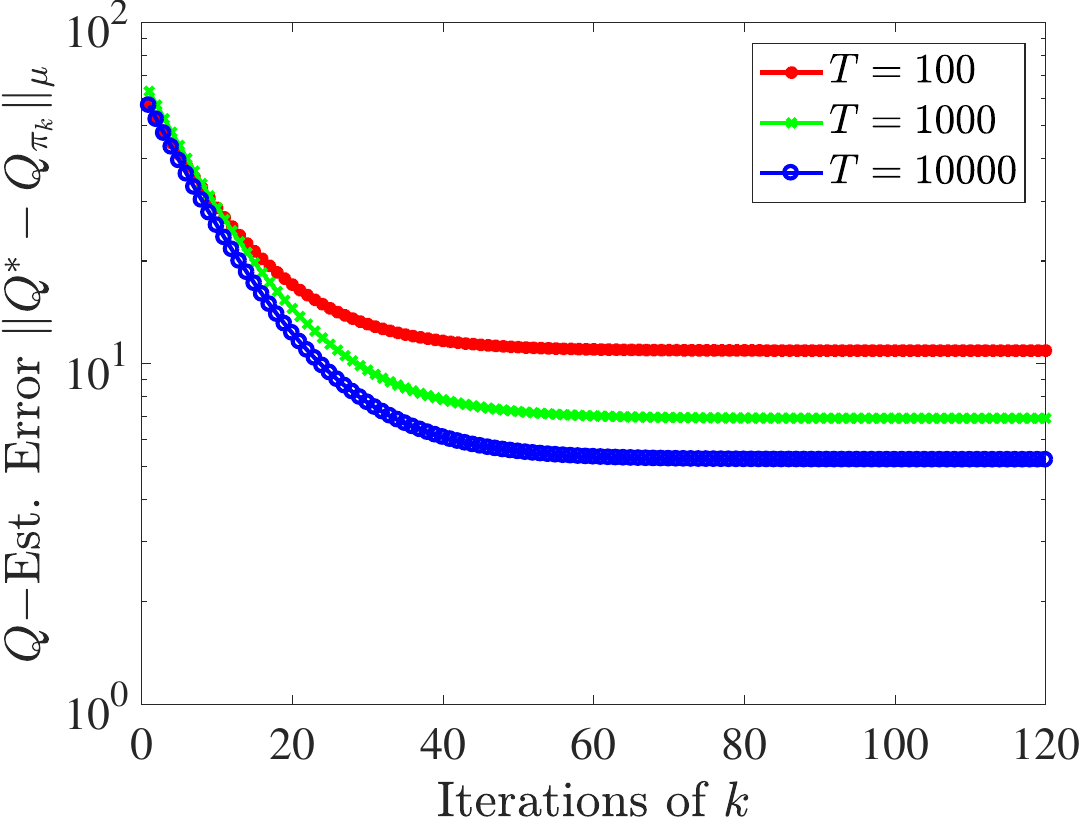}
		&
		\hskip -10pt\includegraphics[width=0.34\textwidth]{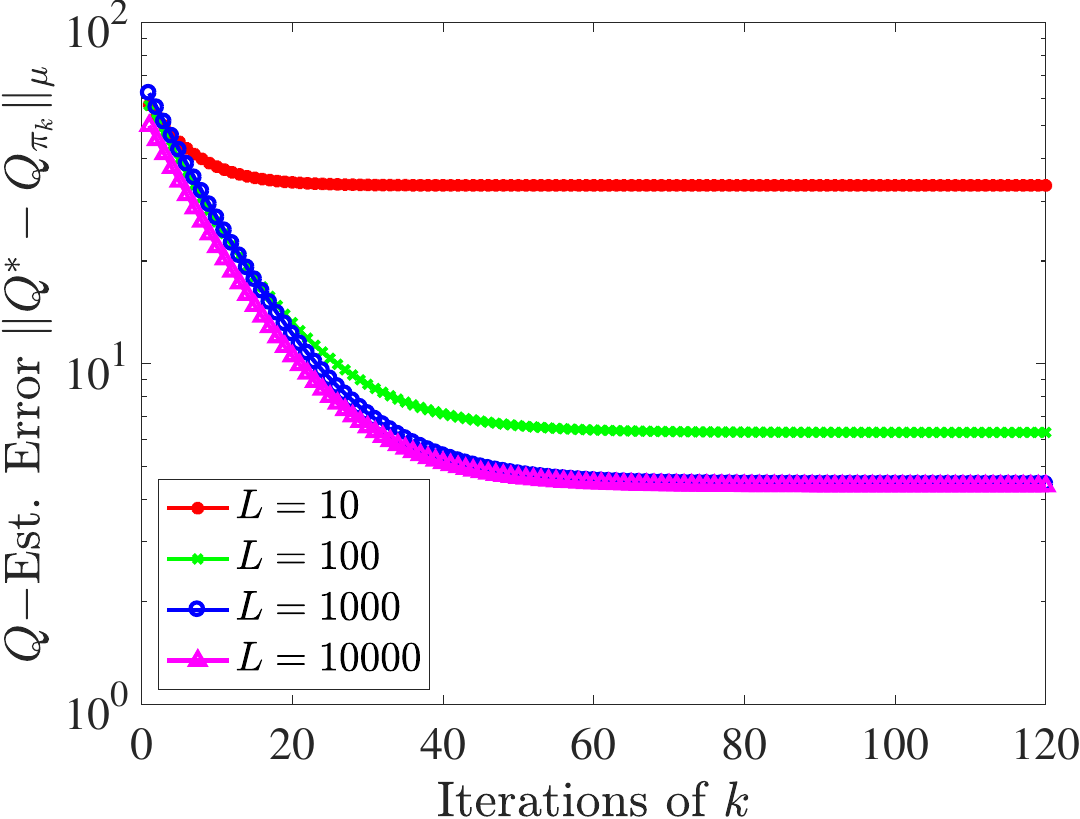}
		&
		\hskip -10pt\includegraphics[width=0.34\textwidth]{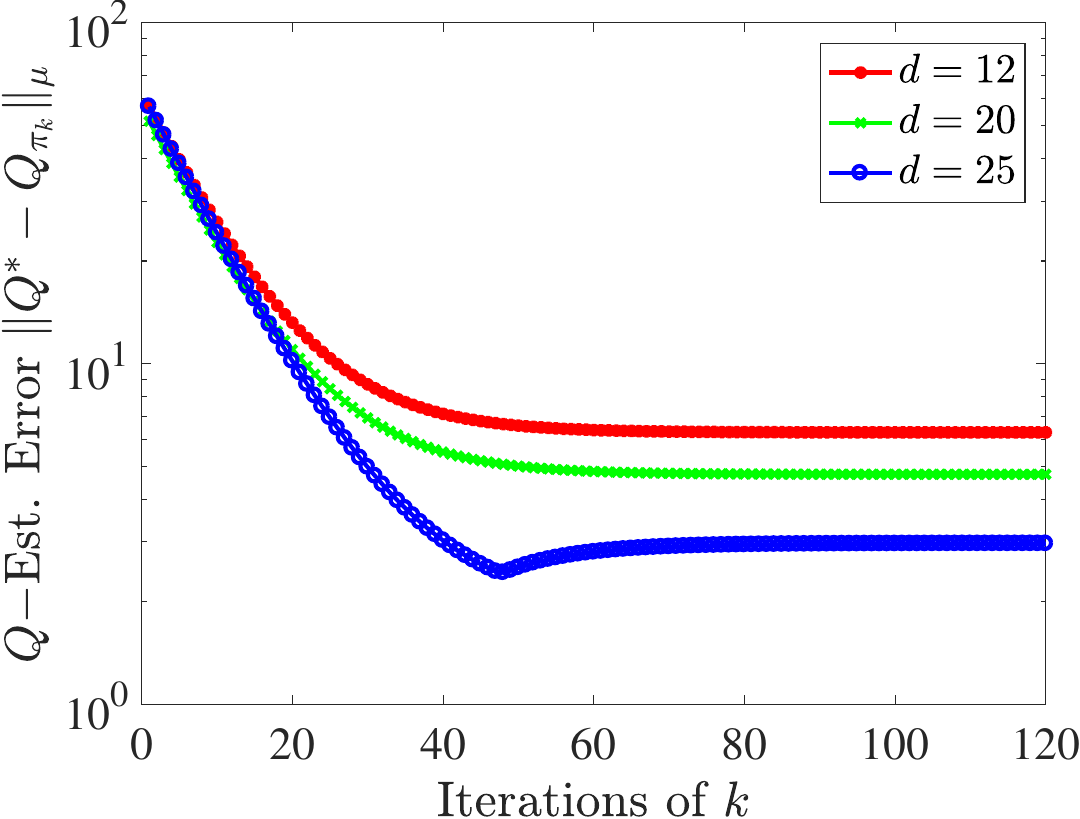}
		\\
		\hskip -4pt(a) varying $T$  & \hskip -3pt(b) varying $L$   & \hskip -5pt(c) varying $d$   
	\end{tabular}
	\caption{Convergence of the outer-loop of Algorithm \ref{algo:fit_Q_Compet}, under different choices of the algorithm parameters, the number of samples $T$, the decentralized computation iteration numbers $L$, and the dimension of the feature vector $d$. These figures are based on $20$ independent trials.}
\label{fig:variation} 
\end{figure*}  

We then illustrate the convergence of the algorithms under various choices of $T$, $L$, and $d$. First, we fix $L=100$ and $d=12$, and change $T$ to be $100,1000,10000$. Simulation results are collected after $20$ independent trials. As shown in Fig. \ref{fig:variation} (a), the batch MARL algorithm converges successfully. With a finite number of samples and iterations, the Q-function estimation error converges to some constant error, which is smaller for larger $T$, corresponding to the vanishment of the estimation error term in Theorem \ref{thm:main_thm_compet}. However, with a finite $L$, the decentralized computation error is inevitable, so is the approximation error term that depends on the function class $\cH_\Theta$. 
Similarly, as shown in Fig. \ref{fig:variation} (b), where we fix $T=1000,~d=12$ and change $L$ to be $10,~100,~1000,~10000$, the increase of  the inner-loop Algorithm \ref{algo:DIGing} iteration number $L$ can also decrease the final error, corresponding to the decentralized computation error term in Theorem \ref{thm:main_thm_compet}. Such an improvement from $L=100$ to $L=1000$ is much smaller than that from $L=10$ to $L=100$, as the agents have almost already reached consensus after $L=100$ iterations. Hence, the increase of $L$ from $1000$ to $10000$ has almost no effect on the final error, as the decentralized computation error is almost zero.  
 Moreover, note that a relatively larger one-step decentralized computation error (when $L=10$) does not prevent the overall batch algorithm from converging. Finally, we compare the algorithm performance under different function classes $\cH_\Theta$. We fix $T=1000$ and $L=100$, and increase the dimension of the parameter $d$ to $20$ and $25$. The larger $d$ is, the richer the function class $\cH_\Theta$ is. As a result, a larger $d$ leads to a smaller inherent approximation error in the bound in Theorem \ref{thm:main_thm_compet}, as corroborated in Fig. \ref{fig:variation} (c).

\section{Conclusions}

In this paper, we provided a finite-sample analysis for  batch multi-agent RL with networked agents. Specifically, we considered both the cooperative and competitive  MARL settings. In the cooperative setting,   a team of heterogeneous agents connected by a time-varying communication network aim to maximize the globally averaged return of all agents, with no existence  of any central controller. 
In the competitive setting, two teams of such  decentralized RL agents form a zero-sum Markov game. Our settings cover several conventional MARL settings as special cases, e.g., the conventional cooperative  MARL with a common reward function for all agents, and the competitive MARL that is usually modeled as a two-player zero-sum Markov game. 

In both settings, we proposed fitted-Q iteration-based MARL algorithms, with the aid of decentralized  optimization algorithms to solve the fitting problem at each iteration. We quantified the performance bound of the output action-value, using finite number of samples drawn from a single RL trajectory. In addition, with linear function approximation, we further  derived the performance bound after finite number of iterations of decentralized computation. 
To our knowledge, these are  the first finite-sample analysis for batch multi-agent RL, in either the cooperative or the competitive setting. 
One interesting future direction is to extend the finite-sample analysis to more general MARL settings, e.g., general-sum Markov games. It is also promising to sharpen the bounds we obtained, in order to better understand and improve both the sample and computation efficiency of MARL algorithms.

\bibliographystyle{ims}  
\bibliography{rl_ref,dis_opt_ref,AC_over_networks}



\appendix{}

\section{Definitions} \label{sec:append_term_def}
In this section, we provide the detailed definitions of some terms used in the main text for the sake of completeness. 
   
We first give the definition of $\beta$-mixing of a stochastic  process mentioned in Assumption \ref{assum:sample_path}.

\begin{definition}[$\beta$-mixing]\label{def:mixing}
	Let $\{Z_t\}_{t=1,2,\cdots}$ be a stochastic process, and denote the collection of $(Z_1,\cdots,Z_t)$ by $Z^{1:t}$. Let $\sigma(Z^{i:j})$ denote the $\sigma$-algebra generated by $Z^{i:j}$. The $m$-th $\beta$-mixing coefficient of $\{Z_t\}$, denoted by $\beta_m$, is defined as
	\$
	\beta_m=\sup_{t\geq 1}~\EE\bigg[\sup_{B\in\sigma(Z^{t+m:\infty})}\big|P(B\given Z^{1:t})-P(B)\big|\bigg].
	\$
	Then, $\{Z_t\}$ is said to be $\beta$-mixing if $\beta_m\to 0$ as $m\to \infty$. In particular, we say that a $\beta$-mixing process mixes at an \emph{exponential} rate with parameters $\overline{\beta},g,\zeta>0$ if $\beta_m\leq \overline{\beta}\cdot\exp(-gm^\zeta)$ holds for all $m\geq 0$.
\end{definition}

We then provide the formal definitions of  \emph{concentrability coefficients} as in \cite{munos2008finite,perolat2015approximate}, for multi-agent MDPs and team zero-sum Markov games with networked agents,  respectively, used in Assumption \ref{assume:concentrability}.

\begin{definition}[Concentrability Coefficient for Multi-agent MDPs with Networked Agents]\label{def:concentrability_MMDP}
	Let $\nu_1, \nu_2 \in \cP(\cS\times\cA)$ be two probability measures that are  absolutely continuous with respect to the Lebesgue measure on $\cS\times\cA$.  Let   $\{ \pi_t \} $ be a sequence of joint policies for all the agents, with $\pi_t:\cS\to \cP(\cA)$ for all $t$. Suppose the initial state-action pair  $(s_0, a_0)$ has distribution $ \nu_1$, and the   action $a_t$ is sampled from the joint policy $\pi_{t}$.    For any integer $m$, we denote by $ P_{\pi_{m} } P_{\pi_{m-1} } \cdots P_{\pi_{1} }\nu_1$ the   distribution of $(s_m, a_m)$ under the policy sequence $\{ \pi_t \}_{t=1,\cdots,m} $. 
Then, the $m$-th  concentration coefficient is defined as 
 \$
 \kappa^{\text{MDP}}(m; \nu_1, \nu_2) = \sup_{\pi_{1}, \ldots, \pi_{m} } \biggl [\EE _{\nu_2} \biggl | \frac{    \ud (  P_{\pi_{m} } P_{\pi_{m-1} } \cdots P_{\pi_{1} }\nu_1)   } { \ud \nu_2} \biggr | ^2     \biggr ] ^{1/2},
 \$
 where ${\ud(  P_{\pi_{m} } P_{\pi_{m-1} } \cdots P_{\pi_{1} }\nu_1)}/{\ud\nu_2}$ is the Radon-Nikodym derivative  of $P_{\pi_{m} } P_{\pi_{m-1} } \cdots P_{\pi_{1} }\nu_1$ with respect to $\nu_2$,  and the supremum is taken over all possible joint policy sequences $\{\pi_{t}\}_{t=1,\cdots,m}$.
\end{definition}

\begin{definition}[Concentrability Coefficient for Zero-sum Markov Games with Networked Agents]\label{def:concentrability_MG}
		Let $\nu_1, \nu_2 \in \cP(\cS\times\cA\times\cB)$ be two probability measures that are  absolutely continuous with respect to the Lebesgue measure on $\cS\times\cA\times\cB$.  Let   $\{(\pi_t,\sigma_t) \} $ be a sequence of joint policies for all the agents in both teams, with $\pi_t:\cS\to \cP(\cA)$ and $\sigma_t:\cS\to \cP(\cB)$ for all $t$. Suppose the initial state-action pair  $(s_0, a_0,b_0)$ has distribution $ \nu_1$, and the   action $a_t$ and $b_t$ are sampled from the joint policy $\pi_{t}$ and $\sigma_{t}$, respectively.    For any integer $m$, we denote by $ P_{\pi_{m},\sigma_{m}} P_{\pi_{m-1},\sigma_{m-1}} \cdots P_{\pi_{1},\sigma_{1}}\nu_1$ the   distribution of $(s_m, a_m,b_m)$ under the policy sequence $\{ (\pi_t,\sigma_t) \}_{t=1,\cdots,m} $. 
Then, the $m$-th  concentration coefficient is defined as 
\$
 \kappa^{\text{MG}}(m; \nu_1, \nu_2) = \sup_{\pi_{1}, \sigma_{1}, \ldots, \pi_{m},\sigma_{m} } \biggl [\EE _{\nu_2} \biggl | \frac{    \ud (  P_{\pi_{m},\sigma_{m}} P_{\pi_{m-1},\sigma_{m-1}} \cdots P_{\pi_{1},\sigma_{1}}\nu_1)   } { \ud \nu_2} \biggr | ^2     \biggr ] ^{1/2},
 \$
 where ${    \ud (  P_{\pi_{m},\sigma_{m}} P_{\pi_{m-1},\sigma_{m-1}} \cdots P_{\pi_{1},\sigma_{1}}\nu_1)   }/{ \ud \nu_2}$ is the Radon-Nikodym derivative  of $P_{\pi_{m},\sigma_{m}} P_{\pi_{m-1},\sigma_{m-1}} \cdots P_{\pi_{1},\sigma_{1}}\nu_1$ with respect to $\nu_2$, and the supremum is taken over all possible joint policy sequences $\{(\pi_t,\sigma_t )\}_{t=1,\cdots,m} $.
\end{definition}

To characterize the capacity of function classes, we define the (empirical) covering numbers of a function class, following the definition in \cite{antos2008learning}.

\begin{definition}[Definition $3$ in \cite{antos2008learning}]\label{def:covering_num} For any fixed $\epsilon>0$, and a \emph{pseudo metric space} $\cM=(\cM,d)$\footnote{A pseudo-metric satisfies all the properties of a metric except that the requirement of distinguishability is
removed.}, we say that $\cM$ is covered by $m$ discs  $D_1,\cdots,D_m$ if $\cM\subset \bigcup_j D_j$.  We define the covering number $\cN_{c}(\epsilon,\cM,d)$ of $\cM$ as the smallest integer $m$ such that $\cM$ can be covered by $m$ discs each of which having a radius less than $\epsilon$. In particular, for a class $\cH$ of real-valued functions with domain $\cX$ and points $x^{1:T}=(x_1,\cdots,x_T)$ in $\cX$,  the \emph{empirical covering number} is  defined as the  covering number of $\cH$ equipped with the empirical $\ell^1$ pseudo metric,
	\$
	l_{x^{1:T}}(f,g)=\frac{1}{T}\sum_{t=1}^Td\big(f(x_t),g(x_t)\big), \text{~for any~} f,g\in\cH
	\$ 
	where $d$ is a distance function on the range of functions in $\cH$. When this range is the reals, 
we use $d(a,b)=|a-b|$. For brevity, we  denote $\cN_{c}(\epsilon,\cH,l_{x^{1:T}})$ by $\cN_1(\epsilon, \cH,x^{1:T})$. 
	
\end{definition}



\section{Auxiliary Results} \label{sec:append_proofs}
In this section, we present several  auxiliary results used previously and some of their proofs.

\begin{lemma}[Sum of Rank-$1$ Matrices]\label{lemma:append_rank_1}
	Suppose there are $T$ vectors $\{\varphi_t\}_{t=1,\cdots,T}$ with $\varphi_t\in\RR^d$ and $d\leq T$, and the matrix $[\varphi_1,\cdots,\varphi_T]^\top$ has full column-rank, i.e., rank $d$. Let $\Mb=T^{-1}\cdot\sum_{t=1}^T\varphi_t\varphi_t^\top$. Then, the matrix $\Mb$ is full-rank, i.e., has rank-$d$.
\end{lemma}
\begin{proof}
	Since the matrix $[\varphi_1,\cdots,\varphi_T]^\top$ has full column-rank, there exists at least one subset of $d$ vectors $\{\varphi_{t_1},\cdots,\varphi_{t_d}\}$, such that  these $d$ vectors are linearly independent. Let $\cI_d=\{t_1,\cdots,t_d\}$. Consider a  matrix $\tilde \Mb=T^{-1}\cdot\sum_{i=1}^d\varphi_{t_i}\varphi_{t_i}^\top$. Now we show that $\tilde \Mb$ is full rank. Consider a nonzero $v_1\in\RR^d$ such that $v_1\in\text{span}\{\varphi_{t_2},\cdots,\varphi_{t_d}\}$; then $\tilde \Mb v_1=T^{-1}\cdot\varphi_{t_1}(\varphi_{t_1}^\top v_1)$  is a nonzero scalar multiple of $\varphi_{t_1}$. Otherwise, $v_1$ is also orthogonal to $\varphi_{t_1}$. Then, all $\varphi_{t_i}$ for $i=1,\cdots,d$ are in the orthogonal space of $v_1$, which is of dimension $d-1$. This contradicts the fact that the $d$ vectors $\{\varphi_{t_1},\cdots,\varphi_{t_d}\}$ are linearly independent. This way, we can construct nonzero scalar multiples  of all $\varphi_{t_i}$ for $i=1,\cdots,d$, say $\{v_1,\cdots,v_d\}$, which are linearly independent and all lie in the range space of $\tilde \Mb$. Hence $\tilde \Mb$ is full-rank, with the smallest eigenvalue $\lambda_{\min}(\tilde \Mb)>0$. 
	
	In addition, for any $t_j\in[T]$ but $t_j\notin \cI_d$, it holds that $\lambda_{\min}(\tilde \Mb+T^{-1}\cdot\varphi_{t_j} \varphi_{t_j}^\top )\geq \lambda_{\min}(\tilde \Mb)$, since for any $x\in\RR^d$, $x^\top \tilde \Mb x+T^{-1}\cdot (x^\top\varphi_{t_j})^2\geq x^\top \tilde \Mb x$. Therefore, $\tilde \Mb+T^{-1}\cdot\varphi_{t_j} \varphi_{t_j}^\top$ is also full-rank, and so is  the matrix $\Mb$, which  completes the proof.
\end{proof}

Lemma \ref{lemma:append_rank_1} can be used directly to show that with a rich enough data set $\cD=\{(s_t,a_t)\}_{t=1,\cdots,T}$, the matrix $\Mb^{\text{MDP}}$ defined  in Assumption \ref{assum:linear_features} is full rank. This argument also applies to the matrix $\Mb^{\text{MG}}$, and can be used to justify the rationale behind Assumption \ref{assum:linear_features}.

We have the following proposition to bound the empirical covering number of a function class using its pseudo-dimension. The proof of the proposition can be found in  the proof of \cite[Corollary $3$]{haussler1995sphere}.

\begin{proposition}[\cite{haussler1995sphere}, Corollary $3$]\label{prop:hauss_empirical_cov}
	For any set $\cX$, any points $x^{1:T}\in\cX^T$, any class $\cH$ of functions on $\cX$ taking values in $[0,K]$ with pseudo-dimention $V_{\cH^+}<\infty$, and any $\epsilon>0$, then
	\$
	\cN_1(\epsilon,\cH,x^{1:T})\leq e(V_{\cH^+}+1)\bigg(\frac{2eK}{\epsilon}\bigg)^{V_{\cH^+}}.
	\$ 
\end{proposition}

For distributed optimization with strongly convex  objectives, the following lemma from \cite{nedic2017achieving} characterizes the geometric convergence rate of the \emph{DIGing} algorithm (see Algorithm \ref{algo:DIGing}) under time-varying communication networks. Lemma \ref{lemma:nedic_geo_rate} is used in the proof of Corollaries \ref{coro:complex_LFA}.  
%
%

\begin{lemma}[\cite{nedic2017achieving}, Theorem $10$]\label{lemma:nedic_geo_rate}
Consider the decentralized optimization problem 
\#\label{equ:decen_opt_origin}
\min_{x}\quad\sum_{i=1}^N f_i(x),
\#
where for each $i\in[N]$, $f_i:\RR^p\to\RR$ is differentiable, $\mu_i$-strongly convex,  and
has $L_i$-Lipschitz continuous gradients. Suppose Assumption \ref{assum:consensus_mat} on the consensus matrix $\Cb$ holds, recall that  $\delta,B$ are the parameters in the assumption, then
	there exists a constant $\lambda=\lambda(\chi,B,N,\overline{L},\overline{\mu})\in[0,1)$, where $\overline{\mu}=N^{-1}\cdot\sum_{i=1}^N \mu_i$ and   $\overline{L}=\max_{i\in[N]}L_i$,  such that the sequence of matrices  $\{[x^1_l,\cdots,x^N_l]^\top\}$ generated by DIGing algorithm along iterations $l=0,1,\cdots$, converges to the matrix $\bx^*=\bm1(x^*)^\top$, where $x^*$ is the unique solution to \eqref{equ:decen_opt_origin}, at a geometric rate $O(\lambda^l)$. 
\end{lemma}

\subsection{Proof of Lemma \ref{lemma:empirical_cov_num_H_vee}}\label{sec:proof_empirical_cov_num_H_vee}
	\begin{proof}
	First consider  $N=1$, i.e., $\Qb=Q^1\in\cH$. Let $\{Q_{j}\}_{l\in[\cN_1(\epsilon',\cH,\cD_y)]}$ be the $\epsilon'$-covering  of $\cH(\cD_y)$ for some $\epsilon'>0$. 
	Recall that $\cA=\{a_1,\cdots,a_A\}$. Now we show that $\{\max_{a\in\cA}Q_l(\cdot,a)\}_{l\in[\cN_1(\epsilon',\cH,\cD_y)]}$ is a covering of $\cH^{\vee}_1$.
	 
	Since for any $Q^1$, there exists an  $l'\in[\cN_1(\epsilon',\cH,\cD_y)]$ such that 
		\$
	\frac{1}{AT}\sum_{t=1}^{T}\sum_{j=1}^{A}\Big|Q^1(y_t,a_j)-Q_{l'}(y_t,a_{j})\Big|\leq \epsilon'.
	\$
	Let $\epsilon'=\epsilon/A$ and $a^{Q}_t\in\argmax_{a\in\cA}|Q^1(y_t,a)-Q_{l'}(y_t,a)|$,  we have 
		\$
	&\frac{1}{T}\sum_{t=1}^{T}\Big|\max_{a\in\cA}Q^1(y_t,a)-\max_{a\in\cA}Q_{l'}(y_t,a)\Big|\leq \frac{1}{T}\sum_{t=1}^{T}\max_{a\in\cA}\Big|Q^1(y_t,a)-Q_{l'}(y_t,a)\Big|\\
	&\quad=\frac{1}{T}\sum_{t=1}^{T}\Big|Q^1(y_t,a^{Q}_t)-Q_{l'}(y_t,a^{Q}_t)\Big| \leq \frac{1}{T}\sum_{t=1}^{T}\sum_{j=1}^{A}\Big|Q^1(y_t,a_j)-Q_{l'}(y_t,a_{j})\Big|\leq\frac{\epsilon}{A}\cdot A = {\epsilon},
	\$
	where the second inequality  	follows from the fact that  
	\$
	|Q^1(y_t,a^{Q}_t)-Q_{l'}(y_t,a^{Q}_t)|\leq \sum_{j=1}^{A}|Q^1(y_t,a_j)-Q_{l'}(y_t,a_{j})|.
	\$ 
	This shows that  $\cN_1({\epsilon},\cH^{\vee}_1,y_Z)\leq \cN_1({\epsilon}/{A},\cH,\cD_y)$. Moreover, since functions in $\cH^{\vee}_N$ are averages of functions from $N$ $\cH^{\vee}_1$, we have 
	\$
	\cN_1({\epsilon},\cH^{\vee}_N,y_Z)\leq \big[\cN_1({\epsilon},\cH^{\vee}_1,y_Z)\big]^N\leq \big[\cN_1({\epsilon}/{A},\cH,\cD_y)\big]^N.
	\$
	On the other hand, by Corollary $3$ in \cite{haussler1995sphere} (see also Proposition \ref{prop:hauss_empirical_cov} in \S\ref{sec:append_proofs}), we can bound $\cN_1({\epsilon}/{A},\cH,\cD_y)$ by the pseudo-dimension of $\cH$, i.e.,
	\$
	\cN_1({\epsilon},\cH^{\vee}_N,y_Z)\leq \bigg[e(V_{\cH^+}+1)\bigg(\frac{2eQ_{\max}A}{\epsilon}\bigg)^{V_{\cH^+}}\bigg]^N,
	\$
	which concludes the proof.
	\end{proof}

\subsection{Proof of Lemma \ref{lemma:empirical_cov_num_H_vee_2}}\label{sec:proof_empirical_cov_num_H_vee}
	\begin{proof}
	The proof is similar to that in \ref{sec:proof_empirical_cov_num_H_vee} for Lemma \ref{lemma:empirical_cov_num_H_vee}.
	Let $\{Q_{l}\}_{l\in[\cN_1(\epsilon',\cH,\cD_y)]}$ be the $\epsilon'$-covering  of $\cH(\cD_y)$ for some $\epsilon'>0$. Then we claim that 
	\$
	\Big\{\max_{\pi'\in\cP(\cA)}\min_{\sigma'\in\cP(\cB)}\EE_{\pi',\sigma'}[Q_l(\cdot,a',b')]\Big\}_{l\in[\cN_1(\epsilon',\cH,\cD_y)]}
	\$
	covers $\cH^{\vee}_1$. By definition, for any $Q^{1,1}\in\cH$, there exists  $l'\in[\cN_1(\epsilon',\cH,\cD_y)]$ such that 
		\$
	\frac{1}{ABT}\sum_{t=1}^{T}\sum_{j=1}^{A}\sum_{k=1}^{B}\Big|Q^{1,1}(y_t,a_j,b_k)-Q_{l'}(y_t,a_{j},b_k)\Big|\leq \epsilon'.
	\$ 
	Let $(\pi^Q_t,\sigma^Q_t)$ be the strategy pair given $y_t$ such that 
	\$
	&\max_{\pi'\in\cP(\cA),\sigma'\in\cP(\cB)}\Big|\EE_{\pi',\sigma'}[Q^{1,1}(y_t,a',b')]-\EE_{\pi',\sigma'}[Q_{l'}(y_t,a',b')]\Big|\\
	&\quad=\Big|\EE_{\pi^Q_t,\sigma^Q_t}[Q^{1,1}(y_t,a',b')]-\EE_{\pi^Q_t,\sigma^Q_t}[Q_{l'}(y_t,a',b')]\Big|.
	\$
	Then, let $\epsilon'=\epsilon/(AB)$, we have
	\$
	&\frac{1}{T}\sum_{t=1}^{T}\Big|\max_{\pi'\in\cP(\cA)}\min_{\sigma'\in\cP(\cB)}\EE_{\pi',\sigma'}[Q^{1,1}(y_t,a',b')]-\max_{\pi'\in\cP(\cA)}\min_{\sigma'\in\cP(\cB)}\EE_{\pi',\sigma'}[Q_{l'}(y_t,a',b')]\Big|\\
	&\quad \leq \frac{1}{T}\sum_{t=1}^{T}\max_{\pi'\in\cP(\cA),\sigma'\in\cP(\cB)}\Big|\EE_{\pi',\sigma'}[Q^{1,1}(y_t,a',b')]-\EE_{\pi',\sigma'}[Q_{l'}(y_t,a',b')]\Big|\\
	&\quad =\frac{1}{T}\sum_{t=1}^{T}\Big|\EE_{\pi^Q_t,\sigma^Q_t}[Q^{1,1}(y_t,a',b')]-\EE_{\pi^Q_t,\sigma^Q_t}[Q_{l'}(y_t,a',b')]\Big|\\
	&\quad \leq \frac{1}{T}\sum_{t=1}^{T}\EE_{\pi^Q_t,\sigma^Q_t}\Big|Q^{1,1}(y_t,a',b')-Q_{l'}(y_t,a',b')\Big|\leq \frac{\epsilon}{AB}\cdot (AB) = {\epsilon}.
	\$
	Thus, by the relation between $\cH^{\vee}_N$ and $\cH^{\vee}_1$, and  Proposition \ref{prop:hauss_empirical_cov}, we further obtain 
	\$
	\cN_1({\epsilon},\cH^{\vee}_N,y_Z)\leq \big[\cN_1({\epsilon}/{(AB)},\cH,\cD_y)\big]^N\leq \bigg[e(V_{\cH^+}+1)\bigg(\frac{2eQ_{\max}AB}{\epsilon}\bigg)^{V_{\cH^+}}\bigg]^N,
	\$
	which concludes the proof.
	\end{proof}

\end{document}